\documentclass{./include/latex-class-tlp/tlp}

\usepackage[utf8]{inputenc}
\usepackage{amsmath}
\usepackage{amssymb}
\usepackage{microtype}
\usepackage{listings}
\usepackage{stmaryrd}
\usepackage{url}
\usepackage{refcount}
\usepackage[linesnumbered,ruled,vlined]{algorithm2e}

\providecommand{\Underscore}{\textunderscore}

\lstdefinelanguage{clingo}{basicstyle=\ttfamily,keywordstyle=[1]\bfseries,keywordstyle=[2]\bfseries,keywordstyle=[3]\bfseries,showstringspaces=false,literate={_}{\Underscore}1 {\%\%}{}0,escapeinside={\#(}{\#)},alsoletter={\#,\&},keywords=[1]{not,from,import,def,if,else,elif,return,while,break,and,or,for,in,del,and,class,with,as,is,yield,async},keywords=[2]{\#const,\#show,\#minimize,\#base,\#theory,\#count,\#external,\#program,\#script,\#end,\#heuristic,\#edge,\#project,\#show,\#sum},morecomment=[l]{\#\ },morecomment=[l]{\%\ },morestring=[b]",stringstyle={\itshape},commentstyle={\color{darkgray}}}

\lstdefinelanguage{clingcon}[]{clingo}{morekeywords={&dom,&sum,&nsum,&diff,&disjoint,&distinct,&minimize,&maximize,&show}}
\lstdefinelanguage{flingo}[]{clingo}{morekeywords={&sum,&sus,&in,&df,&min,&max,&show}}
\lstdefinelanguage{clingodl}[]{clingo}{morekeywords={&diff}}

\lstdefinelanguage{python}{basicstyle=\ttfamily,keywordstyle=[1]\bfseries,showstringspaces=false,literate={_}{\Underscore}{1},escapeinside={\#(}{\#)},alsoletter={\#,\&},keywords=[1]{not,from,import,def,if,else,elif,return,while,break,and,or,for,in,del,and,class,with,as,is,yield,async},morecomment=[l]{\#\ },morestring=[b]",stringstyle={\itshape},commentstyle={\color{darkgray}}}

 \providecommand{\sysfont}{\textit}

\newcommand{\Clingo}{\sysfont{Clingo}}

\newcommand{\Clingcon}{\sysfont{Clingcon}}
\newcommand{\clingcon}{\sysfont{clingcon}}
\newcommand{\clingo}{\sysfont{clingo}}
\newcommand{\clingodl}{\clingoM{dl}}

\newcommand{\Clingodl}{\ClingoM{dl}}

\newcommand{\ezcsp}{\sysfont{ezcsp}}

\newcommand{\lars}{\sysfont{lars}}

\newcommand{\telingo}{\sysfont{telingo}}

\newcommand{\clingoM}[1]{\clingo{\small\textnormal{[}\textsc{#1}\textnormal{]}}}
\newcommand{\ClingoM}[1]{\Clingo{\small\textnormal{[}\textsc{#1}\textnormal{]}}}

 \providecommand{\logfont}{\textrm}

\newcommand{\HT}{\ensuremath{\logfont{HT}}}

\newcommand{\HTC}{\ensuremath{\logfont{HT}_{\!c}}}

\newcommand{\THT}{\ensuremath{\logfont{THT}}}
\newcommand{\THTf}{\ensuremath{\THT_{\!f}}}

\newcommand{\MHT}{\ensuremath{\logfont{MHT}}}

 \RequirePackage{bm}
\RequirePackage{textcomp}
\RequirePackage{upgreek}

\newcommand{\next}{\text{\rm \raisebox{-.5pt}{\Large\textopenbullet}}}    

\newcommand{\alwaysF}{\ensuremath{\square}}

\newcommand{\eventuallyF}{\ensuremath{\Diamond}}

\newcommand{\finally}{\ensuremath{\bm{\mathsf{F}}}}
\newcommand{\initially}{\ensuremath{\bm{\mathsf{I}}}}

\mathchardef\mhyphen="2D

\newcommand{\intervcc}[2]{\ensuremath{[#1..#2]}}
\newcommand{\intervco}[2]{\ensuremath{[#1..#2)}}

\newcommand{\rangeco}[3]{\ensuremath{#1 \in \intervco{#2}{#3}}}

\newcommand{\tuple}[1]{\ensuremath{\langle #1 \rangle}}
\newcommand{\Htrace}{\ensuremath{\mathbf{H}}}
\newcommand{\Ttrace}{\ensuremath{\mathbf{T}}}
\newcommand{\M}{\ensuremath{\mathbf{M}}}

\newcommand{\handt}{\tuple{H,T}}
\newcommand{\tandt}{\tuple{T,T}}

\lstdefinelanguage{clingos}{language=clingo,basicstyle=\small\ttfamily }

\newtheorem{definition}{Definition}
\newtheorem{theorem}{Theorem}
\newtheorem{proposition}{Proposition}
\newtheorem{corollary}{Corollary}
\newtheorem{lemma}{Lemma}

\newcommand{\eqdef}{=} 

\providecommand{\Next}{\text{\rm \raisebox{-.5pt}{\Large\textopenbullet}}}  
\newcommand{\metric}[3]{\ensuremath{#1_{
\ifthenelse{\equal{#2}{#3}}
{#2}
{
\ifthenelse{\equal{#2}{0}}
{
\ifthenelse{\equal{#3}{\omega}}
{}
{\leq#3}
      }
{
\ifthenelse{\equal{#3}{\omega}}
{\geq#2}
{\intervco{#2}{#3}}
      }
    }}}}
\newcommand{\metricI}[1]{\ensuremath{#1_{\cI}}}
\newcommand{\tmf}{\ensuremath{\tau}}
\newcommand{\cI}{\ensuremath{I}}

\newcommand{\matom}{\ensuremath{\mu}}
\newcommand{\laux}[2]{\ensuremath{L_{#1}^{#2}}}
\newcommand{\ljaux}[3]{\ensuremath{L_{#1}^{#2,#3}}}
\newcommand{\faili}[3]{\ensuremath{{\delta_{#1}^{#2,#3}}}}

\newcommand{\x}{\ensuremath{x}}
\newcommand{\y}{\ensuremath{y}}
\newcommand{\melrule}[1]{\ensuremath{\alwaysF{(#1)}}}
\newcommand{\alphabet}{\ensuremath{\mathcal{A}}}
\newcommand{\program}{\ensuremath{P}}

\newcommand{\body}{\ensuremath{\beta}}
\newcommand{\head}{\ensuremath{\alpha}}

\newcommand{\allI}{\ensuremath{\mathbb{I}}}
\newcommand{\allIp}{\ensuremath{\allI|_{\program}}}

\newcommand{\trans}[2]{\ensuremath{(#1)_{#2}}}
\newcommand{\tk}[1]{\trans{#1}{\kvar}}

\newcommand{\kvar}{\ensuremath{k}}
\newcommand{\tmvar}{\ensuremath{d}}
\newcommand{\timet}{\ensuremath{t}}

\newcommand{\tmflimit}{\ensuremath{\nu}}

\newcommand{\alphabetk}{\ensuremath{\mathcal{A}_k}}
\newcommand{\alphabets}{\ensuremath{\mathcal{A}^{*}}}

\newcommand{\alphabetT}{\ensuremath{\mathcal{T}}}

\newcommand{\den}[1]{\llbracket \, #1 \, \rrbracket}
\newcommand{\undefined}{\ensuremath{\boldsymbol{u}}} \newcommand{\Vh}{\ensuremath{h}} \newcommand{\Vt}{\ensuremath{t}} \newcommand{\htcinterp}{\ensuremath{\langle \Vh,\Vt \rangle}}
\newcommand{\htcttinterp}{\ensuremath{\langle \Vt,\Vt \rangle}}

\newcommand{\Label}[1]{\ensuremath{L_{#1}}}

\newcommand{\Labs}[3]{\ensuremath{{w_{#1}^{#2,#3}}}}
\newcommand{\Labsxy}[3]{\ensuremath{\Labs{\mu}{\x}{\y}}}

\newcommand{\subformulas}[1]{\ensuremath{\mathit{sub}(#1)}}

\newcommand{\mhtToHtf}{\ensuremath{\theta}} \newcommand{\mhtToHt}[1]{\ensuremath{\mhtToHtf(#1)}}
\newcommand{\mhtToHtc}[1]{\ensuremath{\mhtToHtf^c(#1)}}
\newcommand{\htToMhtf}{\ensuremath{\sigma}} \newcommand{\htToMht}[1]{\ensuremath{\htToMhtf(#1)}}
\newcommand{\htcToMht}[1]{\ensuremath{\htToMhtf^c(#1)}}

\newcommand{\true}{\ensuremath{\boldsymbol{t}}}
\newcommand{\false}{\ensuremath{\boldsymbol{f}}}

\newcommand{\trangeco}[3]{\ensuremath{#2 \leq #1 < #3}}

\newcommand{\trivalI}{\ensuremath{\bm{m}}}
\newcommand{\trivalIhtcletter}{\bm{m}} \newcommand{\trivalIhtc}{\ensuremath{\ensuremath{\trivalIhtcletter}^{\htcinterp}}}
\newcommand{\trivalItotalhtc}{\ensuremath{\ensuremath{\trivalIhtcletter}^{\tuple{t,t}}}}

\newcommand{\trivalIm}{\ensuremath{\bm{m}^{\tmf}}}
\newcommand{\trivaluation}[3]{\ensuremath{#3(#1,#2)}}
\newcommand{\trival}[2]{\trivaluation{#1}{#2}{\trivalIm}}

\newcommand{\alphabetI}{\ensuremath{\mathfrak{I}}}
\newcommand{\alphabetaux}{\ensuremath{\mathcal{L}_{\program}}}
\def \paramextended {\trivalIm}
\newcommand{\htcinterpex}{\chi{(\paramextended,P)}}
\newcommand{\kinlambda}{\rangeco{k}{0}{\lambda}}
\newcommand{\htcexmtri}{\ensuremath{\ensuremath{\trivalIhtcletter}^{\htcinterpex}}}
\newcommand{\exmtri}{\trivalI_{[\paramextended,P]}}
\renewenvironment{quote}{\list{}{\leftmargin15pt
    \rightmargin0cm
  }
  \item\relax
}
{\endlist}
\newcommand{\casex}[2]{\begin{quote}
    \hspace{-10pt}{\bfseries $\bm{#1}$}: #2\end{quote}}

\newcommand{\setpropcounter}[1]{
\setcounter{lastproposition}{\value{proposition}}
\setcounter{proposition}{\getrefnumber{#1}}
\addtocounter{proposition}{-1}
}
\newcounter{lastproposition}
\newcommand{\resetpropcounter}{
\setcounter{proposition}{\value{lastproposition}}
}
\newcommand{\relR}{\ensuremath{R}}
\newcommand{\openi}[1]{\ensuremath{\Omega(#1)}}

\newcommand{\htToMhttrace}[1]{\ensuremath{\htToMhtf_\mathcal{A}(#1)}}

\newcommand{\comment}[1]{}

\lefttitle{Arvid Becker et al.}

\jnlPage{1}{8}
\jnlDoiYr{2021}
\doival{10.1017/xxxxx}

\makeatletter
\begin{document}
\title[Implementing Metric ASP]{Implementing Metric Temporal Answer Set Programming
\thanks{Research partially funded by DFG grant SCHA 550/15, Germany.}}

\begin{authgrp}
  \author{\sn{Arvid} \gn{Becker}}
  \affiliation{University of Potsdam, Germany}
  \author{\sn{Pedro} \gn{Cabalar}}
  \affiliation{University of Corunna, Spain}
  \author{\sn{Mart\'{\i}n} \gn{Di\'eguez}}
  \affiliation{University of Angers, France}
  \author{\sn{Susana} \gn{Hahn}}
  \affiliation{University of Potsdam, Germany}\affiliation{Potassco Solutions, Germany}
  \author{\sn{Javier} \gn{Romero}}
  \affiliation{University of Potsdam, Germany}
  \author{\sn{Torsten} \gn{Schaub}}
  \affiliation{University of Potsdam, Germany}\affiliation{Potassco Solutions, Germany}
\end{authgrp}
\maketitle
\begin{abstract}
  We develop a computational approach to Metric Answer Set Programming
  to express quantitative temporal constrains, such as durations and deadlines.
  We investigate two specific fragments:
  plain metric logic programs, restricted to local temporal constraints, and
  general metric logic programs, which allow for arbitrary metric formulas.
  A central challenge in this context is maintaining scalability when dealing with fine-grained timing constraints,
  which can significantly exacerbate grounding bottleneck of Answer Set Programming (ASP).
  To address this issue, we
  propose translations of both fragments into
  standard ASP and
  ASP extended with difference constraints,
  a simplified form of linear constraints,
  and prove their correctness and completeness.
  Our implementation, realized via meta-encodings,
  effectively decouples metric ASP from the granularity of time,
  resulting in a solution that is independent of time precision.
  \textit{Under consideration in Theory and Practice of Logic Programming (TPLP).}
\end{abstract}
\begin{keywords}
  Metric Answer Set Programming,
  Equilibrium Logic,
  Difference Constraints
\end{keywords}
 \section{Introduction}\label{sec:introduction}

Metric temporal logics~\citep{TIMEHandbook} allow for expressing quantitative temporal constrains,
like durations and deadlines.
As an example, consider the dentist scenario~\citep{mellarkod07a}:
{\em``Ram is at his office and has a dentist appointment in one hour.
For the appointment, he needs his insurance card which is at home and cash to pay the doctor.
He can get cash from the nearby Atm.
Table~\ref{table:dentist} shows the time in minutes needed to travel between locations: Dentist, Home, Office and Atm.
For example, the time needed to travel between Ram's office to the Atm is 20 minutes.
The available actions are: moving from one location to another and picking items such as cash or insurance.
The goal is to find a plan which takes Ram to the doctor on time.''}
\begin{table}[ht]
  \begin{tabular}{|r|r|r|r|r|}
    \cline{1-5}
                     & \textit{Dentist} & \textit{Home} & \textit{Office} & \textit{Atm} \\ \cline{1-5}
    \textit{Dentist} & 0                & 20            & 30              & 40           \\ \cline{1-5}
    \textit{Home}    & 20               & 0             & 15              & 15           \\ \cline{1-5}
    \textit{Office}  & 30               & 15            & 0               & 20           \\ \cline{1-5}
    \textit{Atm}     & 40               & 15            & 20              & 0            \\ \cline{1-5}
  \end{tabular}
  \caption{Durations between locations} \label{table:dentist}
\end{table}
This example combines planning and scheduling, and
nicely illustrates the necessity to combine qualitative and quantitative temporal constraints.

Extensions to the Logic of Here-and-There and Equilibrium Logic~\citep{pearce96a} were developed in~\citep{cadiscsc20a,becadiscsc24a}
to semantically ground the incorporation of metric constraints into Answer Set Programming~(ASP;~\citealp{lifschitz19a}).
\footnote{While \cite{cadiscsc20a} equate time steps with state indices in a linear temporal model,
  \cite{becadiscsc24a} pursue a more general approach by
  associating each state with an arbitrary natural number representing a specific point in time.}
Building upon these semantic foundations, we develop a computational approach to metric ASP.
A central challenge in this is to maintain scalability when dealing with fine-grained timing constraints,
which can significantly exacerbate ASP's grounding bottleneck.
To address this issue, we leverage extensions of ASP with difference constraints~\citep{jakaosscscwa17a},
a simplified form of linear constraints,
to handle time-related aspects externally.
This approach effectively decouples metric ASP from the granularity of time,
resulting in a solution that is unaffected by time precision.
In detail, we
  (i) propose translations of metric logic programs into regular logic programs and their extension with difference constraints,
 (ii) prove the completeness and correctness of both translations in terms of equilibrium logic and its extensions
      with metric and difference constraints,
(iii) describe an implementation of both translations in terms of meta-encodings, and
(iv) present empirical performance results with novel metric encodings across three different domains,
namely, the dentist scenario presented above, multi-agent synchronization, and job scheduling.
From a modeling perspective, we may consider metric logic programs as a high-level modeling language
for logic programs with difference constraints in the context of temporal domains.
Our meta-programming approach enables the rapid exploration of alternative logical designs
without requiring modifications to an underlying solver,
such as the temporal ASP system \telingo~\citep{cadilasc20a}.
Instead, our approach provides a versatile environment suited for
both theoretical research and the generation of system blueprints.

We begin by examining a restricted, yet expressive, fragment of our language,
which allows the next operator to be constrained by a temporal interval.
Subsequently, we extend this approach to encompass the full language.
First, we demonstrate how arbitrary metric formulas can be transformed into a logic programming representation.
Then, we illustrate how the techniques developed for the restricted fragment can be generalized
to accommodate this more comprehensive format.

 \subsection*{Related work}\label{sec:related:work}
Pioneering efforts to extend ASP with linear integer constraints for
expressing quantitative temporal relations were introduced by~\cite{baboge05a} and~\cite{mellarkod07a}.
These early approaches ultimately inspired the development of modern hybrid ASP systems,
such as \clingcon~\citep{bakaossc16a}, \clingodl~\citep{jakaosscscwa17a}, and \ezcsp~\citep{balduccini09a}.
For instance,
the latter was used by \cite{bamamale17a} to address planning problems involving mixed discrete-continuous dynamics.
While these systems effectively manage numerical constraints,
they lack any support for high-level temporal languages and reasoning.

The logical foundation of our work lies in Metric Equilibrium Logic,
as defined by~\cite{becadiscsc24a}
as an extension of Linear temporal equilibrium logic~\citep{agcadipescscvi20a}.
The latter serves as the basis for the temporal ASP system \telingo~\citep{cadilasc20a},
which supports qualitative temporal reasoning but cannot express metric constraints over time intervals.
We bridge this gap by providing a general definition of metric logic programs and corresponding implementations
based on meta-encodings,
which may ultimately serve as blueprints for a future metric extension of \telingo.

Metric concepts have also been explored in stream reasoning,
most notably in the \lars~framework~\citep{bedaei18a},
which extends ASP with window operators.
However, our focus differs:
while \lars\ reasons over data streams,
our work targets general metric temporal reasoning.
This distinction is also reflected in the semantic foundations:
we rely on Metric Equilibrium Logic, while \lars\ utilizes a reduct-based semantics.
Similarly, metric extensions in Datalog~\citep{wagrkaka19a}
have led to systems such as \textit{meteor}~\citep{wahuwagr22a} and (temporal) \textit{vadalog}~\citep{beblnisa22a}.
These extensions are monotonic and thus incapable of expressing default reasoning.
Moreover, they are optimized for large-scale data querying with continuous time,
making them less suitable for the search-intensive problems encountered in ASP.
Interestingly, \cite{watekocu21a} investigate metric Datalog with negation under the stable model semantics,
employing an \HT-based semantics to define a restricted class of metric logic programs.
In the spirit of Datalog,
these programs limit the use of logical and temporal operators,
disallowing disjunction and ``existential'' operators (such as eventually, since, or until) in rule heads.
Finally,
\citep{huozdi20a} defines reductions of (monotonic) Metric temporal logic to Linear temporal logic.
A comprehensive comparative account of metric approaches in logic programming is given in~\citep{becadiscsc24a}.

In the context of planning,
\cite{sobatu04a} explore actions with durations in plain ASP.
While the proposed action language deals with quantitative temporal constraints,
it does not offer any metric logic foundations.
Finally, outside of ASP,
\cite{crecod03a} introduce durative actions and deadlines in PDDL~2.1,
though deadlines were modeled indirectly via special encodings rather than as native constructs.
\cite{gehalosadi09a} later extend this language to express deadlines for any state or condition via hard or
soft constraints in PDDL~3.
While these approaches are tailored to specific PDDL constructs,
our objective is to provide a general logical framework for metric reasoning in ASP.
 \section{Background}
\label{sec:background}

To ease the formal elaboration of our translations from metric ASP to
regular logic programs and their extension with difference constraints,
we put ourselves into the context of the logical framework of Here-and-There and Equilibrium Logic.
\subsection{The Logic of Here-and-There \citep{pearce96a}}
\label{sec:ht}
A formula over an alphabet \alphabet\ is defined as
\begin{align*}
  \varphi &::= \bot \mid a \mid \varphi\wedge\varphi \mid \varphi\vee\varphi \mid \varphi\to\varphi
  \quad\text{ where } a\in\alphabet\ .
\end{align*}
We define the derived operators
$\neg\varphi\eqdef\varphi\to\bot$
and
$\top\eqdef\neg\bot$.
Elements $a$ of \alphabet\ are called (Boolean) \emph{atoms}.
A \emph{literal} is an atom or an atom preceded by negation, viz.\ $a$ or $\neg a$.
A \emph{theory} is a set of formulas.
We sometimes write $\varphi\leftarrow\psi$ instead of $\psi\to\varphi$
to follow logic programming conventions.
A \emph{program} is a set of implications of the form $\varphi \leftarrow \psi$
where $\varphi$ is a disjunction of literals and $\psi$ is a conjunction of literals.

We represent an \emph{interpretation} $T$ as a set of atoms $T\subseteq\alphabet$.
An \HT-\emph{interpretation} is a pair $\tuple{H,T}$
of interpretations such that $H\subseteq T$;
it is said to be \emph{total} if $H=T$.
An HT-interpretation $\tuple{H,T}$ \emph{satisfies} a formula~$\varphi$,
written $\tuple{H,T}\models\varphi$,
if
\begin{enumerate}
\item $\tuple{H,T}\models a$ if $a \in H$
\item $\tuple{H,T}\models \varphi \wedge \psi$ if $\tuple{H,T} \models \varphi$ and $\tuple{H,T} \models \psi$
\item $\tuple{H,T}\models \varphi \vee   \psi$ if $\tuple{H,T} \models \varphi$ or  $\tuple{H,T} \models \psi$
\item $\tuple{H,T}\models \varphi \to    \psi$ if $\tuple{H',T} \not\models \varphi$ or $\tuple{H',T} \models \psi$
  for each~$H'\in\{H, T\}$
\end{enumerate}
An \HT-interpretation $\tuple{H,T}$ is an \emph{\HT-model} of a theory $\Gamma$
if $\tuple{H,T}\models\varphi$ for each~$\varphi\in\Gamma$.
A total model~$\tuple{T,T}$ of a theory~$\Gamma$ is an \emph{equilibrium model}
if there is no other model~$\tuple{H,T}$ of~$\Gamma$ with~$H \subset T$;
$T$ is also called a \emph{stable model} of $\Gamma$.
 \subsection{The Logic of Here-and-There with Constraints \citep{cakaossc16a}}
\label{sec:htc}
The syntax of $\HTC$ relies on a signature~$\tuple{\mathcal{X},\mathcal{D},\mathcal{C}}$,
similar to constraint satisfaction problems,
where elements of set~$\mathcal{X}$ represent variables and elements of~$\mathcal{D}$ are domain values
(usually identified with their respective constants).
The constraint atoms in $\mathcal{C}$
provide an abstract way to relate values of variables and constants according to the atom's semantics.
For instance, \emph{difference constraint atoms} are expressions of the form `$x - y \leq d$',
containing variables~$x,y\in\mathcal{X}$ and the domain value~$d\in\mathcal{D}$.
A constraint formula $\varphi$ over $\mathcal{C}$ is defined as
\begin{align*}
  \varphi &::= \bot \mid c \mid \varphi\wedge\varphi \mid  \varphi\vee\varphi \mid  \varphi\rightarrow \varphi
  \quad\text{ where } c\in\mathcal{C}
\end{align*}
Concepts like defined operators, programs, theories, etc.\ are analogous to \HT.
Variables can be assigned some value from $\mathcal{D}$ or left \emph{undefined}.
For the latter, we use the special symbol $\undefined\not\in\mathcal{D}$ and
the extended domain $\mathcal{D}_u \eqdef \mathcal{D} \cup \{\undefined\}$.
A \emph{valuation} $v$ is a function $v:\mathcal{X}\rightarrow\mathcal{D}_u$.
We let $\mathbb{V}$ stand for the set of all valuations.
We sometimes represent a valuation $v$ as a set
\(
\{ (x,v(x)) \mid x\in\mathcal{X}, v(x)\in\mathcal{D}\}
\)
of pairs, so that $(x,\undefined)$ is never formed.
This allows us to use standard set inclusion, $v\subseteq v'$, for comparing $v,v'\in\mathbb{V}$.

The semantics of constraint atoms is defined in \HTC\ via \emph{denotations},
that is, functions
\(
\den{\cdot}:\mathcal{C}\rightarrow 2^{\mathbb{V}}
\)
mapping each constraint atom to a set of valuations.
An \emph{\HTC-interpretation} is a pair $\langle h,t \rangle$
of valuations such that $h\subseteq t$;
it is \emph{total} if $h=t$.
Given a denotation $\den{\cdot}$,
an \HTC-interpretation $\langle h,t \rangle$ \emph{satisfies} a constraint formula~$\varphi$,
written $\langle h,t \rangle \models \varphi$,
if\footnote{Since we use strict denotations, it is sufficient to check in Condition~\ref{sat:htc:atom} satisfaction in $h$.}
\begin{enumerate}
\item $\langle h,t \rangle \models c \text{ if } h\in \den{c}$\label{sat:htc:atom}
\item $\langle h,t\rangle \models \varphi\wedge\psi \text{ if } \langle h,t\rangle \models \varphi \text{ and } \langle h,t\rangle \models \psi$
\item $\langle h,t\rangle \models \varphi\vee\psi   \text{ if } \langle h,t\rangle \models \varphi \text{ or }  \langle h,t\rangle \models \psi$
\item $\langle h,t\rangle \models \varphi \rightarrow \psi
  \text{ if }\langle v,t\rangle \not\models \varphi
  \text{ or }\langle v,t\rangle     \models \psi
  \text{ for each }v\in\{h,t\}$
\end{enumerate}
An \HTC-interpretation~$\tuple{h,t}$ is an \emph{\HTC-model} of a theory~$\Gamma$
if $\tuple{h,t}\models\varphi$ for every~$\varphi\in\Gamma$.
A total model~$\langle t,t\rangle$ of a theory~$\Gamma$ is a \emph{constraint equilibrium model}
if there is no other model~$\tuple{h,t}$ of~$\Gamma$ with~$h\subset t$.
 \subsection{The Metric Temporal Logic of Here-and-There\\ \citep{cadiscsc20a,becadiscsc24a}}
\label{sec:mht}
Given an alphabet \alphabet\ and
\(
\allI\eqdef\{ \intervco{m}{n}\mid m\in\mathbb{N}, n\in\mathbb{N}\cup\{\omega\}\}
\),
a metric temporal formula is defined as\footnote{More general formulas, including until, release and past operators, are presented in~\citep{becadiscsc24a}.}
\begin{align*}
  \varphi &::= \bot \mid a \mid \varphi\wedge\varphi \mid \varphi\vee\varphi \mid \varphi\to\varphi \mid
  \initially                     \mid
  \metricI{\next}        \varphi \mid
  \metricI{\alwaysF}     \varphi \mid
  \metricI{\eventuallyF} \varphi
  \quad\text{ where } a\in\alphabet, I\in\allI
\end{align*}
The last three cases deal with metric temporal operators,
which are indexed by some interval $I$.
In words,
\metricI{\next},
\metricI{\alwaysF}, and
\metricI{\eventuallyF}
are called
\emph{next},
\emph{always}, and
\emph{eventually}.
\initially\ simply refers to the initial state.
We write
$\next$, $\alwaysF$, $\eventuallyF$ for
$\next_{\intervco{0}{\omega}}$, $\alwaysF_{\intervco{0}{\omega}}$, $\eventuallyF_{\intervco{0}{\omega}}$, respectively.
In addition to the aforedefined Boolean operators,
we define
$\finally   \eqdef  \neg \metric{\next}{0}{\omega}     \top$,
which allow us to refer to the final state. Concepts like programs, theories, etc.\ are analogous to \HT.

The semantics of temporal formulas is defined via \emph{traces},
being sequences $\Ttrace=(T_i)_{\rangeco{i}{0}{\lambda}}$ of interpretations $T_i\subseteq\alphabet$;
$\lambda$ is the length of \Ttrace.
Here, we consider only traces of finite length. We define the ordering $\Htrace\leq\Ttrace$ between traces of the same length $\lambda$ as $H_i\subseteq T_i$ for each $\rangeco{i}{0}{\lambda}$,
and $\Htrace<\Ttrace$ as both $\Htrace\leq\Ttrace$ and $\Htrace\neq\Ttrace$.
An \emph{\HT-trace} over \alphabet\ of length $\lambda$ is
a sequence of pairs
\(
(\tuple{H_i,T_i})_{\rangeco{i}{0}{\lambda}}
\)
with $H_i\subseteq T_i$ for any $0 \leq i < \lambda$;
it is \emph{total} if $H_i=T_i$.
For convenience, we represent it as the pair $\tuple{\Htrace,\Ttrace}$ of traces
$\Htrace = (H_i)_{\rangeco{i}{0}{\lambda}}$ and $\Ttrace = (T_i)_{\rangeco{i}{0}{\lambda}}$.

Metric information is captured by timing functions.
Given $\lambda\in\mathbb{N}$,
we say that
\(
\tmf: \intervco{0}{\lambda} \to \mathbb{N}
\)
is a (strict) \emph{timing function} wrt $\lambda$
if
$\tmf(0)=0$ and
$\tmf(\kvar)<\tmf(\kvar+1)$ for $0\leq\kvar<\lambda-1$.
A \emph{timed} \HT-trace $(\tuple{\Htrace,\Ttrace},\tmf)$ over \alphabet\ and $(\mathbb{N},<)$ of length $\lambda$
is a pair consisting of
an \HT-trace $\tuple{\Htrace,\Ttrace}$ over \alphabet\ of length $\lambda$ and
a timing function $\tmf$ wrt $\lambda$.
A timed \HT-trace $\M=(\tuple{\Htrace,\Ttrace}, \tmf)$
of length $\lambda$ over alphabet \alphabet\
\emph{satisfies} a metric formula $\varphi$ at $\rangeco{k}{0}{\lambda}$,
written \mbox{$\M,k \models \varphi$}, if
\begin{enumerate}
\item $\M,k \models a$ if $a \in H_k$
\item \label{def:mhtsat:and} $\M, k \models \varphi \wedge \psi$
  if
  $\M, k \models \varphi$
  and
  $\M, k \models \psi$
\item \label{def:mhtsat:or} $\M, k \models \varphi \vee \psi$
  if
  $\M, k \models \varphi$
  or
  $\M, k \models \psi$
\item $\M, k \models \varphi \to \psi$
  if
  $\M', k \not \models \varphi$
  or
  $\M', k \models  \psi$,
  \\ for both $\M'=\M$ and $\M'=(\tuple{\Ttrace,\Ttrace}, \tmf)$
\item $\M, k \models \initially$
  if
  $k =0$
\item \label{def:mhtsat:next}$\M, k \models \metricI{\next}\, \varphi$
  if
  $k+1<\lambda$ and $\M, k{+}1 \models \varphi$
  and $\tmf(k{+}1)-\tmf(k) \in \cI$
\item \label{def:mhtsat:eventuallyF} $\M, k \models \metricI{\eventuallyF}\, \varphi$
  if
  $\M, i \models \varphi$ for some $i$
  with
  $\trangeco{i}{k}{\lambda}$ and
  $\tmf(i)-\tmf(k) \in \cI$
\item \label{def:mhtsat:alwaysF} $\M, k \models \metricI{\alwaysF}\, \varphi$
  if
  $\M, i \models \varphi$ for all $i$
  with
  $\trangeco{i}{k}{\lambda}$ and
  $\tmf(i)-\tmf(k) \in \cI$
\end{enumerate}
\comment{Finish fixin intervals}
A timed \HT-trace $\M$ is an \emph{\MHT-model} of a metric theory $\Gamma$ if $\M,0 \models \varphi$ for all $\varphi \in \Gamma$.
A total \MHT-model~$(\tuple{\Ttrace, \Ttrace}, \tmf)$ of a theory~$\Gamma$ is a \emph{metric equilibrium model}
if there is no other model $(\tuple{\Htrace,\Ttrace}, \tmf)$ of~$\Gamma$ with $\Htrace<\Ttrace$.
 \par
\newcommand{\distloc}{\ensuremath{\delta\langle L,L'\rangle}}
\newcommand{\tat}{\textit{at}}
\newcommand{\tram}{\textit{ram}}
\newcommand{\tgo}{\textit{go}}
\newcommand{\thas}{\textit{has}}
\newcommand{\toffice}{\textit{office}}
\newcommand{\thome}{\textit{home}}
\newcommand{\tdentist}{\textit{dentist}}
\newcommand{\tatm}{\textit{atm}}
\newcommand{\tcard}{\textit{card}}
\newcommand{\tcash}{\textit{cash}}
\newcommand{\ticard}{\textit{icard}}
\newcommand{\tgoal}{\textit{goal}}
For illustration,
let us consider the formalization of the dentist scenario in~\eqref{ex:dentist:one} to~\eqref{ex:dentist:ten}.
We assume that variables $L$ and $L'$ are substituted by distinct locations $\toffice$, $\tatm$, $\tdentist$, and $\thome$;
and variable $I$ by items $\tcash$ and $\ticard$.
We use \distloc\ to refer to the distance between two locations from Table~\ref{table:dentist}.
As in~\citep{mellarkod07a},
we assume that Ram is automatically picking up items when being at the same position. \begin{align}
\alwaysF(\tat(\tram,\toffice)                                   &\leftarrow \initially)\label{ex:dentist:one}\\
  \alwaysF(\tat(\tcash,\tatm)                                     &\leftarrow \initially)\\
  \alwaysF(\tat(\ticard,\thome)                                   &\leftarrow \initially)\label{ex:dentist:tri}\\
\alwaysF(\textstyle\bigvee_{L'\neq L} \tgo(\tram,L')            &\leftarrow \tat(\tram,L)\wedge \neg\finally)\label{ex:dentist:for}\\
\alwaysF(\thas(\tram,I)                                         &\leftarrow \tat(\tram,L) \wedge \tat(I,L))\label{ex:dentist:six}\\
  \alwaysF(\tat(I,L)                                              &\leftarrow \tat(\tram,L) \wedge \thas(\tram,I))\label{ex:dentist:svn}\\
\alwaysF(\Next_{\intervco{\distloc}{\distloc+1}}{\tat(\tram,L')}&\leftarrow\tat(\tram,L) \wedge \tgo(\tram,L'))\label{ex:dentist:eit}\\
\alwaysF(\Next{\thas(\tram,I)}                                  &\leftarrow \thas(\tram,I) \wedge \neg\finally)\label{ex:dentist:nin}\\
  \alwaysF(\Next{\tat(I,L)}                                       &\leftarrow \neg \thas(\tram,I) \wedge \tat(I,L) \wedge \neg\finally)\label{ex:dentist:ten}
\end{align}
In brief,
Rules~\eqref{ex:dentist:one} to~\eqref{ex:dentist:tri} give the initial situation.
Rule~\eqref{ex:dentist:for} delineates possible actions.
Rules~\eqref{ex:dentist:six} and~\eqref{ex:dentist:svn} capture indirect effects.
Rule~\eqref{ex:dentist:eit} is the effect of moving from location $L$ to $L'$:
it uses the next operator restricted by the duration between locations.
Rules~\eqref{ex:dentist:nin} and~\eqref{ex:dentist:ten} address inertia.

 \section{Plain Metric Logic Programs}
\label{sec:mlp}

Metric logic programs,
defined as metric theories composed of implications akin to logic programming rules,
derive their semantics from their metric equilibrium models.
Syntactically,
a \emph{plain metric logic program}
over \alphabet\
is a set of \emph{plain metric rules} of form
\[
  \melrule{\head\leftarrow\body}
  \quad\text{ or }\quad
  \melrule{\metricI{\Next} a\leftarrow\body}
\]
for
$\head\eqdef a_1\vee  \dots\vee   a_m\vee  \neg a_{m+1}\vee  \dots\vee  \neg a_n$,
$\body\eqdef b_1\wedge\dots\wedge b_o\wedge\neg b_{o+1}\wedge\dots\wedge\neg b_p$, and
$0\leq m\leq n\leq o\leq p$
with
$a,a_i\in\alphabet$
for
$1\leq i\leq n$,
and
$b_i\in\alphabet\cup\{\initially,\finally\}$
for
$1\leq i\leq p$.

While our considered language fragment excludes global temporal operators and disjunctive metric heads,
it effectively captures state transitions and allows for imposing timing constraints upon them.
A comprehensive treatment is provided in Section~\ref{sec:gmlp}.

Our two alternative translations share a common structure, each divided into three distinct parts.
The first part maps a plain metric program into a regular one.
This part captures the state transitions along an \HT-trace specified by the metric program,
and is common to both translations.
The second and third parts capture the timing function along with
its interplay with the interval constraints imposed by the metric program, respectively.
The two variants of these parts are described in Section~\ref{sec:mlp:ht} and~\ref{sec:mlp:htc} below.

The first part of our translation takes a plain metric program over \alphabet\
and yields rules over
\(
\alphabets\eqdef\bigcup\nolimits_{\kvar\in\mathbb{N}} \alphabetk
\)
for
\(
\alphabetk\eqdef\{a_\kvar \mid a \in \alphabet \}
\)
and $\kvar \in \mathbb{N}$.
Atoms of form $a_k$ in $\alphabets$ represent the values taken by variable $a\in\alphabet$
at different points $k$ along a trace of length $\lambda$.

We begin by inductively defining the translation $\tk{r}$ of a plain metric rule $\alwaysF{r}$ at $k$
for $0 \leq \kvar < \lambda$ and $\lambda\in\mathbb{N}$
as follows:\footnote{Note that $\top$, $\neg$ and $\finally$ are defined operators.}
\begin{align*}
\tk{\bot}             &\eqdef \bot\\
  \tk{a}                &\eqdef a_\kvar \text{ if } a \in \alphabet \\
  \tk{\metricI{\next}a} &\eqdef
                          \begin{cases}
                            \bot    & \text{if } \kvar = \lambda-1 \\
                            \trans{a}{\kvar+1}        & \text{otherwise}
                          \end{cases} \\
  \tk{\initially}       &\eqdef
                          \begin{cases}
                            \top        & \text{if } \kvar = 0 \\
                            \bot        & \text{otherwise}
                          \end{cases}  \\
\tk{\varphi_1\wedge\varphi_2}&\eqdef \tk{\varphi_1}\wedge\tk{\varphi_2}  \\
  \tk{\varphi_1\vee\varphi_2}  &\eqdef \tk{\varphi_1}\vee\tk{\varphi_2}    \\
  \tk{\head \leftarrow\body}   &\eqdef\{\tk{\varphi_1}\leftarrow\tk{\varphi_2}\}
\end{align*}
Note that we drop the always operator \alwaysF{} preceding rules in the translation;
it is captured by producing a rule instance for every $0 \leq \kvar < \lambda$.
Accordingly, for a plain metric program \program\ over \alphabet\ and $\lambda\in\mathbb{N}$, we define
\begin{align*}
  \Pi_\lambda(\program)
  &=
  \textstyle\bigcup_{\alwaysF{r} \in P,\, 0 \leq \kvar < \lambda} \tk{r}
  \quad\text{ over }
  \alphabets.
\end{align*}

For illustration,
consider the instance of \eqref{ex:dentist:eit} for moving from $\toffice$ to $\thome$
\begin{align}\label{ex:dentist:eit:translated}
  \alwaysF(\Next_{\intervco{15}{16}}{\tat(\tram,\thome)}&\leftarrow\tat(\tram,\toffice) \wedge \tgo(\tram,\thome)) \ .
\end{align}
Our translation ignores \alwaysF\ at first and yields:
\begin{align}
                      \bot&\leftarrow\tat(\tram,\toffice)_k \wedge \tgo(\tram,\thome)_k\quad\text{ for }\kvar = \lambda-1\label{ex:dentist:eit:translated:one}\\
  \tat(\tram,\thome)_{k+1}&\leftarrow\tat(\tram,\toffice)_k \wedge \tgo(\tram,\thome)_k\quad\text{ otherwise}\label{ex:dentist:eit:translated:two}
\end{align}
When assembling $\Pi_\lambda(\program)$ for $\lambda=100$ and $P$ being the rules in~\eqref{ex:dentist:one}
to~\eqref{ex:dentist:ten}, we account for \alwaysF\ by adding 99 instances of the rule
in~\eqref{ex:dentist:eit:translated:two}
and a single instance of~\eqref{ex:dentist:eit:translated:one}.

This first part of our translation follows Kamp's translation~\citep{kamp68a} for Linear temporal logic.
Of particular interest is the translation of $\metricI{\Next} a\leftarrow\body$.
The case analysis accounts for the actual state transition of the next operator,
which is infeasible at the end of the trace.
Thus, we either derive $a_{k+1}$ or a contradiction.
The metric aspect is captured by the translations in Section~\ref{sec:mlp:ht} and~\ref{sec:mlp:htc}.
Whenever all intervals in a plain metric programs $P$ are of form \intervco{0}{\omega},
we get a one-to-one correspondence between
\MHT-traces of length $\lambda$ of $P$ with an arbitrary yet fixed timing function and
\HT-interpretations of $\Pi_\lambda(\program)$.
\comment{REV: This links to our temporal bases results}
Finally, it is worth noting that the size of the resulting program $\Pi_\lambda(\program)$ grows with $\lambda$.
 \subsection{Translation of Plain Metric Logic Programs to \HT}
\label{sec:mlp:ht}

We begin by formalizing timing functions $\tau$ via Boolean atoms in
\(
\alphabetT\eqdef\{\timet_{\kvar,\tmvar} \mid \kvar,\tmvar\in\mathbb{N} \}.
\)
An atom like $\timet_{\kvar,\tmvar}$ represents that $\tau(\kvar)=\tmvar$.
To obtain finite theories,
we furthermore impose an upper bound $\tmflimit\in\mathbb{N}$ on the range of $\tau$.
Hence, together with the trace length $\lambda$,
our formalization $\Delta_{\lambda,\tmflimit}$ only captures
timing functions $\tau$ satisfying $\tmf(\lambda-1)\leq\tmflimit$.

Given $\lambda,\tmflimit \in \mathbb{N}$,
we let
\begin{align}\label{def:ht:delta}
  \Delta_{\lambda,\tmflimit} &\eqdef
     \{ \timet_{0,0} \} \cup
     \{ \textstyle\bigvee_{d<d'\leq\tmflimit}\timet_{\kvar+1,\tmvar'}\leftarrow\timet_{\kvar,\tmvar} \mid
            0 \leq \kvar < \lambda-1, 0 \leq \tmvar \leq \tmflimit\}
\end{align}
Starting from $\tau(0)=0$, represented by $\timet_{0,0}$,
the rule in~\eqref{def:ht:delta} assigns strictly increasing time points to consecutive states,
reflecting that $\tmf(\kvar)<\tmf(\kvar+1)$ for $0\leq\kvar<\lambda-1$.

The last part of our formalization accounts for the interplay of the timing function with the interval conditions imposed in the program.
Given $\lambda,\tmflimit \in \mathbb{N}$
and
a plain metric program $\program$,
we let
\begin{align}
  \Psi_{\lambda,\tmflimit}(\program) \eqdef
  &\ \{ \bot \leftarrow \tk{\body} \wedge \timet_{\kvar,\tmvar} \wedge \timet_{\kvar+1,\tmvar'} \mid
    0 \leq \kvar < \lambda-1, 0\leq d<d'\leq\tmflimit,
    \label{def:ht:psi:one}\\&\qquad\qquad\nonumber
  \phantom{n\in\mathbb{N}, }\tmvar'-\tmvar < m, \melrule{\metric{\next}{m}{n} a \leftarrow \body} \in \program \}\;\cup \\
  &\ \{ \bot \leftarrow \tk{\body} \wedge \timet_{\kvar,\tmvar} \wedge \timet_{\kvar+1,\tmvar'} \mid
    0 \leq \kvar < \lambda-1, 0\leq d<d'\leq\tmflimit,
    \label{def:ht:psi:two}\\&\qquad\qquad\nonumber
           n\in\mathbb{N},  \tmvar'-\tmvar \geq n, \melrule{\metric{\next}{m}{n} a \leftarrow \body} \in \program \}
\end{align}
The integrity constraints ensure that for every plain metric rule $\melrule{\metric{\next}{m}{n} a \leftarrow \body}$
the duration between the $k$th and $(k+1)$st state in a trace
falls within interval $\intervco{m}{n}$.
With $\tmf(\kvar) =d$ and $\tmf(\kvar+1) =d'$,
this amounts to checking whether
\(
\tmvar+m \leq \tmvar' < \tmvar+n
\),
if $n$ is finite;
otherwise, the verification of the upper bound in~\eqref{def:ht:psi:two} is dropped for $n=\omega$.

Note that the size of both $\Delta_{\lambda,\tmflimit}$ and $\Psi_{\lambda,\tmflimit}(\program)$
is proportional to $\mathcal{O}(\lambda\cdot\tmflimit^2)$.
Hence, long traces and even more severely fine-grained timing functions lead to a significant blow up
when translating plain metric programs into regular ones with the above formalization.

For the rule in~\eqref{ex:dentist:eit:translated},
we get
\begin{align}
  \bot&\leftarrow\tat(\tram,\toffice)_k\wedge\tgo(\tram,\thome)_k\wedge\timet_{\kvar,\tmvar}\wedge\timet_{\kvar+1,\tmvar'}
        \quad\text{ for } d'-d < 15
  \\
  \bot&\leftarrow\tat(\tram,\toffice)_k\wedge\tgo(\tram,\thome)_k\wedge\timet_{\kvar,\tmvar}\wedge\timet_{\kvar+1,\tmvar'}
        \quad\text{ for } d'-d \geq 16
\end{align}
and
$0\leq\kvar <\lambda-1$,
$0\leq\tmvar<\tmvar'\leq\tmflimit$.
For $\lambda=100$ and $\tmflimit=1000$,
this then amounts to roughly $10^8$ instances for each of the above constraints.

In what follows,
we characterize the effect of our formalization in terms of \HT-models,
and ultimately show the completeness and correctness of our translation.
\begin{definition}\label{def:ht:timed}
  An \HT\ interpretation \handt\
  over \alphabet\ with $\mathcal{T}\subseteq\alphabet$,
  is \emph{timed} wrt $\lambda\in\mathbb{N}$,

  if
  there is a timing function \tmf\ wrt $\lambda$
  such that
for all $0\leq\kvar<\lambda$,
  $\tmvar \in \mathbb{N}$,
  we have
  \begin{align*}
    \tuple{H,T}\models\timet_{\kvar,\tmvar}\text{ iff }\tmf(\kvar)=\tmvar
    \quad\text{ and }\quad
    \tuple{T,T}\models\timet_{\kvar,\tmvar}\text{ iff }\tmf(\kvar)=\tmvar
  \end{align*}
\end{definition}
We also call $\tmf$ the timing function induced by \handt.
\begin{proposition}\label{prop:ht:delta:timed}
  Let
  $\lambda,\tmflimit\in\mathbb{N}$.

  If $\tandt$ is an equilibrium model of $\Delta_{\lambda,\tmflimit}$
  then $\tandt$ is timed wrt $\lambda$.
\end{proposition}
\begin{proposition}\label{prop:ht:timed:delta}
  Let $\handt$ be an \HT\ interpretation
  and
  $\lambda\in\mathbb{N}$.

  If $\handt$ is timed wrt $\lambda$ and induces $\tmf$,
  then $\handt$ is an \HT-model of $\Delta_{\lambda,\tmflimit}$
  for $\tmflimit=\tmf(\lambda-1)$.
\end{proposition}
Clearly, the last proposition extends to equilibrium models.

Given a timed \HT-trace $\M = (\tuple{H_\kvar,T_\kvar}_{\rangeco{\kvar}{0}{\lambda}},\tmf)$ of length $\lambda$ over \alphabet,
we define $\mhtToHt{\M}$ as \HT\ interpretation
\(
\tuple{H\cup X ,T\cup X }
\)
where
\begin{align*}
  H &=\{a_\kvar\in\alphabetk \mid 0 \leq \kvar < \lambda, a\in H_\kvar \} &
  T &=\{a_\kvar\in\alphabetk \mid 0 \leq \kvar < \lambda, a\in T_\kvar \} \\
  X &=\{\timet_{\kvar, \tmvar} \mid \tmf(\kvar)=\tmvar, 0\leq\kvar<\lambda, \tmvar\in\mathbb{N}\}
\end{align*}
Note that $\mhtToHt{\M}$ is an \HT\ interpretation timed wrt $\lambda$.

Conversely,
given an \HT\ interpretation $\tuple{H,T}$ timed wrt $\lambda$ over $\alphabets\cup\alphabetT$
and its induced timing function $\tmf$,
we define $\htToMht{\tuple{H,T}}$ as
the timed \HT-trace
\begin{align*}
  (\tuple{ \{a\in\alphabet\mid a_\kvar\in H\}, \{a\in\alphabet\mid a_\kvar\in T\}}_{\rangeco{\kvar}{0}{\lambda}}, \tmf)
\end{align*}
In fact, both functions $\htToMhtf$ and $\mhtToHtf$ are invertibles,
and we get a one-to-one correspondence between \HT\ interpretations timed wrt $\lambda$ and timed \HT-traces of length $\lambda$.

Finally, we have the following completeness and correctness result.
\begin{theorem}[Completeness]\label{thm:mlp:ht:completeness}
  Let \program\ be a plain metric logic program
  and
  $\M=(\tuple{\Ttrace,\Ttrace}, \tmf)$ a total timed \HT-trace of length $\lambda$.

  If
  $\M$ is a metric equilibrium model of $\program$,
  then
  $\mhtToHt{\M}$ is an equilibrium model of $\Pi_\lambda(\program)\cup\Delta_{\lambda,\tmflimit}\cup\Psi_{\lambda,\tmflimit}(\program)$
  with $\tmflimit=\tmf(\lambda-1)$.
\end{theorem}
\begin{theorem}[Correctness]\label{thm:mlp:ht:correctness}
  Let
  \program\ be a plain metric logic program,
  and
  $\lambda,\tmflimit \in \mathbb{N}$.

  If
  $\tandt$ is an equilibrium model of $\Pi_\lambda(\program)\cup\Delta_{\lambda,\tmflimit}\cup\Psi_{\lambda,\tmflimit}(\program)$,
  then
  $\htToMht{\tandt}$ is a metric equilibrium model of $\program$.
\end{theorem}
 \subsection{Translation of Plain Metric Logic Programs to \HTC}
\label{sec:mlp:htc}

We now present an alternative, refined formalization of the second and third parts,
utilizing integer variables and difference constraints to capture the timing function more effectively.
To this end,
we use the logic of \HTC\ to combine the Boolean nature of ASP with constraints on integer variables.

Given base alphabet $\alphabet$ and $\lambda\in\mathbb{N}$,
we consider the \HTC\ signature $\tuple{\mathcal{X},\mathcal{D},\mathcal{C}}$
where\footnote{In \HTC~\citep{cakaossc16a},
  Boolean variables are already captured by truth values \true\ and \undefined\ (rather than \false[alse]).}
\begin{align*}
  \mathcal{X} &= \alphabets \cup \{ \timet_k \mid 0\leq k<\lambda\} \\
  \mathcal{D} &= \{\true\}\cup\mathbb{N}\\
  \mathcal{C} &= \{a=\true\mid a\in\alphabets\}
                 \cup \{\x=\tmvar \mid \x    \in \mathcal{X}\setminus\alphabets,\tmvar\in \mathbb{N} \}\;\cup\\
              &\quad\;\{\x-\y\leq \tmvar \mid \x,\y \in \mathcal{X}\setminus\alphabets\in\mathbb{N}, \tmvar\in \mathbb{Z} \}
\end{align*}
Rather than using Boolean variables,
this signature represents timing functions $\tau$ directly by integer variables $t_k$,
capturing that $\tau(k)=t_k$ for $0\leq k<\lambda$.
This is enforced by the integer constraints in $\mathcal{C}$,
whose meaning is defined by the following denotations:
\begin{align*}
  \den{a=\true}          &= \{v\in\mathbb{V}\mid v(a)=\true\} \qquad\qquad\qquad\text{ for all }a\in \alphabets\\
  \den{\x=\tmvar}        &= \{v\in\mathbb{V}\mid v(\x),       \tmvar\in\mathbb{N},\ v(\x)= \tmvar\}\\
  \den{\x-\y\leq \tmvar} &= \{v\in\mathbb{V}\mid v(\x),v(\y)\in\mathbb{N},\tmvar\in\mathbb{Z},\ v(\x)-v(\y)\leq \tmvar\}
\end{align*}
When dealing with Boolean variables, we simplify notation by representing $a=\true$ as $a$.

This leads us to the the following counterpart of $\Delta_{\lambda,\tmflimit}$ in~\eqref{def:ht:delta}.
Given $\lambda \in \mathbb{N}$, we define
\begin{align}\label{def:htc:delta}
  \Delta^{c}_{\lambda} &= \{ \timet_{0} = 0 \} \cup \{ \timet_{\kvar}-\timet_{\kvar{+}1} \leq -1 \mid 0 \leq \kvar < \lambda-1 \}
\end{align}
Starting from $\timet_0=0$,
the difference constraints in~\eqref{def:htc:delta} enforce that $\timet_{\kvar}<\timet_{\kvar{+}1}$
reflecting that $\tau(0)=0$ and $\tmf(\kvar)<\tmf(\kvar+1)$ for $0\leq\kvar<\lambda-1$.
Moreover, $\Delta^{c}_{\lambda}$ is unbound and thus imposes no restriction on timing functions.
And no variable $\timet_{\kvar}$ is ever undefined:
\begin{proposition}\label{pro:htc:defined}
Let $\htcinterp$ be an \HTC\ interpretation and $\lambda \in \mathbb{N}$

If $\htcinterp$ is an \HTC-model of $\Delta^{c}_{\lambda}$, then $\Vh(\timet_{\kvar})\in\mathbb{N}$ for all $0\leq\kvar<\lambda$.
\end{proposition}
We also have $\Vt(\timet_{\kvar})\in\mathbb{N}$ by definition of \HTC\ interpretations, that is, since $\Vh\subseteq\Vt$.

Our variant of the third part of our translation re-expresses the ones in~(\ref{def:ht:psi:one}/\ref{def:ht:psi:two})
in terms of integer variables and difference constraints.
Given $\lambda \in \mathbb{N}$ and a plain metric logic program $\program$,
we define
\begin{align}
  \Psi^{c}_{\lambda}(\program) \eqdef
  &\ \{ \bot \leftarrow \tk{\body} \wedge \neg (\timet_{\kvar}-\timet_{\kvar+1} \leq -m) \mid
    0 \leq \kvar < \lambda-1,
    \label{def:htc:psi:one}\\&\qquad\qquad\nonumber
  \phantom{n\in\mathbb{N},} \melrule{\metric{\next}{m}{n} a \leftarrow \body} \in \program\} \; \cup \\
  &\ \{ \bot \leftarrow \tk{\body} \wedge \neg (\timet_{\kvar+1} - \timet_{\kvar} \leq n - 1) \mid
    0 \leq \kvar < \lambda-1,
    \label{def:htc:psi:two}\\&\qquad\qquad\nonumber
           n\in\mathbb{N},  \melrule{\metric{\next}{m}{n} a \leftarrow \body} \in \program\}
\end{align}

In fact, both $\Delta^{c}_{\lambda}$ and $\Psi^{c}_{\lambda}(\program)$ drop the upper bound on the range of a
timing function, as required in their Boolean counterparts.
Hence, their size is only proportional to $\mathcal{O}(\lambda)$,
and thus considerably smaller than their purely Boolean counterparts.

For the rule in~\eqref{ex:dentist:eit:translated},
we get
\begin{align}
  \bot&\leftarrow\tat(\tram,\toffice)_k\wedge\tgo(\tram,\thome)_k\wedge\neg (\timet_{\kvar}-\timet_{\kvar+1} \leq -15)
  \\
  \bot&\leftarrow\tat(\tram,\toffice)_k\wedge\tgo(\tram,\thome)_k\wedge\neg (\timet_{\kvar+1} - \timet_{\kvar} \leq 15)
\end{align}
for
$0 \leq \kvar < \lambda-1$.
Given $\lambda=100$, this only amounts to $10^2$ instances.

Mirroring our approach in Section~\ref{sec:mlp:ht},
we capture the meaning of $\Delta^{c}_{\lambda}$ using specialized \HTC-models.
This leads to the completeness and correctness of our translation.
\begin{definition}
  An \HTC\ interpretation \htcinterp\
  over $\tuple{\mathcal{X},\mathcal{D},\mathcal{C}}$
  is \emph{timed} wrt $\lambda$,
if
  there is a timing function \tmf\ wrt $\lambda$
  such that
  \(
  \Vh(\timet_{\kvar})=\tmf(\kvar)
  \)
  and
  \(
  \Vt(\timet_{\kvar})=\tmf(\kvar)
  \)
  for all $0\leq\kvar<\lambda$.
\end{definition}
As above, we call $\tmf$ the timing function induced by \htcinterp.

\begin{proposition}\label{prop:htc:delta:timed}
  Let $\htcinterp$ be an \HTC\ interpretation
  and
  $\lambda\in\mathbb{N}$.

  If $\htcinterp$ is an \HTC-model of $\Delta^{c}_{\lambda}$
  then $\htcinterp$ is timed wrt $\lambda$.
\end{proposition}
\begin{proposition}\label{prop:htc:timed:delta}
  Let $\htcinterp$ be an \HTC\ interpretation
  and
  $\lambda\in\mathbb{N}$.

  If $\htcinterp$ is timed wrt $\lambda$
  then $\htcinterp$ is an \HTC-model of $\Delta^{c}_{\lambda}$.
\end{proposition}
Unlike Proposition~\ref{prop:ht:delta:timed}, the latter refrain from requiring \HTC-interpretations in equilibrium.

Given an \HT-trace $\M = (\tuple{H_\kvar,T_\kvar}_{\trangeco{\kvar}{0}{\lambda}},\tmf)$ of length $\lambda$,
we define $\mhtToHtc{\M}$ as the \HTC\ interpretation
\(
\tuple{h\cup x,t\cup x}
\)
where $h,t,x$ are valuations such that
\begin{align*}
  h &=\{(a_\kvar,\true) \mid 0 \leq \kvar < \lambda, a\in H_\kvar \} &
  t &=\{(a_\kvar,\true) \mid 0 \leq \kvar < \lambda, a\in T_\kvar \} \\
  x &= \{(\timet_{\kvar},\tmvar) \mid \tmf(\kvar)=\tmvar, 0\leq\kvar<\lambda, \tmvar\in\mathbb{N} \}
\end{align*}
Similar to above, $\mhtToHtc{\M}$ is an \HTC\ interpretation timed wrt $\lambda$.

Conversely,
given an \HTC\ interpretation $\htcinterp$ timed wrt $\lambda$
and its induced timing function $\tmf$,
we define $\htcToMht{\htcinterp}$ as
the timed \HT-trace
\begin{align*}
  (\tuple{ \{a\mid\Vh(a_\kvar)=\true\}, \{a\mid\Vt(a_\kvar)=\true\}}_{\trangeco{\kvar}{0}{\lambda}}, \tmf)
\end{align*}
As above, functions $\htToMhtf^c$ and $\mhtToHtf^c$ are invertibles.
Hence,
we get a one-to-one correspondence between \HTC\ interpretations timed wrt $\lambda$ and timed \HT-traces of length $\lambda$.

Finally, we have the following completeness and correctness result.
\begin{theorem}[Completeness]\label{thm:mlp:htc:completeness}
  Let \program\ be a plain metric logic program
  and
  $\M=(\tuple{\Ttrace,\Ttrace}, \tmf)$ a total timed \HT-trace of length $\lambda$.

  If
  $\M$ is a metric equilibrium model of $\program$,
  then
  $\mhtToHtc{\M}$ is a constraint equilibrium model of $\Pi_\lambda(\program)\cup\Delta^c_{\lambda}\cup\Psi^c_{\lambda}(\program)$.
\end{theorem}
\begin{theorem}[Correctness]\label{thm:mlp:htc:correctness}
  Let
  \program\ be a plain metric logic program,
  and
  $\lambda \in \mathbb{N}$.

  If
  $\htcttinterp$ is a constraint equilibrium model of $\Pi_\lambda(\program)\cup\Delta^c_{\lambda}\cup\Psi^c_{\lambda}(\program)$,
  then
  $\htcToMht{\htcttinterp}$ is a metric equilibrium model of $\program$.
\end{theorem}
 \section{General Metric Logic Programs}
\label{sec:gmlp}

In the preceding section,
we introduced a simplified fragment of metric theories and
demonstrated its application in modeling quantitatively constrained transitions for dynamic problems.
While this fragment enables the representation of basic planning and scheduling scenarios,
it lacks the capacity to express variable queries involving global temporal operators.

In this section, we show how general metric theories can be reduced to a more general logic programming format and
how rules in this format can be implemented in analogy to the techniques shown in the previous section.
As in Section~\ref{sec:mlp},
we define metric logic programs as metric theories composed of implications analogous to logic programming rules,
with their semantics determined by their metric equilibrium models.

Syntactically,
a (general) \emph{metric logic program}
over \alphabet\
is a set of (general) \emph{metric rules} of form
\begin{align}\label{def:general:rule}
  \melrule{\head\leftarrow\body}
\end{align}
for
$\head\eqdef \matom_1    \vee  \dots\vee   \matom_m\vee  \neg \matom_{m+1}\vee  \dots\vee  \neg \matom_n$,
$\body\eqdef \matom_{n+1}\wedge\dots\wedge \matom_o\wedge\neg \matom_{o+1}\wedge\dots\wedge\neg \matom_p$, and
$0\leq m\leq n\leq o\leq p$,
where for $1\leq i\leq p$ each $\matom_i$ is a \emph{metric atom} over \alphabet\
defined as
\comment{REV: We can think of using something other than $b$, notice that we will have to say $b \in\alphabet\cup\{\bot,\top\}$ whenever we use $b$}
\begin{align*}
  \matom & ::=
        \bot                     \mid
        \top                     \mid
        a\in\alphabet            \mid
        \initially               \mid
        \metricI{\Next}        b \mid
        \metricI{\alwaysF}     b \mid
        \metricI{\eventuallyF} b
        \quad\text{ where }    b \in\alphabet\cup\{\bot,\top\}\text{ and }I\in\allI
\end{align*}
We include the truth constants, for instance,
to express via $\metricI{\eventuallyF}\top$ that there is a state within interval $I$ or
to define \finally\ as $\neg\metric{\next}{0}{\omega}\top$.
We write $\varphi\in r$ whenever a metric formula $\varphi$ occurs in a rule $r$ as in~\eqref{def:general:rule},
that is, if $\varphi=\matom_i$
for some $1\leq i\leq p$,
or $\varphi=b$ where $\matom_i\in\{\metricI{\Next}b,\metricI{\eventuallyF}b,\metricI{\alwaysF}b \}$
for some $1\leq i\leq p$.

Note that metric atoms only apply temporal operators to positive atoms and Boolean constants.
The next result shows that this is no restriction.
\begin{proposition}
  Given a metric formula $\varphi$,
  we have the following equivalences
  \begin{enumerate}
  \item
    \(
    \metricI{\eventuallyF}\neg \varphi
    \equiv
    \neg \metricI{\alwaysF} \varphi
    \)
  \item
    \(
    \metricI{\alwaysF} \neg \varphi
    \equiv
    \neg \metricI{\eventuallyF} \varphi
    \)
  \item
    \(
    \metricI{\Next} \neg \varphi
    \equiv
    \metricI{\Next} \top \wedge \neg \metricI{\Next} \varphi
    \)
  \end{enumerate}
\end{proposition}

To reduce arbitrary metric formulas to a metric program,
we define a Tseitin-style translation~\citep{tseitin68a} that replaces each subformula with a corresponding fresh atom.
To this end,
we let \subformulas{\varphi} stand for the set of all subformulas of a metric formula $\varphi$ (including $\varphi$ itself).
The fresh variables needed for translating $\varphi$ are given by
\(
\mathcal{L}_\varphi=\{\Label{\phi} \mid \phi \in \subformulas{\varphi}\}
\).
With this,
we define the translation of a metric formula $\varphi$ into a metric logic program $\Theta(\varphi)$ as
\begin{align}\label{eq:tseitin:general}
  \Theta(\varphi) & \eqdef
                    \{\metric{\alwaysF}{0}{\omega} (\initially \rightarrow \Label{\varphi})\}
                    \cup
                    \textstyle\bigcup_{\phi \in \subformulas{\varphi}} \eta^*(\phi)
\end{align}
The definition of the sets $\eta^*(\phi)$ is given in Table~\ref{tab:tseitin:metric:general}.
\comment{REV: Consider renaming $\eta^*(\phi)$}
\begin{table}[htp]
  \begin{center}
    \(
    \begin{array}{|l|l|l|}
      \cline{1-3}
      \bm{\phi} & \bm{\eta(\phi)}  & \bm{\eta^*(\phi)} \\
      \cline{1-3}
      a
                & \alwaysF(\Label{a} \leftrightarrow a)
                                   & \begin{array}{l}
                                     \alwaysF(\Label{a} \rightarrow a) \\
                                     \alwaysF(a \rightarrow \Label{a})
                                   \end{array}                                                                                                 \\
      \cline{1-3}
      \varphi \wedge \psi
                & \alwaysF(\Label{\varphi \wedge \psi} \leftrightarrow \Label{\varphi} \wedge \Label\psi)
                                   & \begin{array}{l}
                                     \alwaysF(\Label{\varphi \wedge \psi} \rightarrow \Label{\varphi}) \\
                                     \alwaysF(\Label{\varphi \wedge \psi} \rightarrow \Label\psi)      \\
                                     \alwaysF(\Label{\varphi} \wedge \Label\psi \rightarrow \Label{\varphi \wedge \psi})
                                   \end{array}
      \\
      \cline{1-3}
      \varphi \vee \psi
                & \alwaysF(\Label{\varphi \vee \psi} \leftrightarrow \Label{\varphi} \vee \Label\psi)
                                   & \begin{array}{l}
                                     \alwaysF(\Label{\varphi} \rightarrow \Label{\varphi \vee \psi}) \\
                                     \alwaysF(\Label\psi \rightarrow \Label{\varphi \vee \psi})      \\
                                     \alwaysF(\Label{\varphi \vee \psi} \rightarrow \Label{\varphi} \vee \Label\psi)
                                   \end{array}
      \\
      \cline{1-3}
      \varphi \rightarrow \psi
                & \alwaysF(\Label{\varphi \rightarrow \psi} \leftrightarrow \Label{\varphi} \rightarrow \Label\psi)
                                   & \begin{array}{l}
                                     \alwaysF(\neg \Label{\varphi} \rightarrow \Label{\varphi \rightarrow \psi})               \\
                                     \alwaysF(\Label\psi \rightarrow \Label{\varphi \rightarrow \psi})                         \\
                                     \alwaysF(\Label{\varphi \rightarrow \psi} \wedge \Label{\varphi} \rightarrow  \Label\psi) \\
                                     \alwaysF( \top \rightarrow \Label{\varphi} \vee \neg \Label{\psi} \vee \Label{\varphi \rightarrow \psi})
                                   \end{array}
      \\
      \cline{1-3}
      \bot & \alwaysF(\Label{\bot} \leftrightarrow \bot) & \;\;\alwaysF(\Label{\bot} \rightarrow \bot  )
      \\
      \cline{1-3}
      \initially
                & \alwaysF(\Label{\initially} \leftrightarrow \initially)
                                   &
                                     \begin{array}{l}
                                       \alwaysF(\Label{\initially} \rightarrow \initially) \\
                                       \alwaysF(\initially \rightarrow \Label{\initially}) \\
                                     \end{array}
      \\
      \cline{1-3}
      \metricI{\Next} \varphi
                & \alwaysF(\Label{\metricI{\Next} \varphi} \leftrightarrow \metricI{\Next} \Label{\varphi})
                                   &
                                     \begin{array}{l}
                                       \alwaysF(\Label{\metricI{\Next}\varphi} \rightarrow \metricI{\Next}\Label\varphi) \\
                                       \alwaysF(\metricI{\Next}\Label\varphi \rightarrow \Label{\metricI{\Next}\varphi}) \\
                                     \end{array}
      \\
      \cline{1-3}
      \metricI{\eventuallyF} \varphi
                & \alwaysF(\Label{\metricI{\eventuallyF} \varphi} \leftrightarrow \metricI{\eventuallyF} \Label{\varphi})
                                   &
                                     \begin{array}{l}
                                       \alwaysF(\Label{\metricI{\eventuallyF}\varphi} \rightarrow \metricI{\eventuallyF}\Label\varphi) \\
                                       \alwaysF(\metricI{\eventuallyF}\Label\varphi \rightarrow \Label{\metricI{\eventuallyF}\varphi}) \\
                                     \end{array}
      \\
      \cline{1-3}
      \metricI{\alwaysF} \varphi
                & \alwaysF(\Label{\metricI{\alwaysF} \varphi} \leftrightarrow \metricI{\alwaysF} \Label{\varphi})
                                   &
                                     \begin{array}{l}
                                       \alwaysF(\Label{\metricI{\alwaysF}\varphi} \rightarrow \metricI{\alwaysF}\Label\varphi) \\
                                       \alwaysF(\metricI{\alwaysF}\Label\varphi \rightarrow \Label{\metricI{\alwaysF}\varphi}) \\
                                     \end{array}
      \\
      \cline{1-3}
    \end{array}
    \)
  \end{center}
  \caption{Definition of sets $\eta(\phi)$ and $\eta^*(\phi)$ for any metric temporal formula $\phi$.}
  \label{tab:tseitin:metric:general}
\end{table}
 Such a Tseitin-style translation adds a fresh symbol for each subformula and puts them into an equivalence.
This is shown on the left hand side of Table~\ref{tab:tseitin:metric:general}.
These equivalences are then translated into implications, which amount to the logic programs in the right column.

\newcommand{\qex}{\metric{\eventuallyF}{0}{60} \tgoal}
For example, consider the following metric formula
enforcing the goal condition of Ram being in the office with both items within one hour.
\begin{align}\label{eq:tseitin:general:example}
(\neg\neg\qex) \wedge ((\tat(\tram, \toffice) \wedge \thas(\tram, \tcash) \wedge \thas(\tram, \ticard))\rightarrow \tgoal)
\end{align}
We apply double negation to the goal,
which is equivalent to \( (\qex \to \bot) \to \bot \),
to ensure that this condition is inferred indirectly by other formulas, rather than directly entailed.
In what follows we will abbreviate formula \eqref{eq:tseitin:general:example} as $q$
and $\tat(\tram, \toffice) \wedge \thas(\tram, \tcash) \wedge \thas(\tram, \ticard)$ as $g$ for brevity.

The metric logic program capturing the formula \eqref{eq:tseitin:general:example} is then given by:
\[
\Theta(q) = \{\metric{\alwaysF}{0}{\omega} (\initially \rightarrow \Label{q})\} \cup Q^*
\]
where \( Q^* \) is derived by translating equivalences in \( Q \) into implications, as illustrated in Table~\ref{tab:tseitin:metric:general}.
\begin{align*}
    Q = & \{\alwaysF(\Label{a} \leftrightarrow a) \mid a \in \{\tgoal, \tat(\tram, \toffice), \thas(\tram, \tcash), \thas(\tram, \ticard)\}\} \cup \\
    & \{\alwaysF(\Label{\bot} \rightarrow \bot), \\
    & \;\; \alwaysF(\Label{\tat(\tram, \toffice) \wedge \thas(\tram, \tcash)} \leftrightarrow \Label{\tat(\tram, \toffice)} \wedge \Label{\thas(\tram, \tcash)}), \\
    & \;\; \alwaysF(\Label{g} \leftrightarrow \Label{\tat(\tram, \toffice) \wedge \thas(\tram, \tcash)} \wedge \Label{\thas(\tram, \ticard)}), \\
    & \;\; \alwaysF(\Label{g\rightarrow \tgoal} \leftrightarrow \Label{g} \to \Label{\tgoal}), \\
    & \;\; \alwaysF(\Label{\metric{\eventuallyF}{0}{60}\,\tgoal} \leftrightarrow \metric{\eventuallyF}{0}{60} \Label{\tgoal}), \\
    & \;\; \alwaysF(\Label{\metric{\eventuallyF}{0}{60}\,\tgoal \rightarrow \bot} \leftrightarrow \Label{\metric{\eventuallyF}{0}{60}\,\tgoal} \rightarrow \Label{\bot}), \\
    & \;\; \alwaysF(\Label{q} \leftrightarrow (\Label{\metric{\eventuallyF}{0}{60}\,\tgoal \to \bot}) \rightarrow \Label{\bot}) \}
\end{align*}

While program \(\Theta(q)\) correctly captures the goal condition,
it can be simplified
and directly expressed as a metric logic program
as follows:
\comment{REV: Elaborate here on why they are equivalent}
\begin{align}
    \alwaysF(\tgoal                                 &\leftarrow \tat(\tram, \toffice) \wedge \thas(\tram, \tcash) \wedge \thas(\tram, \ticard))\label{ex:dentist:goal}\\
    \alwaysF(\bot                                   &\leftarrow \initially \wedge \neg \metric{\eventuallyF}{0}{60}\,\tgoal)\label{ex:dentist:query}
 \end{align}

To show the correspondence between the models of $\varphi$ and $\Theta(\varphi)$,
we define the restriction of an \HT-trace\ to an alphabet $\alphabet$
as
\begin{align}
  {(\tuple{H_i,T_i})_{\trangeco{i}{0}{\lambda}}}|_{\alphabet} = (\tuple{H_i\cap\alphabet,T_i\cap\alphabet})_{\trangeco{i}{0}{\lambda}}
\end{align}
\begin{theorem}[Correctness and Completeness]\label{thm:correct-mel-to-mlp}
  Let $\varphi$ be a metric formula over $\alphabet$.
  Then, we have\footnote{We let $(\tuple{\Htrace',\Ttrace'}, \tmf) \models \Theta(\varphi)$ abbreviate that
    $(\tuple{\Htrace',\Ttrace'}, \tmf)$ is a \MHT-model of $\Theta(\varphi)$.}
  \begin{align*}
    \{(\tuple{\Htrace,\Ttrace}, \tmf) \mid (\tuple{\Htrace,\Ttrace}, \tmf) \models \varphi\} =
    \{({\tuple{\Htrace',\Ttrace'}}|_\alphabet, \tmf) \mid (\tuple{\Htrace',\Ttrace'}, \tmf) \models \Theta(\varphi)\}
  \end{align*}
\end{theorem}
 \subsection{Base Translation of General Metric Logic Programs}
\label{sec:gmlp:base:translation}
\newcommand{\newPi}{\ensuremath{\bm{\Pi}}}
\newcommand{\newPsi}{\ensuremath{\bm{\Psi}}}

Analogous to Section~\ref{sec:mlp},
our translation of general metric logic programs is divided into two parts:
a basic temporal component that addresses state transitions, and
metric components that map timing functions to either \HT\ or \HTC.
The metric components are detailed in Sections~\ref{sec:gmlp:ht} and~\ref{sec:gmlp:htc}.
In the following, we focus on the basic temporal component.

Given a metric logic program $\program$
over \alphabet\ and $\lambda\in\mathbb{N}$,
we define
\begin{align}
  \newPi_\lambda(\program) & \eqdef
                             \{ \tk{r} \mid \alwaysF{r} \in P, \trangeco{k}{0}{\lambda}\}
                             \cup
                             \textstyle
                             \bigcup_{\varphi \in r, \alwaysF{r} \in \program, \trangeco{k}{0}{\lambda}} \eta_k^*(\varphi)
\end{align}
where \tk{r} is obtained from $r$ by replacing in $r$ each metric atom \matom\ by \laux{\matom}{k}.

The Tseitin-style translation of each metric formula in $r$ is given in the right column of Table~\ref{tab:tseitin:metric:temporal}.
\begin{table}[htp]
  \begin{center}
    \(
    \begin{array}{|l|l|l|}
      \cline{1-3}
      \bm{\varphi}
      & \bm{\eta_k(\varphi)}
      & \bm{\eta^*_k(\varphi)}
      \\
      \cline{1-3}
      a
      &  \laux{a}{k} \leftrightarrow a_k
      &
        \begin{array}{l}
          a_k \rightarrow \laux{a}{k} \\
          \laux{a}{k} \rightarrow a_k \\
        \end{array}
      \\
      \cline{1-3}
      \bot
      & \laux{\bot}{k} \leftrightarrow \bot
      & \laux{\bot}{k} \rightarrow \bot
      \\
      \cline{1-3}
      \top
      & \laux{\top}{k} \leftrightarrow \top
      & \top \rightarrow \laux{\top}{k}
      \\
      \cline{1-3}
      \finally
      & \laux{\finally}{k} \leftrightarrow
        \begin{cases} \top &\text{ if } k = \lambda - 1\\ \bot &\text{ if } k < \lambda - 1 \end{cases}
      &
        \begin{array}{l}
          \top \rightarrow \laux{\finally}{\lambda - 1} \\
          \laux{\finally}{k} \rightarrow \bot \qquad\ \hspace{42px}  \text{ if } k < \lambda - 1
        \end{array}
      \\
      \cline{1-3}
      \initially
      & \laux{\initially}{k} \leftrightarrow \begin{cases} \top &\text{ if } k = 0\\ \bot &\text{ if } k > 0 \end{cases}
      &
        \begin{array}{l}
          \top \rightarrow \laux{\initially}{0} \\
          \laux{\initially}{k} \rightarrow \bot \qquad\ \hspace{42px} \text{ if } k > 0
        \end{array}
      \\
      \cline{1-3}
      \metricI{\Next} b
      & \laux{\metricI{\Next} b}{k} \leftrightarrow
        \begin{cases}
          \laux{b}{k+1} \wedge \neg \faili{\cI}{k}{k+1} &\text{if } k < \lambda - 1\\
          \bot &\text{if } k = \lambda - 1
        \end{cases}
      &
        \begin{array}{l}
          \laux{\metricI{\Next} b}{k} \rightarrow \begin{cases}
            \laux{b}{k+1} & \;\qquad\ \;\text{if } k < \lambda - 1  \\
            \bot          & \;\qquad\ \;\text{if } k = \lambda - 1
          \end{cases}
          \\
          \laux{\metricI{\Next} b}{k} \wedge \faili{\cI}{k}{k+1} \rightarrow \bot \qquad\  \text{if } k < \lambda - 1 \\
          \laux{b}{k+1} \wedge \neg \faili{\cI}{k}{k+1} \rightarrow \laux{\metricI{\Next} b}{k}
          \text{ if } k < \lambda -1
        \end{array}
      \\
      \cline{1-3}
      \metricI{\alwaysF} b     & \begin{aligned}
        & \laux{\metricI{\alwaysF} b}{k} \leftrightarrow \bigwedge_{j=k}^{\lambda-1} \ljaux{\metricI{\alwaysF} b}{k}{j} \\
        & \ljaux{\metricI{\alwaysF} b}{k}{j} \leftrightarrow \faili{\cI}{k}{j} \vee \laux{b}{j} \qquad \hfill \text{ for }\trangeco{j}{k}{\lambda}
      \end{aligned}
      &
        \begin{array}{l}
          \bigwedge_{j=k}^{\lambda-1} \ljaux{\metricI{\alwaysF} b}{k}{j} \rightarrow \laux{\metricI{\alwaysF} b}{k}                        \\
          \laux{\metricI{\alwaysF} b}{k} \rightarrow \ljaux{\metricI{\alwaysF} b}{k}{j} \hfill \text{ for }\trangeco{j}{k}{\lambda}                   \\
          \faili{\cI}{k}{j} \rightarrow \ljaux{\metricI{\alwaysF} b}{k}{j} \hfill \text{ for }\trangeco{j}{k}{\lambda}                                \\
          \laux{b}{j} \rightarrow \ljaux{\metricI{\alwaysF} b}{k}{j} \hfill \text{ for }\trangeco{j}{k}{\lambda}                                   \\
          \ljaux{\metricI{\alwaysF} b}{k}{j} \wedge \neg \faili{\cI}{k}{j} \rightarrow \laux{b}{i} \qquad \hfill \text{ for }\trangeco{j}{k}{\lambda} \\
        \end{array}                                                                                                \\
      \cline{1-3}
      \metricI{\eventuallyF} b & \begin{aligned}
        & \laux{\metricI{\eventuallyF} b}{k} \leftrightarrow \bigvee_{j=k}^{\lambda-1} \ljaux{\metricI{\eventuallyF} b}{k}{j}    \\
        & \ljaux{\metricI{\eventuallyF} b}{k}{j} \leftrightarrow \neg \faili{\cI}{k}{j} \wedge \laux{b}{j} \quad\; \hfill \text{ for }\trangeco{j}{k}{\lambda}
      \end{aligned}
      &
        \begin{array}{l}
          \laux{\metricI{\eventuallyF} b}{k} \rightarrow \bigvee_{j=k}^{\lambda-1} \ljaux{\metricI{\eventuallyF} b}{k}{j}         \\
          \ljaux{\metricI{\eventuallyF} b}{k}{j} \rightarrow  \laux{\metricI{\eventuallyF} b}{k} \hfill \text{ for }\trangeco{j}{k}{\lambda} \\
          \ljaux{\metricI{\eventuallyF} b}{k}{j} \wedge \faili{\cI}{k}{j} \rightarrow \bot \hfill \text{ for }\trangeco{j}{k}{\lambda}  \\
          \ljaux{\metricI{\eventuallyF} b}{k}{j} \rightarrow \laux{b}{j} \hfill \text{ for }\trangeco{j}{k}{\lambda}                         \\
          \neg \faili{\cI}{k}{j} \wedge \laux b{j} \wedge \rightarrow \ljaux{\metricI{\eventuallyF} b}{k}{j} \quad\; \hfill \text{ for }\trangeco{j}{k}{\lambda}
        \end{array}                                                                                             \\
      \cline{1-3}
    \end{array}
    \)
  \end{center}
  \caption{Definition of sets $\eta_k(\varphi)$ and $\eta_k^*(\varphi)$ for any metric atom $\varphi$ and $\trangeco{k}{0}{\lambda}$.}
  \label{tab:tseitin:metric:temporal}
\end{table}
 It relies on additional auxiliary atoms of form $\laux{\varphi}{k}$, \ljaux{\varphi}{\kvar}{j} and \faili{\cI}{k}{j}.
\comment{REV: We should mention that $\delta$ is classical and
that is why we can move it to the body in some implications.
For this we would need a proposition showing that it behaves classically.}
Atoms of form $\laux{\varphi}{k}$ provide an equivalence to
the satisfaction of $\varphi$ at state $k$.
Atoms of form $\ljaux{\varphi}{\kvar}{j}$ are used to define $\laux{\varphi}{k}$ and
indicate state $j$ as a witness to the truth of $\varphi$ at state $k$.
Atoms of form \faili{\cI}{k}{j} provide the link to the respective encoding of metric information,
defined in the next two sections.
Specifically, \faili{\cI}{k}{j} indicates whether a transition between states $k$ and $j$
fails to be done within the time allotted by interval \cI,
viz.\ $\tmf(j)-\tmf(k)\notin\cI$.
Interestingly, such atoms only occur in rule bodies in $\eta^*_k(\varphi)$.
\comment{REV: Hence,
without definitions for $\faili{\cI}{k}{j}$,
$\newPi_\lambda(\program)$
gives the temporal stable models according to linear-time temporal \HT,
since all $\faili{\cI}{k}{j}$ become false,
making metric constraints are trivially satisfied.}
\comment{REV: We should add citation to \textit{agcadipescscvi20a}}

For illustration consider the metric rule in~\eqref{ex:dentist:query}.
Our translation first yields:
\begin{align}
    \laux{\bot}{k} &\leftarrow \laux{\initially}{k} \wedge \neg \laux{\metric{\eventuallyF}{0}{60}\,\tgoal}{k}&
    \quad\text{ for }\trangeco{k}{0}{\lambda}
\end{align}

Followed by the implications obtained
from the translation of each metric formula in the rule:
\begin{align}
    \laux{\bot}{k} &\leftrightarrow  \bot &
            \quad\text{ for }\trangeco{k}{0}{\lambda}\\
\laux{\initially}{0} &\leftrightarrow  \top &
    \\
\laux{\initially}{k} &\leftrightarrow  \bot &
            \quad\text{ for }\trangeco{k}{0}{\lambda-1}\\
\laux{\tgoal}{k} &\leftrightarrow  \tgoal_k &
            \quad\text{ for }\trangeco{k}{0}{\lambda}\\
\laux{\metric{\eventuallyF}{0}{60}\,\tgoal}{k} &\leftrightarrow  \bigvee_{j=k}^{\lambda-1} \ljaux{\metricI{\eventuallyF} \tgoal}{k}{j}  &
            \quad\text{ for }\trangeco{k}{0}{\lambda}\\
\ljaux{\metric{\eventuallyF}{0}{60}\,\tgoal}{k}{j}& \leftrightarrow \neg \faili{\cI}{k}{j} \wedge \laux{\tgoal}{j} &
            \quad\text{ for }\trangeco{k}{0}{\lambda},\trangeco{j}{k}{\lambda} \label{ex:dentist:query:delta}
\end{align}
Observe that when $\faili{\cI}{k}{j}$ is false,
the witness of truth for $\metric{\eventuallyF}{0}{60}\,\tgoal$
at state $j$ in~\eqref{ex:dentist:query:delta} simplifies to
$\ljaux{\metric{\eventuallyF}{0}{60}\,\tgoal}{k}{j}\leftrightarrow\laux{\tgoal}{j}$.
This indicates that the metric condition over the interval is trivially satisfied,
allowing any state that meets the goal condition to serve as a witness,
regardless of its associated time point.
Conversely, for the $\alwaysF$ operator,
such a simplification of $\faili{\cI}{k}{j}$ would imply that all states lie within the interval.
This would strengthen the witness requirements,
enforcing compliance with the condition across all future states.

Clearly, our translation still bears redundancies and can be further optimized.

\subsection{Translation of General Metric Logic Programs to \HT}
\label{sec:gmlp:ht}

This section extends the definition of timing functions within \HT\ from plain to general metric logic programs.
Timing functions are formalized as in Section~\ref{sec:mlp:ht}, utilizing a set of atoms
\(
\alphabetT\eqdef\{\timet_{\kvar,\tmvar} \mid \kvar,\tmvar\in\mathbb{N} \}
\)
and the axioms of $\Delta_{\lambda,\tmflimit}$ in~\eqref{def:ht:delta}.
Each atom $\timet_{\kvar,\tmvar}$ represents that $\tau(\kvar)=\tmvar$.
The key distinction lies in the formalization of
how these timing functions interact with the interval conditions specified in the program.

To this end, let \allIp\ stand for all intervals occuring in a metric logic program \program.

Given $\lambda,\tmflimit \in \mathbb{N}$
and
a general metric program $\program$,
we let
\begin{align}
  \newPsi_{\lambda,\tmflimit}(\program) \eqdef
  & \ \{ \faili{\intervco{m}{n}}{\kvar}{j} \leftarrow
    \timet_{\kvar,\tmvar} \wedge \timet_{j,\tmvar'} \mid
    0 \leq \kvar \leq j < \lambda-1,
    0\leq d<d'\leq\tmflimit,
    \label{def:ht:interval:one}\\&\qquad\qquad\nonumber
    \tmvar'-\tmvar < m,
    \intervco{m}{n} \in \allIp \}\; \cup            \\
  & \ \{ \faili{\intervco{m}{n}}{\kvar}{j} \leftarrow
    \timet_{\kvar,\tmvar} \wedge \timet_{j,\tmvar'} \mid
    0 \leq \kvar \leq j < \lambda-1,
    0\leq d<d'\leq\tmflimit,
    \label{def:ht:interval:two}\\&\qquad\qquad\nonumber
    \tmvar'-\tmvar \geq n,
    \intervco{m}{n} \in \allIp, n\in\mathbb{N} \}\;
\end{align}
Differing from the rule-oriented integrity constraints used to define $\Psi_{\lambda,\tmflimit}(\program)$ in~\eqref{def:ht:psi:one} and~\eqref{def:ht:psi:two},
we now utilize rules that directly signal interval condition violations.
Each such violation is indicated by an atom \faili{\cI}{k}{j},
which signifies that the difference $\tmf(j)-\tmf(k)$ falls outside the interval \cI,
where $k$ and $j$ are indices representing specific time points.

For the interval \intervco{15}{16} in rule~\eqref{ex:dentist:eit:translated},
we get
\begin{align}
  \faili{\intervco{15}{16}}{\kvar}{j} & \leftarrow \timet_{\kvar,\tmvar} \wedge \timet_{j,\tmvar'}
  \quad\text{ for } d'-d < 15\\
  \faili{\intervco{15}{16}}{\kvar}{j} & \leftarrow \timet_{\kvar,\tmvar} \wedge \timet_{j,\tmvar'}
  \quad\text{ for } d'-d \geq 16
\end{align}
and
$0 \leq \kvar \leq j < \lambda-1$,
$0\leq\tmvar<\tmvar'\leq\tmflimit$.
For $\lambda=100$ and $\tmflimit=1000$,
this then amounts to roughly $10^{10}$
\comment{REV: Accorging to some calculations it was $5*10^7$ before and $3*10^9$ now. So I rounded as before.}
instances for each of the above constraints.
Showing an increase by a factor of $10^2$
wrt. Section~\ref{sec:mlp:ht}
which amounts for the addition of $j$ ranging from $k$ to $100$.
Since $\faili{\cI}{\kvar}{j}$ is used to express falling outside the interval,
intervals like \intervco{15}{16}, with a distance of 1, generate the largest number of rules.
On the other hand, the interval \intervco{0}{\omega}, used in rules~\eqref{ex:dentist:nin} and~\eqref{ex:dentist:ten},
does not yield a single rule.
\comment{REV: We can elaborate on why there are no rules.}
Furthermore, the instance of rule \eqref{ex:dentist:eit}
for moving from $\tatm$ to $\thome$
uses the same interval as rule~\eqref{ex:dentist:eit:translated},
since the distance between these locations is the same.
In contrast to Section~\ref{sec:mlp:ht},
the rules for this interval condition are reused.

This approach culminates in the following completeness and correctness result.
\begin{theorem}[Completeness]\label{thm:mlp:full:ht:completeness}
  Let \program\ be a metric logic program over $\alphabet$,
  and
  $\M=(\tuple{\Ttrace,\Ttrace}, \tmf)$ a total timed \HT-trace of length $\lambda$.

  If
  $\M$ is a metric equilibrium model of $\program$,
  then
  there exists an equilibrium model $\tandt$ of $\newPi_\lambda(\program)\cup\Delta_{\lambda,\tmflimit}\cup\newPsi_{\lambda,\tmflimit}(\program)$
  with $\tmflimit=\tmf(\lambda-1)$
  such that $\tandt|_{\alphabets\cap\alphabetT}=\mhtToHt{\M}$.
\end{theorem}
\begin{theorem}[Correctness]\label{thm:mlp:full:ht:correctness}
  Let
  \program\ be a metric logic program,
  and
  $\lambda,\tmflimit \in \mathbb{N}$.

  If
  $\tandt$ is an equilibrium model of $\newPi_\lambda(\program)\cup\Delta_{\lambda,\tmflimit}\cup\newPsi_{\lambda,\tmflimit}(\program)$,
  then
  $\htToMht{\tandt|_{\alphabets\cap\alphabetT}}$ is a metric equilibrium model of $\program$.
\end{theorem}
The restriction to $\tandt|_{\alphabets\cap\alphabetT}$ allows us to ignore auxiliary atoms.

\subsection{Translation of General Metric Logic Programs to \HTC}
\label{sec:gmlp:htc}

This section adapts the framework of the preceding one,
utilizing \HTC\ as the target system instead of \HT,
for extending the definitions of plain metric programs from Section~\ref{sec:mlp:htc} to general ones.
Specifically,
we keep the \HTC\ signature
$\tuple{\mathcal{X},\mathcal{D},\mathcal{C}}$
and extend the set $\mathcal{X}$ to include Boolean variables of the form $\laux{\mu}{k}$ $\ljaux{\mu}{k}{j}$ and $\faili{k}{j}{\cI}$.
The denotation for Boolean atoms, equalities, and difference constraints
remains as defined in Section~\ref{sec:mlp:htc}.
Furthermore,
we keep employing integer variables
\(
\alphabetT^c=\{ \timet_k \mid 0\leq k<\lambda\}
\)
to directly represent timing functions,
such that $\tau(k)=t_k$ for $0\leq k<\lambda$, adhering to the axioms of $\Delta^{c}_{\lambda}$
in~\eqref{def:htc:delta}.

However, to formalize the interaction between these timing functions and the interval conditions specified in the program,
we adopt the approach of the preceding section,
deriving atoms of the form \faili{\cI}{\kvar}{j} indicating that the difference $\tmf(j)-\tmf(k)$ falls outside interval \cI.
\begin{align}
  \newPsi^{c}_{\lambda}(\program) \eqdef
  & \ \{ \faili{\intervco{m}{n}}{\kvar}{j} \leftarrow \neg (\timet_{\kvar}-\timet_{j} \leq -m) \mid
    0 \leq \kvar \leq j < \lambda-1, \intervco{m}{n} \in \allIp \}\; \cup\label{def:htc:psi:one:general}                                     \\
  & \ \{ \faili{\intervco{m}{n}}{\kvar}{j} \leftarrow \neg (\timet_{j} - \timet_{\kvar} \leq n - 1) \mid
    0 \leq \kvar \leq j < \lambda-1, \intervco{m}{n} \in \allIp, n\in\mathbb{N} \}\label{def:htc:psi:two:general}
\end{align}

For the interval \intervco{15}{16} in rule~\eqref{ex:dentist:eit:translated},
we get
\begin{align}
  \faili{\intervco{15}{16}}{\kvar}{j} & \leftarrow \neg (\timet_{\kvar}-\timet_{j} \leq -15)\\
  \faili{\intervco{15}{16}}{\kvar}{j} & \leftarrow \neg (\timet_{j} - \timet_{\kvar} \leq 15)
\end{align}
for
$0 \leq \kvar \leq j < \lambda-1$.
Given $\lambda=100$, this amounts to $10^4$ instances,
showing the same increase proportional $\lambda$ as in the \HT\ approach
with respect to the plain metric fragment for \HTC,
due to the inclusion of $j$.
\comment{REV: We could elaborate a bit more there.}

To show completeness and correctness,
we define the restriction of an \HTC\ interpretation to an alphabet $\alphabet$
as
\(
\htcinterp|_{\alphabet} = \tuple{\Vh|_{\alphabet},\Vt|_{\alphabet}}
\)
where $f|_{\alphabet}$ is defined such that $f|_{\alphabet}(a) = f(a)$ for all $a\in\alphabet$.
\begin{theorem}[Completeness]\label{thm:mlp:full:htc:completeness}
  Let \program\ be a metric logic program over $\alphabet$,
  and
  $\M=(\tuple{\Ttrace,\Ttrace}, \tmf)$ a total timed \HT-trace of length $\lambda$.

  If
  $\M$ is a metric equilibrium model of $\program$,
  then
  there exists a constraint equilibrium model $\htcttinterp$ of $\newPi_\lambda(\program)\cup\Delta^{c}_{\lambda}\cup\newPsi^{c}_{\lambda}(\program)$
  such that $\htcttinterp|_{\alphabets\cap\alphabetT^c}=\mhtToHtc{\M}$.
\end{theorem}
\begin{theorem}[Correctness]\label{thm:mlp:full:htc:correctness}
  Let
  \program\ be a metric logic program,
  and
  $\lambda \in \mathbb{N}$.

  If
  $\htcttinterp$ is an constraint equilibrium model of $\newPi_\lambda(\program)\cup\Delta^{c}_{\lambda}\cup\newPsi^{c}_{\lambda}(\program)$,
  then
  $\htcToMht{\htcttinterp|_{\alphabets\cap\alphabetT^c}}$ is a metric equilibrium model of $\program$.
\end{theorem}

\section{Implementation}
\label{sec:implementation}

In what follows,
we rely on a basic acquaintance with the ASP system \clingo~\citep{PotasscoUserGuide}.
We show below how easily our approach is implemented via \clingo's meta-encoding framework.
This serves us as a blueprint for a more sophisticated future implementation.

However, before delving into details, we give in Figure~\ref{fig:workflow} an overview of our framework and its workflows.
\begin{figure}[ht]
   \centering
   \includegraphics[width=0.9\linewidth]{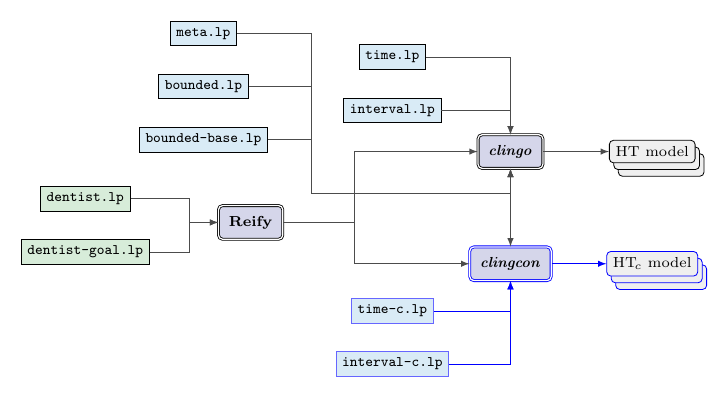}
   \caption{Workflow of our implementations applied to the dentist example.}
   \label{fig:workflow}
\end{figure}
File names depicted in the diagram correspond to those introduced in the listings of this section.
Nodes with a green background represent user-provided input programs, illustrated here by the dentist example, whereas solver nodes are distinguished by a purple double border.
The blue nodes denote the meta-encodings used for the translation, some of which are shared by both approaches, while the \HTC-specific pipeline is explicitly highlighted with blue arrows.
All presented encodings, as well as our evaluation results, are publicly available at \url{https://github.com/potassco/memelingo}.

\Clingo\ allows for reifying a ground logic program in terms of facts, which can then be (re)interpreted by a meta-encoding.
The result of another grounding is then channeled to the respective back-end, in our case a regular or hybrid ASP solver, respectively.
Though, for brevity, we must refer to~\citep{karoscwa21a} for details,
we mention that a reified ground program is represented by instances of predicates
\lstinline{atom_tuple},
\lstinline{literal_tuple},
\lstinline{rule},
\lstinline{output},
etc.
\footnote{Below we draw upon the symbol table, captured by \lstinline{output/2}, for extracting syntactic entities.}

\begin{lstlisting}[caption={Metric logic program for dentist example (\texttt{dentist.lp})},label={prg:dentist},language=clingos]
item(icard). item(cash).
loc(dentist). loc(office). loc(atm). loc(home).

at(ram,office) :- initially.
at(cash,atm)   :- initially.
at(icard,home) :- initially.

go(ram,L') : loc(L'), L'!=L :- at(ram,L), not finally.

has(ram,I) :- at(ram,L), at(I,L), item(I).
at(ram,L)  :- at(ram,L), has(ram,I).

next((D,D+1),at(ram,L')) :- at(ram,L), go(ram,L'), distance(L,L',D).%%#(\label{ex:dentist:eit:code}#)

next((0,w),has(ram,I)) :- has(ram,I), not finally.
next((0,w),at(I,L)) :- at(I,L), item(I), not has(_,I), not finally.
\end{lstlisting}
Listing~\ref{prg:dentist}
uses the dentist example
(leaving out the representation of Table~\ref{table:dentist} in terms of \lstinline{distance/3})
to illustrate
a metric logic program.
This input program is then reified and interpreted by the corresponding meta-encodings.
For simplicity, we assume that each rule is implicitly in the scope of an always operator \alwaysF.
Moreover, included temporal operators and their comprised atoms are exempt from simplifications during grounding.\footnote{In meta-encodings, this is done by adding corresponding \lstinline{#external} directives.}
We use predicate \lstinline{next/2} for the metric next operator.
As an example,
consider the instance of Line~\ref{ex:dentist:eit:code} for moving from $\toffice$ to $\thome$,
viz.\ the counterpart of~\eqref{ex:dentist:eit:translated}.
\begin{lstlisting}[language=clingos]
next((15,16),at(ram,home)) :-
             at(ram,office), go(ram,home), distance(office,home,15).
\end{lstlisting}
Note that \lstinline[mathescape]{next((0,w),$\cdot$)} in the head of the last two rules
stands for $\next$ aka $\next_{\intervco{0}{\omega}}$.

\begin{lstlisting}[caption={Metric logic program for the goal in dentist example (\texttt{dentist-goal.lp})},label={prg:dentist-goal},language=clingos]
goal :- at(ram,dentist), has(ram,icard), has(ram,cash).
:- initially, not eventually((0,60),goal).
\end{lstlisting}
The rules in Listing~\ref{prg:dentist} do not ensure that Ram reaches the dentist within one hour with all necessary items.
This condition is represented in Listing~\ref{prg:dentist-goal},
corresponding to rules~\eqref{ex:dentist:goal} and \eqref{ex:dentist:query}.
However, these rules use the global metric operator $\metricI{\eventuallyF}$,
which means they do not conform a (plain) metric logic program.
As a result, we must adopt a different approach for these programs,
replacing Listing~\ref{prg:dentist-goal},
which will be discussed in Section~\ref{sec:implementation:plain}.
Before proceeding,
we will conclude this section with the remaining
implementation details that are
common to both fragments.

\begin{lstlisting}[caption={Timed meta-encoding (\texttt{meta.lp})},label={prg:meta},language=clingos]
time(0..lambda-1).%%#(\label{meta:time}#)

conjunction(B,T) :- literal_tuple(B), time(T),
       hold(L,T) : literal_tuple(B, L), L > 0;
   not hold(L,T) : literal_tuple(B,-L), L > 0.

body(normal(B),T) :- rule(_,normal(B)), conjunction(B,T).
body(sum(B,G),T)  :- rule(_,sum(B,G)), time(T),
  #sum{W,L :     hold(L,T), weighted_literal_tuple(B, L,W), L>0;
       W,L : not hold(L,T), weighted_literal_tuple(B,-L,W), L>0} >= G.

  hold(A,T) : atom_tuple(H,A)   :- rule(disjunction(H),B), body(B,T).
{ hold(A,T) : atom_tuple(H,A) } :- rule(     choice(H),B), body(B,T).
\end{lstlisting}
Listing~\ref{prg:meta} modifies the basic meta-encoding in~\citep{karoscwa21a}
by adding a variable \lstinline{T} for time steps to all derived predicates.
Their range is fixed in Line~\ref{meta:time}.
In this way, an atom \lstinline{hold(a,k)} stands for $a_k$ in \alphabets,
where \lstinline{a} is the numeric identifier of $a$ in \alphabet.
While this encoding handles Boolean connectives,
the metric ones are treated in Listing~\ref{prg:metabound}.
The rules in Line~\ref{meta:bounded:true:one} and~\ref{meta:bounded:true:two} restore the symbolic representation
of the numerically identified atoms, which allows us to analyze the inner structure of modalized propositions.
Lines~(\ref{meta:bounded:initially:one}/\ref{meta:bounded:initially:two}) and
(\ref{meta:bounded:finally:one}/\ref{meta:bounded:finally:two})
deal with \initially\ and \finally, respectively.
The remaining metric operators
are handled differently for each metric fragment,
as described in Sections~\ref{sec:mlp} and~\ref{sec:gmlp}.
These will be discussed in detail in
Sections~\ref{sec:implementation:plain} and~\ref{sec:implementation:general}, respectively.
\begin{lstlisting}[caption={Meta encoding for base metric operators (\texttt{bounded-base.lp})},label={prg:metabound},language=clingos]
true(O,K) :- output(O,B), time(K), hold(L,K) : literal_tuple(B,L).%%#(\label{meta:bounded:true:one}#)
hold(L,K) :- output(O,B), time(K), true(O,K),  literal_tuple(B,L).%%#(\label{meta:bounded:true:two}#)

:- true(initially,K), time(K), K!=0.%%#(\label{meta:bounded:initially:one}#)
true(initially,0).%%#(\label{meta:bounded:initially:two}#)

:- true(finally,K), time(K), K!=lambda-1.%%#(\label{meta:bounded:finally:one}#)
true(finally,lambda-1).%%#(\label{meta:bounded:finally:two}#)
\end{lstlisting}

\subsection{Implementation of Plain Metric Logic Programs}
\label{sec:implementation:plain}

\begin{lstlisting}[caption={Meta encoding for metric operators of $\Pi_\lambda(\program)$ (\texttt{bounded.lp})},label={prg:metabound:op},language=clingos]
:- true(next(_,_),K), time(K), K=lambda-1.%%#(\label{meta:bounded:next:one}#)
true(V,K+1) :- true(next(I,V),K).%%#(\label{meta:bounded:next:two}#)
\end{lstlisting}

Lines~\ref{meta:bounded:next:one} and~\ref{meta:bounded:next:two}
of Listing~\ref{prg:metabound:op}
realize the metric next operator,
$\metricI{\Next}a$, represented by term \lstinline{next(I,a)}.
Together, Listing~\ref{prg:meta},~\ref{prg:metabound} and \ref{prg:metabound:op} account for $\Pi_\lambda(\program)$.

When expressing time via Boolean variables,
the previous listings are combined with Listing~\ref{prg:time:ht} and~\ref{prg:interval:ht} below,
which realize $\Delta_{\lambda,\tmflimit}$ and $\Psi_{\lambda,\tmflimit}(\program)$, respectively.
Atoms $\timet_{\kvar,\tmvar}$ in \alphabetT\ are represented by \lstinline{t(k,d)}.
The upper bound $\tmflimit$ on the timing function's range is given by \lstinline{v},
and $\omega$ is represented by \lstinline{w}.
The two encodings directly mirror the definitions of $\Delta_{\lambda,\tmflimit}$ and
$\Psi_{\lambda,\tmflimit}(\program)$,
with one key difference:
the rule bodies in~\eqref{def:ht:psi:one} and~\eqref{def:ht:psi:two} are replaced in
Lines~\ref{interval:ht:one} and~\ref{interval:ht:two} of Listing~\ref{prg:interval:ht} by auxiliary atoms of
predicate \lstinline{true/2}.
\begin{lstlisting}[caption={Meta encoding for $\Delta_{\lambda,\tmflimit}$ (\texttt{time.lp})},label={prg:time:ht},language=clingos]
timepoint(0..v).

t(0,0).
t(K+1,D') : D<D', timepoint(D') :- t(K,D), time(K+1).%%#(\label{time:ht:two}#)
\end{lstlisting}
\begin{lstlisting}[caption={Meta encoding for $\Psi_{\lambda,\tmflimit}(\program)$ (\texttt{interval.lp})},label={prg:interval:ht},language=clingos]
:- true(next((M,N),V),K), t(K,D), t(K+1,D'), D' - D < M.%%#(\label{interval:ht:one}#)
:- true(next((M,N),V),K), t(K,D), t(K+1,D'), D' - D >= N , N!=w.%%#(\label{interval:ht:two}#)
\end{lstlisting}

When expressing time in terms of integer variables,
we rely on difference constraints for modeling timing functions.
Such simplified linear constraints have the form `$x - y \leq d$' for $x,y\in\mathcal{X}$ and $d\in\mathbb{Z}$
and are supported by the \clingo\ extensions \clingcon~\citep{bakaossc16a} and \clingodl~\citep{jakaosscscwa17a}.
We use below \clingcon's syntax and represent them as `\lstinline[mathescape]|&sum{$x$ ; $y$} <= $d$|'.
A Boolean atom $a$ can be seen as representing `$a=\true$'.
In the case at hand,
Listing~\ref{prg:meta},~\ref{prg:metabound} and \ref{prg:metabound:op} are now completed by Listing~\ref{prg:time:htc} and~\ref{prg:interval:htc} below.
As above, they faithfully replicate the definitions of $\Delta^{c}_{\lambda}$ and $\Psi^{c}_{\lambda}(\program)$.
Unlike above, however,
the timing function is now captured by integer variables of form \lstinline{t(k)} and
its range restriction is now obsolete.
The rules in Listing~\ref{prg:time:htc} mirror the two conditions on timing functions,
namely, that \lstinline{t(0)} equals zero and that the instances of \lstinline{t(K+1)} receive a strictly greater integer
than the ones of \lstinline{t(K)} for \lstinline{K} ranging from \lstinline{0} to \lstinline{lambda-1}.
Similarly, given that \lstinline{w} stands for $\omega$,
the two rules in Listing~\ref{prg:interval:htc} correspond to
the difference constraints in~\eqref{def:htc:psi:one} and~\eqref{def:htc:psi:two}.
\begin{lstlisting}[caption={Meta encoding for $\Delta^{c}_{\lambda}$ (\texttt{time-c.lp})},label={prg:time:htc},language=clingos]
&sum{t(0)} = 0.
&sum{t(K) ; -t(K+1)} <= -1 :- time(K), time(K+1).
\end{lstlisting}
\begin{lstlisting}[caption={Meta encoding for $\Psi^{c}_{\lambda}(\program)$ (\texttt{interval-c.lp})},label={prg:interval:htc},language=clingos]
&sum{t(K); -t(K+1)} <= -M  :- true(next((M,N),V),K),
                              time(K), time(K+1).
&sum{t(K+1); -t(K)} <= N-1 :- true(next((M,N),V),K),
                              time(K), time(K+1), N!=w.
\end{lstlisting}
As in Listing~\ref{prg:interval:ht}, we use auxiliary atoms of predicate \lstinline{true/2} rather than the
corresponding rule bodies.
Notably, we shifted in Listing~\ref{prg:interval:htc} the difference constraints in~\eqref{def:htc:psi:one}
and~\eqref{def:htc:psi:two} from the body to the head.
This preserves (strong) equivalence in \HTC\ whenever all variables comprised in a constraint atom are defined,
as guaranteed by Proposition~\ref{pro:htc:defined}.
\comment{REV: For the other fragment we do this shift in the translation already.}

In the two approaches at hand,
we may compensate the lack of global metric operators in out example by
replacing Line~1 by `\lstinline{:- finally, not goal}' and
enforcing the time limit either
by setting $\tmflimit$ to 60 in our \HT-based approach or
by extending Listing~\ref{prop:htc:timed:delta} with `\lstinline|&sum{t(K)} <= 60 :- time(K), not time(K+1).|' in our \HTC-based approach.
However, this is only effective when the goal is achieved at the final step.

The example in Listing~\ref{prg:dentist}
is addressed with \clingo\ in the following way.
\begin{lstlisting}[basicstyle=\small\ttfamily,numbers=none]
clingo dentist.lp --output=reify |
clingo 0 - meta.lp bounded-base.lp bounded.lp
           time.lp interval.lp -c lambda=4 -c v=110
\end{lstlisting}
When using \clingcon\ instead,
it suffices to replace the second line with:
\begin{lstlisting}[basicstyle=\small\ttfamily,numbers=none]
clingcon 0 - meta.lp bounded-base.lp bounded.lp
           time-c.lp interval-c.lp -c lambda=4
\end{lstlisting}
Our choices of \lstinline{lambda} and \lstinline{v}
allow for all movement combinations within the 4 steps required to reach the goal;
we obtain in each case 27 solutions.
Once we include the query-oriented additions from above, we obtain a single model instead.

\subsection{Implementation of General Metric Logic Programs}
\label{sec:implementation:general}

\begin{lstlisting}[caption={Meta encoding for metric operators of $\newPi_\lambda(\program)$ (\texttt{bounded.lp})},label={prg:metabound:op:general},language=clingos]
true(V,K+1) :- true(next(I,V),K), time(K), K<lambda-1.%%#(\label{meta:bounded:next:one:general}#)
:- true(next(_,_),K), time(K), K=lambda-1.
:- true(next(I,_),K), iv(I,(K,K+1)), time(K), K<lambda-1.
true(next(I,V),K) :-  true(V,K+1), not iv(I,(K,K+1)),
                      time(K), K<lambda-1, output(next(I,V),_).%%#(\label{meta:bounded:next:last:general}#)

true(always(I,V),K) :- wit(always(I,V),(K,J)) : time(J), K<=J;%%#(\label{meta:bounded:always:one}#)
                       output(always(I,V),_), time(K).
wit(always(I,V),(K,J)) :- true(always(I,V),K),
                          time(K), time(J), K<=J.
wit(always(I,V),(K,J)) :- iv(I,(K,J)), output(always(I,V),_),
                          time(K), time(J), K<=J.
wit(always(I,V),(K,J)) :- true(V,J), output(always(I,V),_),
                          time(K), time(J), K<=J.
true(V,J) :- wit(always(I,V),(K,J)), not iv(I,(K,J)),
             time(K), time(J), K<=J.%%#(\label{meta:bounded:always:last}#)

wit(eventually(I,V),(K,J)): time(J), K<=J :- true(eventually(I,V),K). %%#(\label{meta:bounded:eventually:one}#)
true(eventually(I,V),K) :- wit(eventually(I,V),(K,J)),
                           time(K), time(J), K<=J.
                        :- wit(eventually(I,V),(K,J)), iv(I,(K,J)),
                           time(K), time(J), K<=J.
true(V,J):-wit(eventually(I,V),(K,J)),
           time(K), time(J), K<=J.
wit(eventually(I,V),(K,J)):- true(V,J), not iv(I,(K,J)),
                             output(eventually(I,V),_),
                             time(K), time(J), K<=J. %%#(\label{meta:bounded:eventually:last}#)
\end{lstlisting}
Listing~\ref{prg:metabound:op:general}
extends Listing~\ref{prg:metabound} to account for the
metric next operator $\metricI{\Next}b$ in Lines~\ref{meta:bounded:next:one:general} to~\ref{meta:bounded:next:last:general},
the metric always operator $\metricI{\alwaysF}b$ in Lines~\ref{meta:bounded:always:one} to~\ref{meta:bounded:always:last},
and the metric eventually operator $\metricI{\eventuallyF}b$ in Lines~\ref{meta:bounded:eventually:one} to~\ref{meta:bounded:eventually:last}.
These rules mirror the definitions of $\eta^*_k(\varphi)$ in Table~\ref{tab:tseitin:metric:temporal},\footnote{The use of the predicate \lstinline{output/2} is merely to ensure rule safety.}
where atoms \lstinline{true(m,k)} correspond to $\laux{\varphi}{k}$,
witnesses $\ljaux{\varphi}{k}{j}$ are expressed via the atom \lstinline{wit(m,(k,j))},
and interval violations $\faili{i}{k}{j}$ are represented by \lstinline{iv(i,(k,j))}.
Together, Listings~\ref{prg:meta},~\ref{prg:metabound}, and~\ref{prg:metabound:op:general} account for $\newPi_\lambda(\program)$.

When expressing time via Boolean variables,
we reuse Listing~\ref{prg:time:ht} to implement $\Delta_{\lambda,\tmflimit}$ and
define $\newPsi_{\lambda,\tmflimit}(\program)$ in Listing~\ref{prg:interval:ht:general} below.
The two rules in Lines~(\ref{interval:ht:general:onea}/\ref{interval:ht:general:oneb})
and Lines~(\ref{interval:ht:general:twoa}/\ref{interval:ht:general:twob})
correspond to \eqref{def:ht:interval:one} and \eqref{def:ht:interval:two},
respectively,
where predicate \lstinline{i/1} collects the intervals in $\allIp$
\begin{lstlisting}[caption={Meta encoding for $\newPsi_{\lambda,\tmflimit}(\program)$ (\texttt{interval.lp})},label={prg:interval:ht:general},language=clingos]
iv((M,N),(K,J)):- i((M,N)), t(K,D), t(J,D'), time(K), time(J), K<=J, %%#(\label{interval:ht:general:onea}#)
                  D' - D <  M.%%#(\label{interval:ht:general:oneb}#)
iv((M,N),(K,J)):- i((M,N)), t(K,D), t(J,D'), time(K), time(J), K<=J, %%#(\label{interval:ht:general:twoa}#)
                  D' - D >= N , N!=w.%%#(\label{interval:ht:general:twob}#)
\end{lstlisting}

When expressing time using integer variables,
we rely on difference constraints to model timing functions,
as described in the previous section.
We reuse Listing~\ref{prg:time:htc} to implement $\Delta^{c}_{\lambda}$
and define $\newPsi^{c}_{\lambda}(\program)$ in Listing~\ref{prg:interval:htc:general} below.
The two rules in Listing~\ref{prg:interval:htc:general} correspond to
rules~\eqref{def:htc:psi:one:general} and~\eqref{def:htc:psi:two:general}.
\begin{lstlisting}[caption={Meta encoding for $\newPsi^{c}_{\lambda}(\program)$ (\texttt{interval-c.lp})},label={prg:interval:htc:general},language=clingos]
iv((M,N),(K,K')):- i((M,N)), time(K), time(K'), K<=K',
                   &sum{t(K); -t(K')} > -M.
iv((M,N),(K,K')):- i((M,N)), time(K), time(K'), K<=K',
                   &sum{t(K'); -t(K)} > N-1, N!=w.
\end{lstlisting}

The example in Listing~\ref{prg:dentist}
can be addressed using \clingo\ and \clingcon\ in the same manner as before.
However, this fragment accommodates the goal condition
from Listing~\ref{prg:dentist-goal}.
Listing~\ref{prg:dentist:output} displays the corresponding \clingo\ output.
\begin{lstlisting}[float=ht,language=bash,basicstyle=\small\ttfamily,caption={System output showcasing the dentist example using \clingo},label={prg:dentist:output}]
UNIX > clingo dentist.lp dentist-goal.lp --output=reify |  \
       clingo 0 - meta.lp bounded-base.lp bounded.lp       \
                  time.lp interval.lp -c lambda=4 -c v=110
Solving...
Answer: 1
true(go(ram,atm),0) true(go(ram,home),1) true(go(ram,dentist),2)
true(at(ram,office),0) true(at(icard,home),0) true(at(cash,atm),0)
true(at(ram,atm),1) true(at(icard,home),1) true(at(cash,atm),1)
true(at(ram,home),2) true(at(icard,home),2)
true(at(ram,dentist),3)
t(0,0) t(1,20) t(2,35) t(3,55)
SATISFIABLE

Models       : 1
Calls        : 1
Time         : 0.332s (Solving: 0.07s 1st Model: 0.01s Unsat: 0.06s)
CPU Time     : 0.322s
\end{lstlisting}
For readability,
we ordered the output atoms and omitted all predicates except
\lstinline{t/2}, \lstinline{go/3}, and \lstinline{at/3}.
The output presents
the expected optimal plan for reaching the dentist on time.
Line~6 details Ram's actions
(traveling from the office to the ATM, then home, and finally to the dentist),
while Lines~7--9 track his location and the items he acquires.
The timing function is encoded as normal atoms over the predicate \lstinline{t/2} (Line~11);
the time difference between consecutive states corresponds to the distance between
locations, achieving the goal in 55~minutes.

The \clingcon\ output (Listing~\ref{prg:dentist:output:clingcon}) differs only in its representation of the timing function,
which utilizes integer variables
and is presented separately as assignments
in Lines~11 and~12.
\begin{lstlisting}[float=ht,language=bash,basicstyle=\small\ttfamily,caption={System output showcasing the dentist example using \clingcon},label={prg:dentist:output:clingcon}]
UNIX > clingo dentist.lp dentist-goal.lp --output=reify | \
       clingcon 0 - meta.lp bounded-base.lp bounded.lp    \
                  time-c.lp interval-c.lp -c lambda=4
Solving...
Answer: 1
true(go(ram,atm),0) true(go(ram,home),1) true(go(ram,dentist),2)
true(at(ram,office),0) true(at(icard,home),0) true(at(cash,atm),0)
true(at(ram,atm),1) true(at(icard,home),1) true(at(cash,atm),1)
true(at(ram,home),2) true(at(icard,home),2)
true(at(ram,dentist),3)
Assignment:
t(0)=0 t(1)=20 t(2)=35 t(3)=55
SATISFIABLE

Models       : 1
Calls        : 1
Time         : 0.016s (Solving: 0.00s 1st Model: 0.00s Unsat: 0.00s)
CPU Time     : 0.013s
\end{lstlisting}
 \subsection{Results}
\label{sec:results}
In this section, we present some empirical results on the scalability
of both approaches.
We use \clingo\ 5.8.0~\citep{karoscwa21a} for \HT,
and \clingcon\ 5.2.1~\citep{bakaossc16a} and \clingodl\ 1.5.0~\citep{jakaosscscwa17a} for \HTC.
In our setting,
the solutions computed by \clingcon\ coincide with those defined by \HTC,
while \clingodl\ returns only assignments
where the time points take the smallest possible positive integer values.

\paragraph{Experimental Setup.}
We evaluated our implementation on three problem domains.
For each of them, we developed a novel metric encoding.
The first is the dentist scenario, used
to compare both metric
fragments, plain (\(P\)) and general (\(G\)), and to
study the scalability of each approach
and the overhead introduced by the general fragment.
In this domain, we enumerate all models
to ensure full inspection of the search space, with and without the goal condition.
Since the formalization uses only intervals of size one,
the minimality condition of \clingodl\ does not reduce the number of models.
The second domain is Multi-Agent Path Finding~\citep{ststfekomawaliatcokubabo19b}.
It is evaluated only for the general fragment,
because the representation requires global temporal operators
(\(\Diamond_I\), \(\Box_I\)) to express
multi-agent coordination and goal conditions,
which are not available in the plain fragment.
In this domain, we analyze the scalability of the general fragment across all approaches in a more complex setting.
We compute a single model, as enumeration would lead to a very large number of models,
so the minimality imposed by \clingodl\ only affects which model is returned.
The metric representation is based on the ASP encoding from~\citet{bekascsosvwa24a}.
The third domain is Job-Shop Scheduling. It is evaluated in the general fragment, computing a single model.
Here, we analyze the impact of varying the horizon on the overall scalability of each approach.
The metric encoding is inspired by~\citet{huozdi20a}.

We conducted all experiments on a cluster\footnote{\url{https://www.cs.uni-potsdam.de/bs/research/labs.html#hardware}}
with Intel Xeon E5-2650v4@2.9GHz CPUs
and 64GB of memory, running Debian Linux 10,
imposing per instance a timeout of 20 minutes and a memory limit of 20GB.

\subsubsection{Dentist Scenario}

In our first setting,
we use the running example from Listing~\ref{prg:dentist}.
As explained before,
we select \lstinline{lambda}=4 and \lstinline{v}=110 to
generate all movement combinations within the 4 steps required to reach the goal,
resulting in 27 solutions.
To investigate scalability, we multiply both the durations in Table~\ref{table:dentist}
and the time limit \lstinline{v} by a factor $f\in\{1, 4, 7, 10\}$.
The results are summarized in Table~\ref{tab:benchmarks}.
They show that \clingo\ does not scale well.
This is most clear at $f=10$, where it reaches the memory limit while grounding.
In contrast,
the performance of \clingcon\ and \clingodl\
is independent of the time granularity, for both fragments $P$ and $G$.
The inner workings of both systems differ,
but in this simple example their behavior is essentially the same.
As anticipated in Sections~\ref{sec:gmlp:ht} and~\ref{sec:gmlp:htc},
there is a linear increase in the number of rules of $G$ with respect to $P$.
The additional rules enable the next operator $\metricI{\Next}$ in the body
and define auxiliary atoms representing interval violations.
These additions may explain why the general fragment
relies more on solving than the plain one, as observed in Table~\ref{tab:benchmarks}.
Note that the global metric operators $\metricI{\eventuallyF}$ and $\metricI{\alwaysF}$
are not used in this setting,
and consequently
the intelligent grounder does not produce the corresponding ground rules.

\begin{table}[ht]
  \centering
    \begin{tabular}{|rr|rrr|rrr|rrr|}
      \cline{1-11}
         $f$ &  &\multicolumn{3}{c|}{\clingo}
         &\multicolumn{3}{c|}{\clingcon}
         &\multicolumn{3}{c|}{\clingodl}\\
         & & solve & ground & \#rules & solve & ground & \#rules & solve & ground & \#rules\\
      \cline{1-11}
      1  & $P$ & 0.03 & 1.07 & 297\,383 & 0.01 & 0.14 & 2\,337 & 0.01 & 0.14 & 2\,337 \\
         & $G$ & 0.29 & 0.61 & 334\,458 & 0.01 & 0.22 & 2\,572 & 0.01 & 0.16 & 2\,572 \\
\cline{1-11}
      4  & $P$ & 1.65 & 21.61 & 4\,812\,563 & 0.01 & 0.18 & 2\,337 & 0.01 & 0.14 & 2\,337 \\
         & $G$ & 36.37 & 15.80 & 5\,395\,908 & 0.01 & 0.19 & 2\,572 & 0.01 & 0.19 & 2\,572 \\
\cline{1-11}
      7  & $P$ & 16.58 & 79.35 & 14\,772\,743 & 0.01 & 0.14 & 2\,337 & 0.01 & 0.13 & 2\,337 \\
         & $G$ & 81.56 & 67.55 & 16\,555\,758 & 0.02 & 0.21 & 2\,572 & 0.01 & 0.16 & 2\,572 \\
\cline{1-11}
      10 & $P$ & \multicolumn{3}{c|}{--} & 0.01 & 0.15 & 2\,337 & 0.01 & 0.18 & 2\,337 \\
         & $G$ & \multicolumn{3}{c|}{--} & 0.02 & 0.19 & 2\,572 & 0.01 & 0.23 & 2\,572 \\
      \cline{1-11}
    \end{tabular}
    \medskip
    \caption{Results for \clingo, \clingcon\ and \clingodl\ without goal condition
    (times in seconds; \textup{--} represents a memout).}
    \label{tab:benchmarks}
\end{table}

In our second setting,
we include the goal constraint from Listing~\ref{prg:dentist-goal},
which can only be handled by the general fragment.
We use the same parameters as before.
In particular, we enumerate all models,
but in this case there is always a single one.
The results are summarized in Table~\ref{tab:benchmarks:goal}.
As before, \clingo\ does not scale well,
while the performance of \clingcon\ and \clingodl\
is independent of the time granularity.
The number of rules increases as expected,
due to the goal constraint and
the rules required
to account for the metric eventually \(\metricI{\eventuallyF}\) occurring in it.
For \clingo\, the inclusion of the goal condition leads to a slight increase in grounding time,
while solving time decreases since the goal constraint significantly reduces the number of models.

\begin{table}[ht]
  \centering
    \begin{tabular}{|rr|rrr|rrr|rrr|}
      \cline{1-11}
         $f$ &  &\multicolumn{3}{c|}{\clingo}
         &\multicolumn{3}{c|}{\clingcon}
         &\multicolumn{3}{c|}{\clingodl}\\
         & & solve & ground & \#rules & solve & ground & \#rules & solve & ground & \#rules\\
      \cline{1-11}
       1 & $G$ & 0.10 & 0.73 & 377\,301 & 0.01 & 0.19 & 2\,731 & 0.01 & 0.20 & 2\,731 \\
\cline{1-11}
       4 & $G$ & 28.33 & 15.69 & 6\,092\,421 & 0.01 & 0.21 & 2\,731 & 0.01 & 0.19 & 2\,731 \\
\cline{1-11}
       7 & $G$ & 17.20 & 77.04 & 18\,694\,341 & 0.01 & 0.18 & 2\,731 & 0.01 & 0.21 & 2\,731 \\
\cline{1-11}
      10 & $G$ & \multicolumn{3}{c|}{--} & 0.01 & 0.23 & 2\,731 & 0.01 & 0.19 & 2\,731 \\
      \cline{1-11}
    \end{tabular}
    \medskip
    \caption{Results for \clingo, \clingcon\ and \clingodl\ with goal condition,
    (times in seconds;
\textup{--} indicates the memory limit was reached).}
    \label{tab:benchmarks:goal}
\end{table}

\subsubsection{Multi-Agent Path Finding}
\newcommand{\tstarta}{\textit{$S_A$}}
\newcommand{\tgoala}{\textit{$G_A$}}
\newcommand{\tmoving}{\textit{moving}}
\newcommand{\ttaken}{\textit{taken}}
\newcommand{\tedge}{\textit{edge}}
\newcommand{\tmove}{\textit{move}}
\newcommand{\tinmotion}{\textit{in_motion}}
\newcommand{\tagenta}{\textit{$A$}}
\newcommand{\tagentb}{\textit{$B$}}
\newcommand{\tv}{\textit{V}}
\newcommand{\tu}{\textit{U}}
\newcommand{\twait}{\textit{wait}}
\newcommand\ASP{\lstinline[language=clingo,columns=flexible,mathescape]}
\newcommand{\vset}{\ensuremath{\mathcal{V}}}
\newcommand{\eset}{\ensuremath{\mathcal{E}}}

Multi-Agent Path Finding~(MAPF,~\cite{ststfekomawaliatcokubabo19b})
is the problem of finding a plan for a set of agents
that move from their start to goal positions.
The agents move on a map represented as a graph $(\vset,\eset)$
where vertices in $\vset$ represent locations and
edges in $\eset$ represent movements between locations, each with a duration.
A valid plan is a sequence of movements for each agent,
free of vertex conflicts and swap conflicts.
Vertex conflicts
occur when two agents occupy the same vertex at the same time,
and swap conflicts
occur when two agents traverse the same edge in opposite directions at the same time.

Our formalization adapts the ASP encoding of~(\cite{bekascsosvwa24a}, Listing 5)
to the metric logic program
in~\eqref{ex:mapf:start}--\eqref{ex:mapf:goal},
capturing movement durations by the timing function.
We use letters $\tu$ and $\tv$ to denote vertices from $\vset$,
and if $(\tu, \tv) \in \eset$, the duration of the movement
is given by $\delta(\tu,\tv)$.
We use letters $\tagenta$ and $\tagentb$ to denote agents, and
each agent $\tagenta$ has a \emph{start} vertex $\tstarta \in \vset$ and
a \emph{goal} vertex $\tgoala \in \vset$.

Rule~\eqref{ex:mapf:start} captures the initial position of each agent.
Rule~\eqref{ex:mapf:move} chooses a possible movement along an edge $(\tu,\tv)$.
Rule~\eqref{ex:mapf:effect} captures the effect of moving from $\tu$ to $\tv$,
which takes exactly $\delta(\tu,\tv)$ time units.
We use the eventually operator rather than the next operator so that
different agents may move asynchronously.
Rule~\eqref{ex:mapf:inertia} captures inertia:
an agent remains at its current vertex $\tu$
if it does not move along any edge $(\tu, \tv)$.
Following~\citep{bekascsosvwa24a},
an agent is at no vertex while moving,
but reappears at the target vertex $\tv$ upon arrival.
Rule~\eqref{ex:mapf:pre} ensures that a movement is only possible
if the agent is currently at the source vertex $\tu$.
This also prevents another action from being taken
until the current movement is completed.
Rule~\eqref{ex:mapf:unique} enforces that each agent moves at most once per time point,
given that $(\tu,\tv)$ and $(\tu',\tv')$ are distinct edges.
Rule~\eqref{ex:mapf:vertexcol} enforces the absence of vertex conflicts,
given that $\tagenta \neq \tagentb$.
Rule~\eqref{ex:mapf:edgecol} enforces the absence of swap conflicts
using the eventually operator to ensure that no swap conflict arises during the traversal of an edge.
Finally, Rule~\eqref{ex:mapf:goal} enforces the goal condition:
every agent must be at its goal vertex at the end of the plan.

\begin{align}
\alwaysF(\tat(\tagenta,\tstarta)                                                      &\leftarrow \initially)\label{ex:mapf:start}\\
\alwaysF(\tmove(\tagenta,\tu,\tv) \vee \neg\tmove(\tagenta,\tu,\tv) &\leftarrow  \neg\finally) \quad \quad \quad \quad \quad \quad (\tu,\tv)\in\eset \label{ex:mapf:move}\\
\alwaysF(\eventuallyF_{\intervco{\delta(\tu,\tv)}{\delta(\tu,\tv)+1}}{\tat(\tagenta,\tv)} &\leftarrow \tmove(\tagenta,\tu,\tv))\label{ex:mapf:effect}\\
\alwaysF(\Next_{\intervco{0}{\omega}}{\tat(\tagenta,\tu)} &\leftarrow \neg\finally \wedge \tat(\tagenta,\tu) \wedge \bigwedge_{(\tu,\tv)\in \eset} \neg\tmove(\tagenta,\tu,\tv))\label{ex:mapf:inertia}
\end{align}\vspace{-2em}\begin{align}
\alwaysF(\bot &\leftarrow \tmove(\tagenta,\tu,\tv) \wedge \neg\tat(\tagenta,\tu))\label{ex:mapf:pre}\\
\alwaysF(\bot &\leftarrow \tmove(\tagenta,\tu,\tv) \wedge \tmove(\tagenta,\tu',\tv')) \quad  (\tu,\tv)\neq(\tu',\tv')\label{ex:mapf:unique}\\
\alwaysF(\bot &\leftarrow \tat(\tagenta,\tu) \wedge \tat(\tagentb,\tu)) \quad \quad \quad \quad \quad \quad \quad \quad \ \ \tagenta\neq\tagentb\label{ex:mapf:vertexcol}\\
\alwaysF(\bot &\leftarrow \tmove(\tagenta,\tu,\tv) \wedge \eventuallyF_{\intervco{0}{\delta(\tu,\tv)}}{\tmove(\tagentb,\tv,\tu)})\label{ex:mapf:edgecol}\\
\alwaysF(\bot &\leftarrow \tgoal(\tagenta,\tu) \wedge \neg\tat(\tagenta,\tu) \wedge \finally)\label{ex:mapf:goal}
\end{align}

\paragraph{Experiments.}
We conducted experiments in two settings
using the encoding in Listing~\ref{prg:mapf},
which is a direct mapping of the rules above.
In each case, we computed one solution.

\begin{lstlisting}[caption={Metric logic program for MAPF},label={prg:mapf},language=clingos]
                     at(A,SA) :- start(A,SA), initially.
move(A,U,V) , not move(A,U,V) :- not finally, agent(A), edge(U,V,D).
  eventually((D,D+1),at(A,V)) :- move(A,U,V), edge(U,V,D).
          next((0,w),at(A,U)) :- not finally, at(A,U), 
                                 not move(A,U,V) : edge(U,V,D).

:- move(A,U,V), not at(A,U).
:- move(A,U,V), move(A,U',V'), (U,V) != (U',V').
:- at(A,U), at(B,U), A != B.
:- move(A,U,V), eventually((0,D), move(B,V,U)), edge(U,V,D).
:- goal(A,U), not at(A,U), finally.
\end{lstlisting}

A preliminary evaluation showed that \clingcon\ was not able
to find a solution without a time-point limit.
The reason is that the MAPF encoding allows all agents to remain idle,
resulting in an unbounded timing function $\tau$.
This is not the case in the dentist scenario,
where at each state there must be a movement that constrains $\tau$.
This poses a challenge for \clingcon,
which attempts to assign times to states without a bound,
and stalls during solving.
In contrast,
\clingodl\ selects the smallest possible positive values,
naturally bounding the search space.
To address this issue while ensuring a fair comparison across all approaches,
we impose the time-point limit $\nu$ (used in the \HT\ approach)
\comment{JR: Before we used \lstinline{v}, here $\nu$, should be consistent}
also for \clingcon\ and \clingodl\ in the \HTC\ approach,
by adding a numerical constraint enforcing assignments to be smaller than $\nu$.
\comment{JR: I am not sure if this fits better here or before.
  This is part of the experiments, so it makes sense here.
  But it breaks a bit the story...
}

\begin{figure}[ht]
   \centering
   \begin{minipage}[b]{0.22\linewidth}
      \includegraphics[width=70pt]{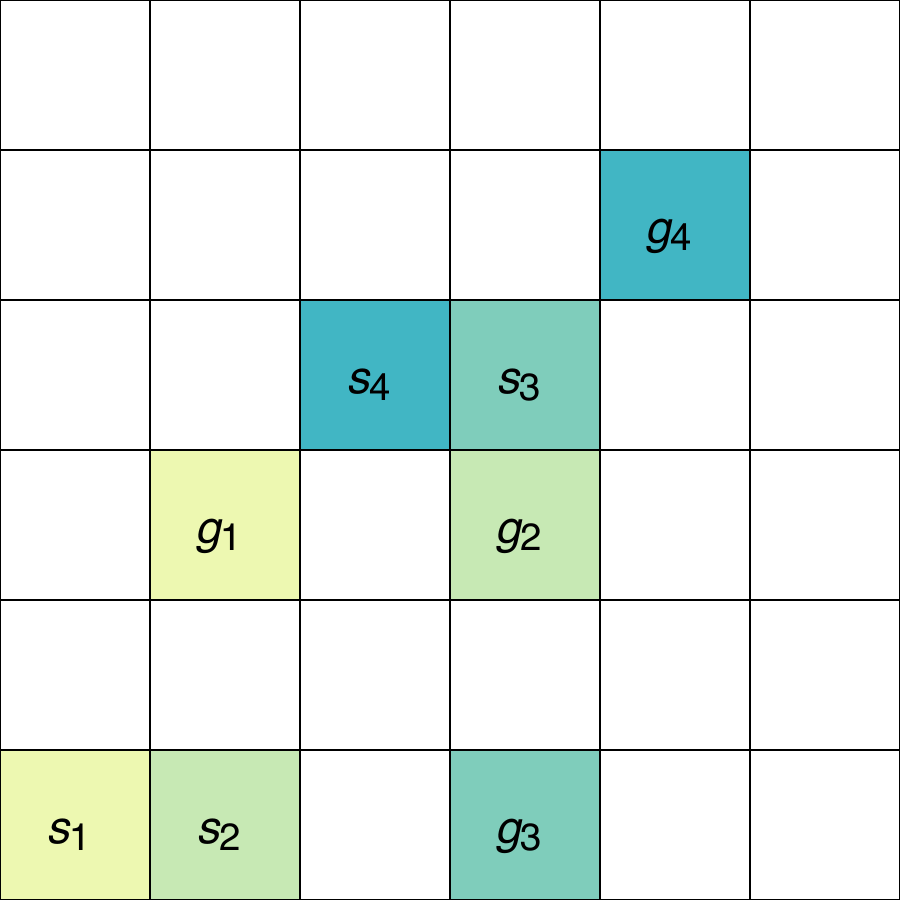}
      \caption{Empty grid (6x6)}
      \label{fig:mapf-empty}
   \end{minipage}
   \begin{minipage}[b]{0.22\linewidth}
      \includegraphics[width=70pt]{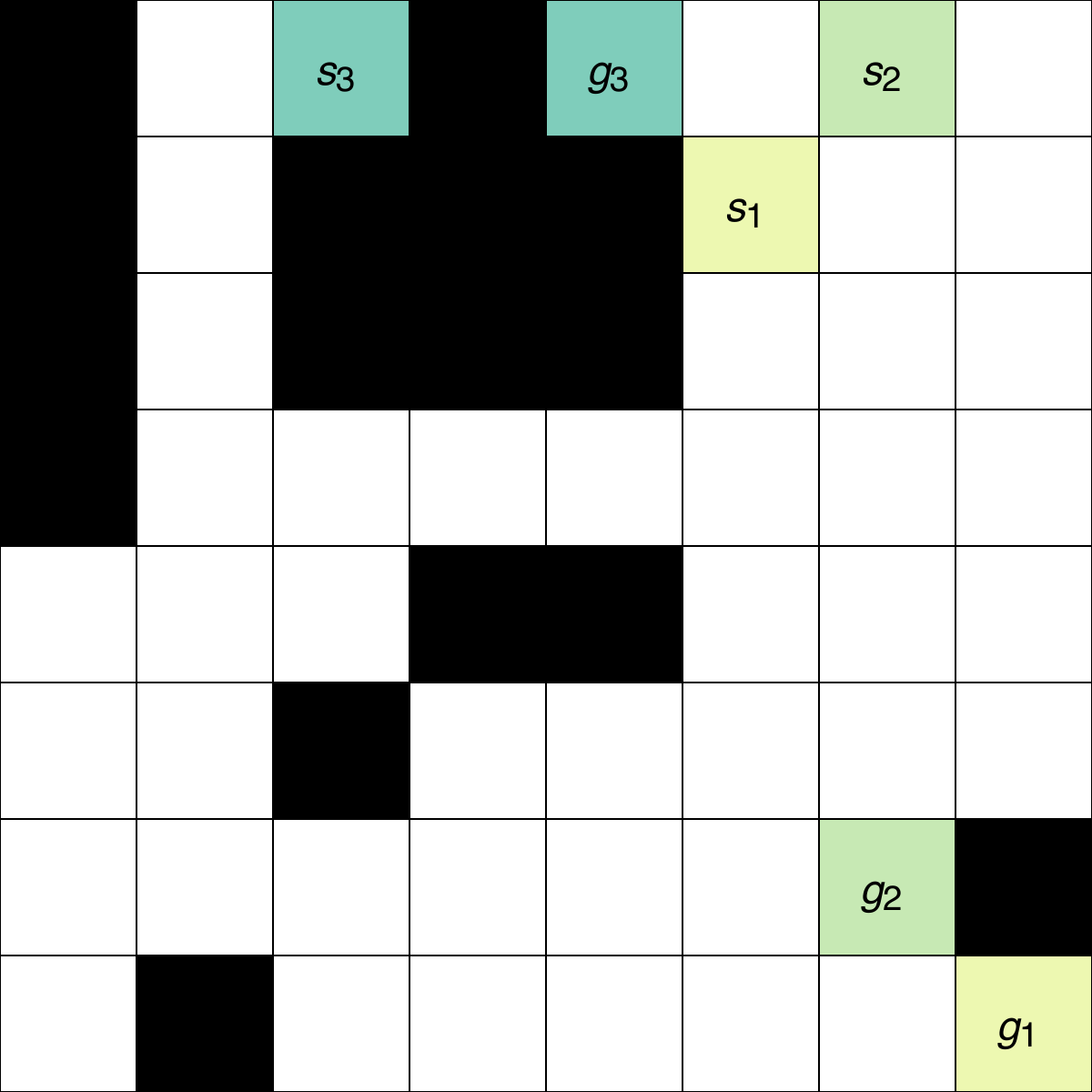}
      \caption{Random grid (8x8)}
      \label{fig:mapf-random}
   \end{minipage}
   \begin{minipage}[b]{0.22\linewidth}
      \includegraphics[width=70pt]{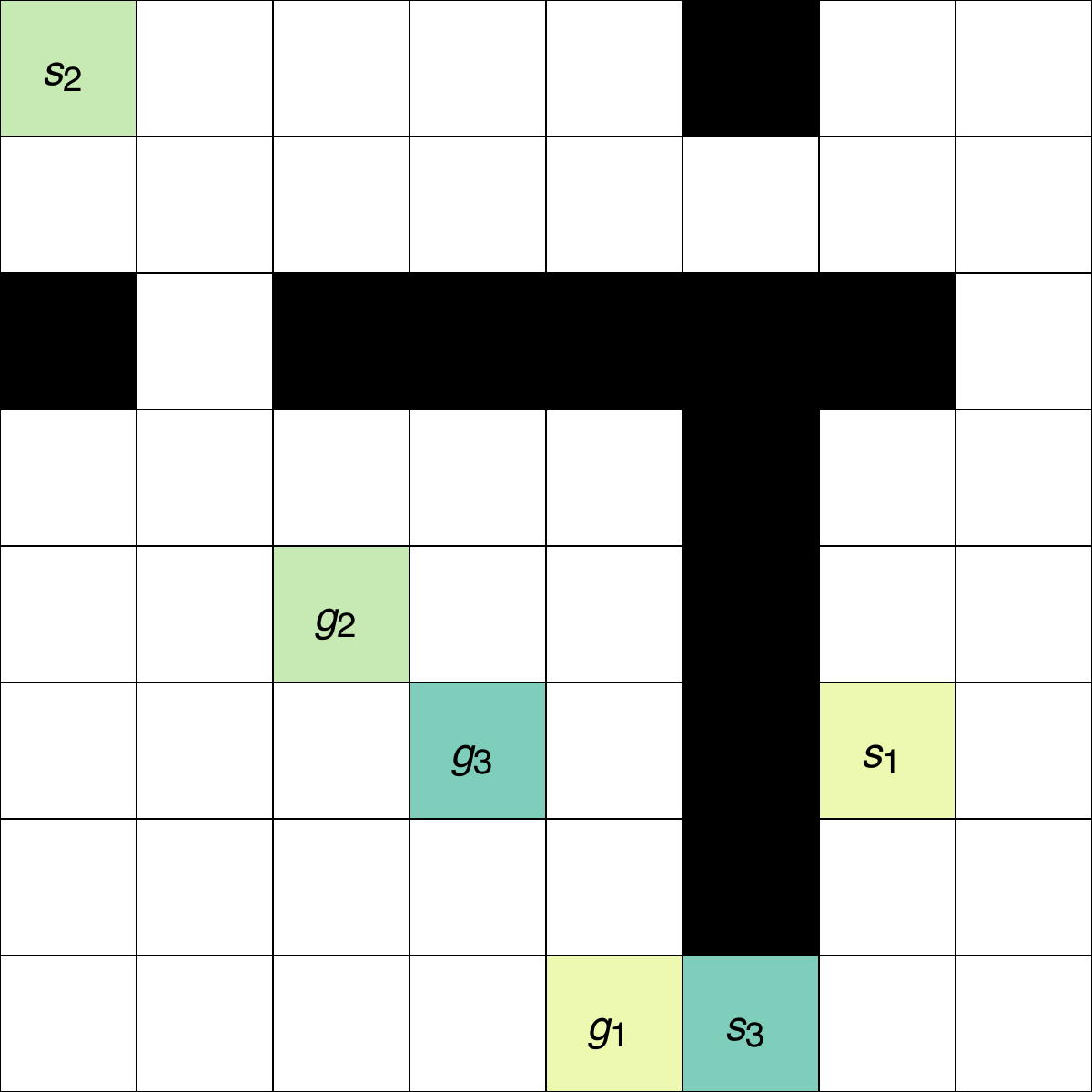}
      \caption{Room grid (8x8)}
      \label{fig:mapf-room}
   \end{minipage}
   \begin{minipage}[b]{0.22\linewidth}
      \includegraphics[width=70pt]{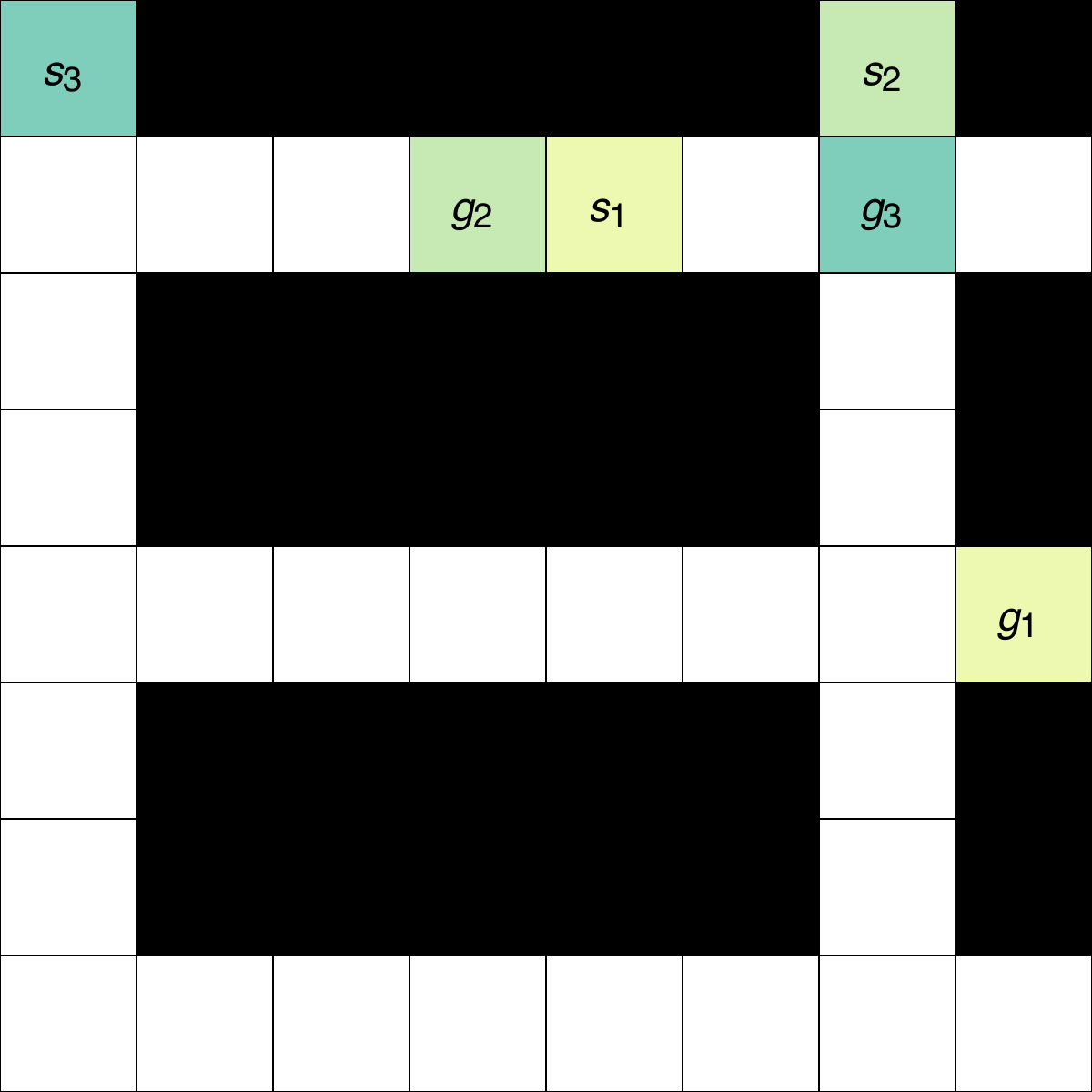}
      \caption{Warehouse grid (8x8)}
      \label{fig:mapf-warehouse}
   \end{minipage}
\end{figure}

In the first setting, we study the scalability of the general fragment
and test the time-granularity independence of the \HTC\ approach
in a more complex scenario.
We use an empty grid map of size $6 \times 6$ (Figure~\ref{fig:mapf-empty})
generated by the framework of~\citet{bekascsosvwa24a}.
We connect vertices with edge durations ranging from 1 to 5 time units,
vary the number of agents from 1 to 4,
and scale edge durations by a factor $f \in \{1, 5, 10, 15, 20, 25, 30, 35\}$,
solving all combinations with $\lambda = 10$ and $\nu = 15 \cdot f$.
Figure~\ref{fig:mapf-plots} shows the performance of each approach
across the different scaling factors.
The reported time is the total time to find the first solution,
and the shaded area below each line indicates the grounding time.
The \HT\ approach using \clingo\ exhibits behavior similar to before: larger edge durations lead to higher runtimes and early timeouts,
as well as higher grounding times,
visible in the light green shaded area for the single-agent case
but not for larger instances due to timeouts.
In contrast, the \HTC\ approaches with \clingcon\ and \clingodl\
show stable performance across all scaling factors,
confirming their independence from time granularity.
Their grounding time is negligible across all agent counts,
but the total time increases with the number of agents, as expected.
Overall, \clingodl\ performs better than \clingcon.
Figure~\ref{fig:mapf-rules} shows the number of ground rules generated by each approach.
The number of rules for \clingo\ grows rapidly as the factor increases,
while for \clingcon\ and \clingodl\ the number of rules remains independent of the time granularity,
growing only slightly with the number of agents due to the additional rules introduced per agent.
The instances where \clingo\ timed out are also plotted, as
the timeout was reached during solving rather than grounding.

\begin{figure}[ht]
   \centering
   \includegraphics[width=0.9\linewidth]{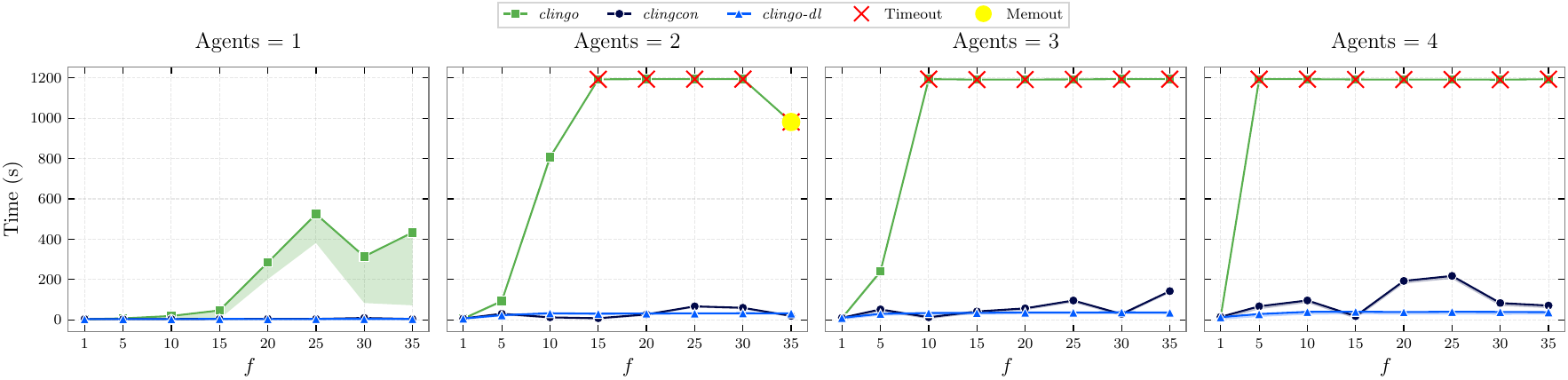}
   \caption{Total times for empty 6x6 grid MAPF instances with increasing number of agents.
   The shaded areas under the lines represent grounding time.}
   \label{fig:mapf-plots}
\end{figure}
\comment{JR: I am not sure about this ``Total'' times,
    but before we had ``Grounding and Solving'' which was also unclear}

\begin{figure}[ht]
   \centering
   \includegraphics[width=0.9\linewidth]{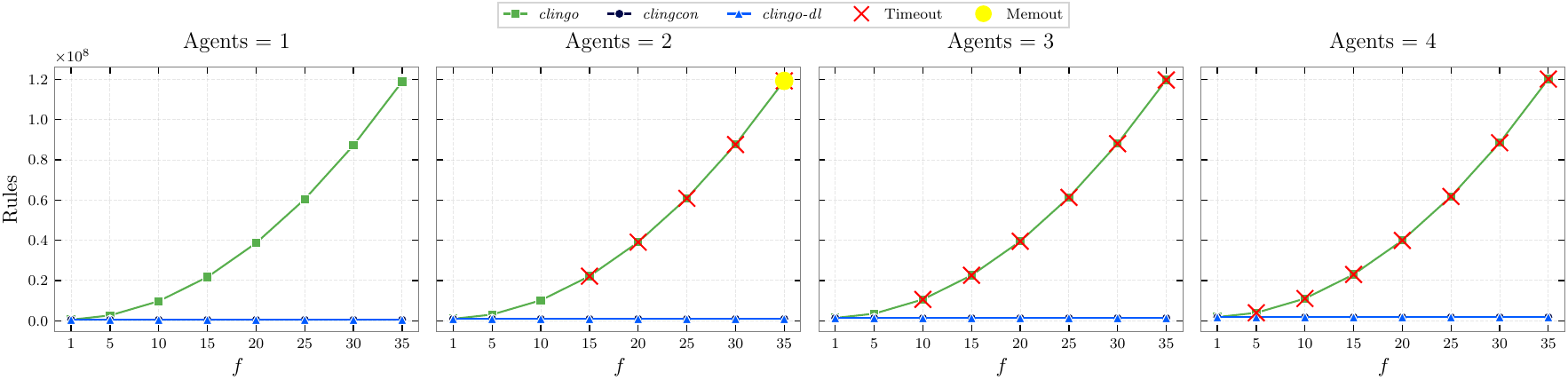}
   \caption{Number of ground rules for empty 6x6 grid MAPF instances with increasing number of agents.}
   \label{fig:mapf-rules}
\end{figure}

In the second setting,
we study the effect of relaxing the horizon on solver performance
across different map structures.
We generated three MAPF instances on $8 \times 8$ grids,
each with a different map layout: random, room, and warehouse
(Figures~\ref{fig:mapf-random},~\ref{fig:mapf-room}, and~\ref{fig:mapf-warehouse}, respectively),
fixing the number of agents to three,
and sampling edge durations between 1 and 5 time units.
We set a time-point limit of $\nu = 100$ and vary the horizon $\lambda$ from
$15$ to $45$ in steps of $5$.
The results are shown in Figure~\ref{fig:mapf8-plot}.
Compared to the previous setting, grounding has a more pronounced impact,
reflecting the larger map size and higher time-point limit.
\Clingodl\ achieves the best performance across all three layouts.
The behavior of \clingcon\ is more irregular across layouts and values of $\lambda$,
and requires further investigation.
In general, as $\lambda$ increases,
solving becomes faster because the problem becomes less constrained.
However, \clingo's grounding time grows with $\lambda$, creating a tradeoff
that yields an optimal operating point, beyond which total time increases again.
This underscores the importance of selecting an appropriate horizon $\lambda$.

\begin{figure}[ht]
   \centering
   \includegraphics[width=\linewidth]{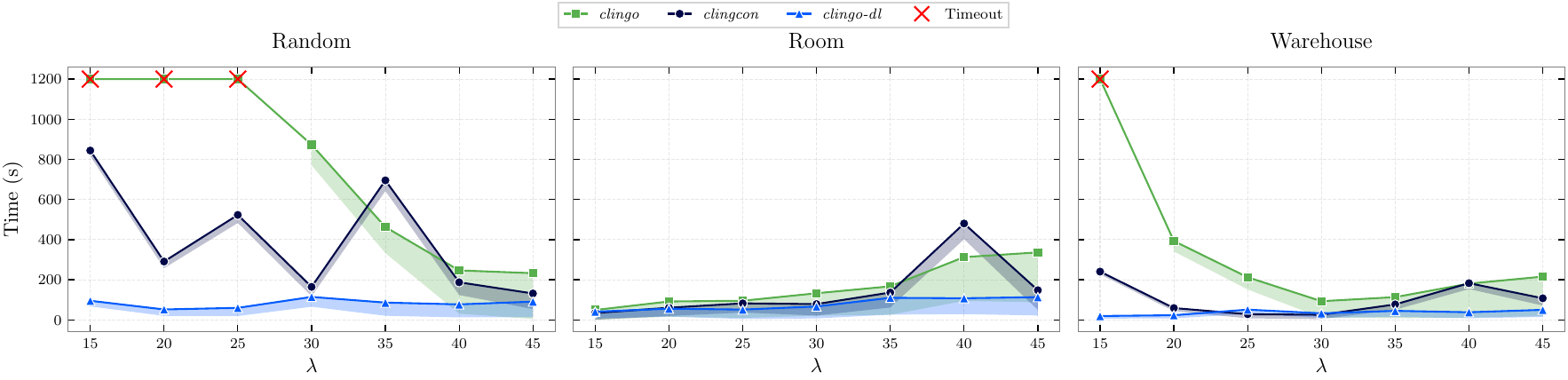}
\caption{Total times for 8x8 grid MAPF instances with increasing horizon.
   The shaded areas under the lines represent grounding time.}
   \label{fig:mapf8-plot}
\end{figure}

\subsubsection{Job-Shop Scheduling}

\newcommand{\tprev}{\textit{prev}}
\newcommand{\startrun}{\textit{start\_run}}
\newcommand{\hasrun}{\textit{has\_run}}
\newcommand{\run}{\textit{run}}
\newcommand{\before}{\textit{before}}
\newcommand{\toperation}{\textit{O}}
\newcommand{\tmach}{\textit{M}}
\newcommand{\tjob}{\textit{J}}

In the Job-Shop Scheduling problem,
there is a set of jobs $\mathcal{J}$, operations $\mathcal{O}$ and machines $\mathcal{M}$.
Each job $J\in \mathcal{J}$ consists of a sequence of operations.
Each operation $O\in \mathcal{O}$ belongs to a single job
and must be processed on machine $\mu(O)$ for a duration $\delta(O)$.
By $\tprev(O)$ we denote the operation that precedes $O \in \mathcal{O}$ in the sequence, if it exists.
The goal is to find a schedule such that all operations are completed in order
within the given time.

Our formalization as a general metric logic program is shown
in~\eqref{ex:job:start}--\eqref{ex:job:goal}.
For some operation $O$,
the atom $\startrun_{O}$ represents its start,
$\run_{O}$ represents that it is running,
$\hasrun_O$ that it has already run, and
$\twait_{O}$ that it is waiting.
Rule~\eqref{ex:job:start} captures the non-deterministic choice to start an operation,
provided it has not yet completed.
Rule~\eqref{ex:job:duration} represents the execution of an operation for its duration,
and Rule~\eqref{ex:job:completion} captures its completion after that duration.
Rule~\eqref{ex:job:machinelimit} prevents two distinct operations from running
on the same machine at the same time,
and Rule~\eqref{ex:job:order} prevents an operation from starting
if its preceding operation in the job has not yet completed.
Finally, Rule~\eqref{ex:job:goal} enforces that all operations complete
by the end of the trace.

\begin{align}
\alwaysF(\twait(\toperation)\vee \startrun_{\toperation} &\leftarrow \neg \hasrun_\toperation)\label{ex:job:start}\\
\alwaysF(\alwaysF_{\intervco{0}{\delta(\toperation)}}{\run_{\toperation}} &\leftarrow \startrun_{\toperation})\label{ex:job:duration}\\
   \alwaysF(\alwaysF_{\intervco{\delta(\toperation)}{\omega}}{\hasrun_{\toperation}} &\leftarrow \startrun_{\toperation})\label{ex:job:completion}\\
\alwaysF(\bot &\leftarrow \run_{\toperation} \wedge \run_{\toperation'}) \quad \quad \quad O \neq O'\text{ and }\mu(\toperation)=\mu(\toperation') \label{ex:job:machinelimit}\\
   \alwaysF(\bot &\leftarrow \startrun_{\toperation} \wedge \neg \hasrun_{\tprev(\toperation)}) \label{ex:job:order}\\
\alwaysF(\bot &\leftarrow \neg \hasrun_\toperation \wedge \finally) \label{ex:job:goal}
\end{align}

\paragraph{Experiments.}
We study the impact of varying the horizon on the overall scalability of each approach.
We evaluate the benchmark instance \textit{ft06} from~\citet{muttho63a},
which consists of $6$ jobs, $6$ machines, and $36$ operations.
The experiments were performed with the known optimal makespan of $55$,
and in a relaxed setting with a time-point limit of $110$ (double the optimal makespan),
to analyze the impact of enforcing optimality via the time-point limit.
The benchmark instance is translated into ASP facts
using predicates \texttt{operation/1} for $\mathit{O}$,
\texttt{machine/2} for $\mu$, \texttt{duration/2} for $\delta$,
and \texttt{job/2} for the job assignment.
The order of operations is captured by incremental numerical identifiers.
The encoding, shown in Listing~\ref{prg:jobshop},
is a direct mapping of Rules~\eqref{ex:job:start}--\eqref{ex:job:goal}.

\begin{lstlisting}[caption={Metric logic program for Job-Shop Scheduling},
                 label={prg:jobshop},
                 language=clingos]
  wait(O) , start_run(O) :- not has_run(O), operation(O).
    always((0,D),run(O)) :- start_run(O), duration(O,D).
always((D,w),has_run(O)) :- start_run(O), duration(O,D).

:- run(O), run(O'), O'!=O, machine(O,M), machine(O',M).
:- start_run(O), not has_run(O-1), job(O,J), job(O-1,J).
:- not has_run(O), finally, operation(O).
\end{lstlisting}

Figure~\ref{fig:jobshop-plot-optimal} shows the performance of each approach
across different values of $\lambda$,
with the time-point limit $\nu$ fixed to the known optimal makespan of $55$.
The problem is unsatisfiable for $\lambda \in \{10, 15\}$;
with the optimal time-point limit of $55$, a horizon of $\lambda > 15$ is required to obtain a solution.
\Clingo\ times out for $\lambda \in \{20, 25, 30\}$,
while \clingcon\ and \clingodl\ time out at $\lambda = 25$.
For all other satisfiable instances, \clingcon\ and \clingodl\ are
significantly faster than \clingo.
Unlike in the previous applications, \clingcon\ shows overall better performance
than \clingodl\ in this setting,
though this warrants further analysis of the constraint solvers' internal behaviour.
The solving time for \clingo\ generally decreases as $\lambda$ increases beyond $35$,
which can be attributed to the fact that a smaller horizon makes the problem harder to solve,
similarly to the MAPF results.
The fast response observed at $\lambda = 50$ for all approaches
can be attributed to the fact that, with a time-point limit of $55$,
this horizon is very close to an identity mapping for $\tmf$,
which simplifies the search considerably.
The grounding time is negligible for all approaches,
as this job-shop instance is designed to be hard to solve but not to ground.

Figure~\ref{fig:jobshop-plot-relaxed} shows the results with a relaxed time-point limit of $110$.
Relaxing the time-point limit eliminates most of the timeouts observed for \clingo\
and removes the erratic behaviour seen in Figure~\ref{fig:jobshop-plot-optimal},
with only $\lambda = 10$ remaining unsolvable within the time limit.
The \HTC\ approach finds a solution immediately across all satisfiable instances.
As in the previous applications, the grounding time for \clingo\ grows with the horizon,
so the relaxed time-point limit benefits smaller horizons where solutions can be found,
but introduces overhead for larger ones.
Figures~\ref{fig:jobshop-plot-optimal-rules} and~\ref{fig:jobshop-plot-relaxed-rules} show that
the number of ground rules for \clingo\ grows linearly with $\lambda$,
reflecting the propositional grounding of time-step constraints,
whereas \clingcon\ and \clingodl\ maintain a compact representation
independent of the time-point limit.

\begin{figure}[ht]
   \centering
   \begin{minipage}[b]{0.45\linewidth}
      \includegraphics[width=\linewidth]{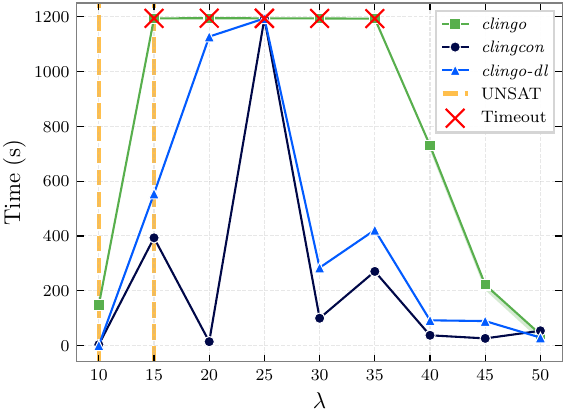}
      \caption{Total times for \textit{ft06} with time-point limit $\nu = 55$.}
      \label{fig:jobshop-plot-optimal}
   \end{minipage}\hfill
   \begin{minipage}[b]{0.45\linewidth}
      \includegraphics[width=\linewidth]{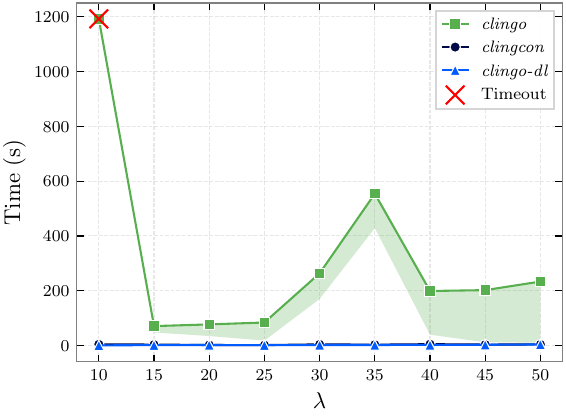}
      \caption{Total times for \textit{ft06} with relaxed time-point limit $\nu = 110$.}
      \label{fig:jobshop-plot-relaxed}
   \end{minipage}
\end{figure}

\begin{figure}[ht]
   \centering
   \begin{minipage}[b]{0.45\linewidth}
      \includegraphics[width=\linewidth]{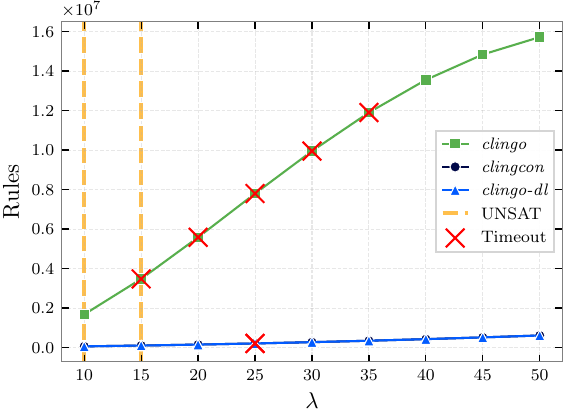}
      \caption{Number of ground rules for \textit{ft06} with time-point limit $\nu = 55$.}
      \label{fig:jobshop-plot-optimal-rules}
   \end{minipage}\hfill
   \begin{minipage}[b]{0.45\linewidth}
      \includegraphics[width=\linewidth]{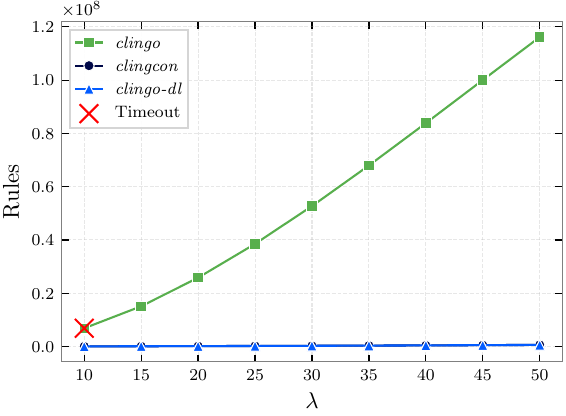}
      \caption{Number of ground rules for \textit{ft06} with relaxed time-point limit $\nu = 110$.}
      \label{fig:jobshop-plot-relaxed-rules}
   \end{minipage}
\end{figure}

  \section{Conclusion}\label{sec:discussion}

We presented a computational approach to metric ASP that allows for fine-grained timing constraints.
In doing so, we proposed two metric fragments.
The first fragment restricts the use of metric operators in rule heads, allowing only the next operator.
The expressiveness of this fragment showed to be enough to model basic transitions.
However, it lacks the expressiveness needed for capturing global metric conditions.
The second fragment, on the other hand,
permits a broader use of global metric operators,
significantly increasing expressiveness at the cost of higher computational complexity.
For each fragment,
we developed two alternative translations from firm semantic foundations, and proved their completeness and correctness.
Our second translation has a clear edge over the first one, when it comes to a fine-grained resolution of time.
This is achieved by outsourcing the treatment of time by using difference constraints.
However, a prominent use case involves employing the identity (timing) function,
where intervals reference only state indices within traces.
This coarser notion of time reduces the discrepancy between our two translations.
Further improvement is possible through more sophisticated Boolean encodings, such as an order
encoding~\citep{crabak94a,bageinospescsotawe15a}.

We use meta-programming for implementation
as it facilitates development and rapid prototyping for different fragments and translations.
Our empirical analysis allowed us to compare translations and target systems,
providing a basis for future work to improve the systems performance.

Overall our hybrid approach is superior to the Boolean one.
Among the former,
\clingodl\ outperforms \clingcon\ in solving time across all benchmarks.
Note however that the semantics developed in this work corresponds to that of \clingcon,
while \clingodl\ only computes canonical models (cf.~\cite{cafascwa23a}) .
\Clingodl\ has furthermore the advantage that it does not require a
fixed time-point limit in advance.
Selecting such a limit is non-trivial in practice: its value is rarely known beforehand,
it critically affects performance, and an inadequate choice may require multiple solver calls.
Moreover, without a finite bound, \clingcon\ cannot terminate its search since the timing
function might be unbounded, while for \clingo\ the time-point limit directly determines the
size of the ground program, making its choice critical for grounding performance.

Outsourcing the treatment of time is not for free either.
\Clingodl\ maps difference constraints into graphs, whose nodes are time variables and weighted edges reflect the actual
constraints.
This results in a quadratic space complexity.
\Clingcon\ pursues a lazy approach to constraint solving that gradually unfolds an ASP encoding of linear constraints.
In the worst case, this amounts to the space requirements of our first translation.
As well, such constraints are hidden from the ASP solver and cannot be used for directing the search.
Hence,
despite our indicative observations,
a detailed empirical analysis is needed to account for the subtleties of our translation and its target systems.
For instance,
the irregular behavior of \clingcon\ across instances and map layouts
should be further investigated, as it suggests sensitivity to factors
not yet fully understood.

 \vspace{30pt}

\noindent \textbf{Competing interests declaration.} The authors declare none.
\bibliographystyle{plainnat} 

\appendix
\section{Three-valued semantics}
\label{sec:three}
In our framework,
\HT-interpretations (and -traces) naturally lend themselves to three-valued models,
where each atom is classified as false, assumed, or proven (in a given state).
This correspondence arises from the fact that the intermediate logic of \HT\ aligns
with G\"odel's three-valued logic $G_3$~\citep{goedel32a}.
Consequently, formula satisfaction in \HT\ can be directly represented by a three-valued function, $\bm{m}(\alpha)$,
assigning values 0 (false), 1 (assumed), or 2 (proven) to each formula $\alpha$.
The significance of this multi-valued truth assignment lies in its ability to simplify \HT-equivalence,
reducing $\alpha \equiv \beta$ to a straightforward comparison of truth values, $\bm{m}(\alpha)=\bm{m}(\beta)$.
This property proves invaluable when introducing auxiliary atoms, enabling the replacement of complex formulas with simpler atoms
while ensuring semantic equivalence.
Specifically, by introducing an auxiliary atom $a$ to represent $\alpha$,
we can safely substitute $\alpha$ with $a$,
provided we add formulas that guarantee $\bm{m}(\alpha)=\bm{m}(a)$.
Building upon these principles, we now present the following formal definitions.
\subsection{Three-valued semantics for \HT}
\label{sec:three:ht}

The following definitions follow the ones in~\citep{capeva05a}.
\begin{definition}[Three-valued interpretation]\label{def:three-valued-interpretation-ht}
  A three-valued interpretation over \alphabet\ is a function $\trivalI: \alphabet \rightarrow \{ 0, 1, 2 \}$.
\end{definition}
A three-valued interpretation is total, if its range is $\{0,2\}$.
\begin{definition}[Extension of three-valued interpretation]\label{def:three-valued-extension-ht}
  A three-valued interpretation $\trivalI$ over \alphabet\ is extended to Boolean formulas as follows.
  \begin{align*}
    \trivalI(\bot) & = 0 \\
    \trivalI(\varphi \wedge \psi) & = \min(\trivalI(\varphi), \trivalI(\psi)) \\
    \trivalI(\varphi \vee \psi) & = \max(\trivalI(\varphi), \trivalI(\psi)) \\
    \trivalI(\varphi \rightarrow \psi) & =
                                         \begin{cases}
                                           2              & \text{if } \trivalI(\varphi) \leq \trivalI(\psi) \\
                                           \trivalI(\psi) & \text{otherwise}
                                         \end{cases}
  \end{align*}
\end{definition}
For the derived operators we have
$\trivalI(\top) = 2$ and
$\trivalI(\neg \varphi)=2$ if $\trivalI(\varphi)=0$ and $0$ otherwise.

\begin{proposition}[Satisfaction total three-valued interpretations]
  For any total three-valued interpretation $\trivalI$
  and for any propositional formula $\varphi$,
  $\trivalI(\varphi)\in \lbrace 0,2\rbrace$.
\end{proposition}

\begin{proposition}[Equivalent formulas]\label{prop:three-valued:eq:ht}
  For any propositional formulas $\varphi$
  and $\psi$,
  $\varphi \equiv \psi$ iff $\trivalI(\varphi) = \trivalI(\psi)$
  for any three-valued interpretation $\trivalI$.
\end{proposition}

In order to prove that \HT{} and its three-valued characterization are equivalent,
we first define a mapping between \HT{}
and three-valued mappings.
Then we prove that such correspondence can be extended
to the full propositional language.

\begin{definition}[Correspondence between \HT{} and three-valued interpretations]\label{def:three-valued-from-ht}
  Given an \HT\ interpretation $\handt$ over \alphabet,
  we define its corresponding three-valued interpretation $\trivalI$ over \alphabet\
  for each atom $a\in\alphabet$
  as
  \begin{align*}
    \trivalI(a) & \eqdef
                          \begin{cases}
                            0 & \text{ if } a \not\in T        \\
                            1 & \text{ if } a \in T\setminus H \\
                            2 & \text{ if } a \in H
                          \end{cases}
  \end{align*}

 Conversely, given a three-valued interpretation $\trivalI$ over \alphabet,
 we define its corresponding \HT\ interpretation over \alphabet\
 as $\handt$,
 where
 \begin{align*}
 	T & \eqdef \{a \in \alphabet \mid \trivalI(a) > 0\}\\
 	H & \eqdef \{a \in \alphabet \mid \trivalI(a) = 2\}
 \end{align*}
\end{definition}
\begin{proposition}[Equi-satisfaction]\label{prop:three-valued-ht-properties}Given any \HT\ interpretation $\handt$ and any three-valued interpretation $\trivalI$
  satisfying the correspondence
  established in Definition~\ref{def:three-valued-from-ht}, it follows that

  	\begin{eqnarray*}
  		\trivalI(\varphi) = 2 &\hbox{ iff }& \handt \models \varphi \hbox{ and }\\
  		\trivalI(\varphi) > 0 &\hbox{ iff }& \tuple{T,T} \models \varphi
  	\end{eqnarray*}
  \noindent for all propositional formulas $\varphi$.
\end{proposition}

\begin{corollary}[Total three-valued correspondence]\label{prop:three-valued:total}Let $\tuple{T,T}$ a total \HT{} interpretation and let $\trivalI$ a three-valued interpretation
  satisfying the correspondence
  established in Definition~\ref{def:three-valued-from-ht},
  it follows that

  \begin{eqnarray*}
  	\trivalI(\varphi) = 2 &\hbox{ iff }& \tuple{T,T} \models \varphi \hbox{ and }\\
  	\trivalI(\varphi) = 0 &\hbox{ iff }& \tuple{T,T} \not \models \varphi.
  \end{eqnarray*}
  \noindent for all propositional formulas $\varphi$.
\end{corollary}
 \subsection{Three-valued semantics for \HTC}
\label{sec:three:htc}
\begin{definition}[Corresponding three-valued interpretation]\label{def:htc:three-val:interpretation}
  Given an \HTC\ interpretation $\htcinterp$ and a denotation $\den{\cdot}$ over $\tuple{\mathcal{X},\mathcal{D},\mathcal{C}}$,
  we define its corresponding three-valued interpretation \trivalIhtc\ over $\mathcal{C}$
  for each constraint atom $c\in\mathcal{C}$
  as\footnote{Since we use strict denotations, it is sufficient to check in the first case only satisfaction in $h$.}
  \begin{align*}
    \trivalIhtc(c) & \eqdef
                     \begin{cases}
                       2 & \hbox{ if }	\Vh \in \den{c}                                   \\
                       1 & \hbox{ if }	\Vh \not \in \den{c} \text{ and } \Vt \in \den{c} \\
                       0 & \hbox{ if }	\Vt \not \in \den{c}
                     \end{cases}
  \end{align*}
\end{definition}
Note that $\trivalIhtc$ is total when $h=t$.

\begin{definition}[Extension of three-valued interpretation]\label{def:def:htc:three-val:extension}
  A three-valued interpretation $\trivalIhtc$ corresponding to an \HTC\ interpretation $\htcinterp$,
  is extended to Boolean constraint formulas as follows.
  \begin{align*}
    \trivalIhtc({\bot}) & = 0 \\
    \trivalIhtc({\varphi \wedge \psi}) & = \min\{\trivalIhtc({\varphi}), \trivalIhtc({\psi})\} \\
    \trivalIhtc({\varphi \vee \psi}) & = \max\{\trivalIhtc({\varphi}), \trivalIhtc({\psi})\} \\
    \trivalIhtc({\varphi \to \psi}) & =
                                      \begin{cases}
                                        2                   & \text{if } \trivalIhtc({\varphi}) \leq \trivalIhtc({\psi}) \\
                                        \trivalIhtc({\psi}) & \text{otherwise}
                                      \end{cases}
  \end{align*}
\end{definition}
For the derived operators we have
$\trivalIhtc(\top) = 2$ and
$\trivalIhtc(\neg \varphi)=2$ if $\trivalIhtc(\varphi)=0$ and $0$ otherwise.

\begin{proposition}[Satisfaction total three-valued interpretations]
  For any total three-valued interpretation $\trivalItotalhtc$ and for any formula $\varphi$, $\trivalItotalhtc(\varphi)\in \lbrace 0,2\rbrace$.
\end{proposition}
\begin{proof}
  Proven by straightforward structural induction on $\varphi$.
\end{proof}

\begin{proposition}[Equivalence of formulas]\label{prop:three-valued:eq:htc}
  For any propositional formulas $\varphi$ and $\psi$, $\varphi \equiv \psi$ iff  $\trivalIhtc(\varphi) = \trivalIhtc(\psi)$
  for any $\tuple{h,t}$.
\end{proposition}
\begin{proof}
  Suppose $\varphi \equiv \psi$
  but there is some $\trivalIhtc$
  such that
  $\trivalIhtc(\varphi) \neq \trivalIhtc(\psi)$.
  By the three-valued semantics,
  $\trivalIhtc(\varphi\leftrightarrow\psi)\neq 2$,
  This contradicts the fact that $\varphi \equiv \psi$.
The converse direction is proven analogously.
\end{proof}

To prove that \HTC{} and our (alternative) three valued characterization are equivalent,
there is no need to stablish a model correspondence
since the three-valued characterization is induced by and \HTC{} interpretation.
Consequently, the following lemma for the full propositional language can be directly proved.

\begin{proposition}[Equi-satisfaction]
  \label{prop:htc:equi-satisfaction}
	Given an \HTC{} interpretation $\tuple{h,t}$ ,
  it follows that

	\begin{eqnarray*}
		\trivalIhtc(\varphi) = 2  &\hbox{ iff } & \tuple{h,t} \models \varphi \hbox{ and }\\
		\trivalIhtc(\varphi) > 0 &\hbox{ iff } &   \tuple{t,t} \models \varphi.
	\end{eqnarray*}
  \noindent for all propositional formulas $\varphi$.

\end{proposition}
\begin{proof}
  Proven by structural induction on $\varphi$,
  taking into account that we have two induction hypotheses,
  one for each of the two cases in the proposition.
\end{proof}

\begin{proposition}[Total three-valued correspondence]
	Given a total \HTC{} interpretation $\tuple{t,t}$,
  it follows that

	\begin{eqnarray*}
    \trivalItotalhtc(\varphi) =2  &\hbox{ iff } &  \tuple{t,t} \models \varphi \hbox{ and }\\
    \trivalItotalhtc(\varphi) = 0 &\hbox{ iff } &  \tuple{t,t} \not \models \varphi.
	\end{eqnarray*}
  \noindent for all propositional formulas $\varphi$.

\end{proposition}
\begin{proof}
  Follows from Proposition~\ref{prop:htc:equi-satisfaction}.
\end{proof}

 \subsection{Three-valued semantics for \MHT}
\label{sec:three:mht}

\begin{definition}[Three-valued metric interpretation]\label{def:three-valued-interpretation-mht}
  Given a timing function $\tmf: \intervco{0}{\lambda}\to\mathbb{N}$ with $\lambda \in \mathbb{N}$,
  we define a three-valued metric interpretation $\trivalIm$
  over alphabet $\alphabet$
  for
  as the function
  \begin{align*}
    \trivalIm&: \intervco{0}{\lambda}\times\alphabet\to\{0,1,2\}
  \end{align*}
\end{definition}
The three-valued interpretation $\trivalIm$ assigns to each state $k\in\intervco{0}{\lambda}$ and each atom $a\in\alphabet$
a truth value in $\{0,1,2\}$.
In analogy to Section~\ref{sec:mht},
the associated time function $\tmf$ captures the time of each state.

A three-valued metric interpretation $\trivalIm$ is total,
if its range is $\{0,2\}$ for all $\rangeco{k}{0}{\lambda}$.

\begin{definition}[Extension of three-valued metric interpretation]\label{def:three-valued-semantics-mht}
  A three-valued metric interpretation $\trivalIm$,
  is extended to metric formulas
  for any time point $k \in \intervco{0}{\lambda}$
  as follows.
  \begin{align*}
    \trival{k}{\bot} & \eqdef 0 \\
    \trival{k}{\varphi \wedge \psi} & \eqdef \min\{\trival{k}{\varphi}, \trival{k}{\psi}\} \\
    \trival{k}{\varphi \vee \psi}   & \eqdef \max\{\trival{k}{\varphi}, \trival{k}{\psi}\} \\
    \trival{k}{\varphi \to \psi} & \eqdef
                                   \begin{cases}
                                     2                & \text{if } \trival{k}{\varphi} \leq \trival{k}{\psi} \\
                                     \trival{k}{\psi} & \text{otherwise}
                                   \end{cases} \\
    \trival{k}{\initially} & \eqdef
                             \begin{cases}
                               2 & \text{if } k = 0 \\
                               0 & \text{otherwise}
                             \end{cases} \\
    \trival{k}{\metricI{\Next} \varphi} & \eqdef
                                          \begin{cases}
                                            0                     & \text{if } k + 1 = \lambda \text{ or } \tau(k + 1) - \tau(k) \not\in I \\
                                            \trival{k + 1}{\varphi} & \text{otherwise}
                                          \end{cases} \\
    \trival{k}{\metricI{\eventuallyF} \varphi} & \eqdef \max\{0\}\cup \{\trival{j}{\varphi} \mid k \leq j < \lambda \text{ and } \tau(j) - \tau(k) \in I\} \\
    \trival{k}{\metricI{\alwaysF} \varphi}     & \eqdef \min\{2\}\cup \{\trival{j}{\varphi} \mid k \leq j < \lambda \text{ and } \tau(j) - \tau(k) \in I\}
  \end{align*}
\end{definition}
Unlike in the three-valued semantics of \THTf~\citep{cakascsc18a},
the sets for eventually and always might be empty due to the restrictions on the interval.
Therefore, we add the default values, 0 and 2, respectively.

\begin{proposition}[Satisfaction total three-valued metric interpretations] For any total three-valued valuation $\trivalIm$, for any metric formula $\varphi$ and for all $\rangeco{k}{0}{\lambda}$,
	 $\trival{k}{\varphi}\in \lbrace 0,2\rbrace$.
\end{proposition}
\begin{proof}
  Proven by straightforward structural induction on $\varphi$.
\end{proof}

\begin{proposition}[Point-wise equivalence for formulas]
  \label{prop:three-valued:eq:mht:pointwise}
  Given a timed \HT-trace $\M=(\tuple{\Htrace,\Ttrace}, \tmf)$
  and its corresponding three-valued metric interpretation $\trivalIm$,
  for any arbitrary metric formulas $\varphi$ and $\psi$,
  and $\kinlambda$,
  it follows that
  $\M,k\models\varphi\leftrightarrow\psi$ iff $\trival{k}{\varphi} = \trival{k}{\psi}$.
\end{proposition}

\begin{proof}
  $\M,k\models\varphi\leftrightarrow\psi$ iff
  $\trivalIm(\varphi\leftrightarrow\psi)=2$.
  We know this is the case iff
  $\trival{k}{\varphi} = \trival{k}{\psi}$
  by the definition of the metric three-valued semantics.
\end{proof}

\begin{proposition}[Equivalent formulas]
  \label{prop:three-valued:eq:mht}
  Two arbitrary metric formulas $\varphi$ and $\psi$ are equivalent,
  in symbols $\varphi \equiv \psi$,
  iff $\trival{k}{\varphi} = \trival{k}{\psi}$
  for any three-valued metric interpretation $\trivalIm$,
  any $\rangeco{k}{0}{\lambda}$ and any timed function $\tau$.
\end{proposition}
\begin{proof}
  Suppose $\varphi \equiv \psi$
  but there is some timed function $\tau$ and $\trivalIm$
  such that
  for some $\rangeco{k}{0}{\lambda}$,
  $\trival{k}{\varphi} \neq \trival{k}{\psi}$.
  By the three-valued semantics,
  $\trival{k}{\varphi\leftrightarrow\psi}\neq 2$,
  This contradicts the fact that $\varphi \equiv \psi$.
The converse direction is proven analogously.
\end{proof}

As in the propositional case,
we prove that $\MHT$ and our three-valued characterization are equivalent. We start by defining a model correspondence.

\begin{definition}[Correspondence between \MHT{} and three-valued metric interpretations]
  \label{def:three-valued:mht:correspondence}
  Given a timed \HT-trace $\M=(\tuple{\Htrace,\Ttrace}, \tmf)$ of length $\lambda$ over \alphabet,
we define its corresponding three-valued metric interpretation $\trivalIm$ over \alphabet\ as

  \begin{align*}
    \trival{k}{a} & \eqdef
                     \begin{cases}
                       0 & \text{ if } a \not\in T_k          \\
                       1 & \text{ if } a \in T_k\setminus H_k \\
                       2 & \text{ if } a \in H_k
                     \end{cases}
  \end{align*}
  \noindent for $a\in\alphabet$ and $\rangeco{k}{0}{\lambda}$.
  Conversely, given a three-valued metric interpretation $\trivalIm$ over \alphabet,
  we define its associated timed \HT-trace over \alphabet\
  as $\M=(\tuple{\Htrace,\Ttrace}, \tmf)$ of length $\lambda$ as
\begin{align*}
  	 H_k & \eqdef \{a \in \alphabet \mid \trival{k}{a} = 2\}\\
  	T_k & \eqdef \{a \in \alphabet \mid \trival{k}{a} > 0\}
  \end{align*}
  \noindent  for every $k\in\intervco{0}{\lambda}$.

\end{definition}

\begin{proposition}[Equi-satisfaction]\label{prop:three-valued-semantics-eq}Given a timed \HT-trace $\M=(\tuple{\Htrace,\Ttrace}, \tmf)$  and a three-valued metric interpretation $\trivalIm$
  satisfying the correspondences established in Definition~\ref{def:three-valued:mht:correspondence},
  it follows that
  	\begin{eqnarray*}
      \trival{k}{\varphi}=2 & \hbox{iff}  & \M,k \models \varphi\\
      \trival{k}{\varphi}>0 & \hbox{iff}& (\tuple{\Ttrace,\Ttrace}, \tmf), k \models \varphi.
  	\end{eqnarray*}
    \noindent  for all metric formulas $\varphi$ and any $\rangeco{k}{0}{\lambda}$.
\end{proposition}
\begin{proof}
  Proven by structural induction on $\varphi$,
  taking into account that we have two induction hypotheses,
  one for each of the two cases in the proposition.
\end{proof}

\begin{corollary}[Total three-valued metric correspondence]\label{prop:three-valued:metric:total}Let $(\tuple{\Htrace,\Ttrace}, \tmf)$ a total timed \HT-trace
  and let $\trivalIm$ a three-valued metric interpretation
  satisfying the correspondences established in Definition~\ref{def:three-valued:mht:correspondence},
  it follows that

	\begin{eqnarray*}
		\trival{k}{\varphi} = 2 &\hbox{ iff }& (\tuple{\Htrace,\Ttrace}, \tmf) \models \varphi \hbox{ and }\\
		\trival{k}{\varphi} = 0 &\hbox{ iff }& (\tuple{\Htrace,\Ttrace}, \tmf) \not \models \varphi.
	\end{eqnarray*}
	\noindent for all metric formulas $\varphi$ and any $\rangeco{k}{0}{\lambda}$.
\end{corollary}
\begin{proof}
  Follows from Proposition~\ref{prop:three-valued-semantics-eq}.
\end{proof}

 \section{Splitting}

\begin{definition}[Splitting set~\cite{ben-eliyahu21a,liftur94a}] Let $P$ be a logic program and let $U$ be a set of atoms.
    $U$ is said to be a \emph{splitting set} for $P$ if for all rule $r \in P$, if one of the atoms in the head of $r$ is in U then all the atoms of $r$ are in $U$.
    We denote by $b_{U}(P)$ the set of rules in $P$ having only atoms from $U$.
\end{definition}

\begin{definition}[Reducing a logic program~\cite{ben-eliyahu21a}] Let $P$ be a logic program and let $X,Y$ be two sets of atoms.
    the procedure $Reduce(P,X,Y)$ returns a logic program obtained from $P$ in which all atoms in $X$ are set to $true$ and the atoms in $Y$ are set to $false$.
    The procedure $Reduce(P,X,Y)$ is shown in Algorithm~\ref{alg:reduce}.
\end{definition}

\begin{algorithm}[h!]
    \SetKwInOut{Input}{input}\SetKwInOut{Output}{output}
    \Input{A program $P$ and two sets of atoms: $X$ and $Y$}
    \Output{An update of $P$ assuming that all the atoms in $X$ are true and all atoms in $Y$ are false.}

    \ForEach{atom $a\in X$}{
        \ForEach{rule $r\in P$}{
            \lIf{$a$ appears negative in the body of $r$}{delete $r$}
            \lElseIf{$a$ appears in the head of $r$}{ delete $r$}
            \lElse{delete each positive appearance of $a$ in the body of $r$}
        }
    }

    \ForEach{atom $a\in Y$}{
        \ForEach{rule $r\in P$}{
            \lIf{$a$ appears positive in the body of $r$}{delete $r$}
            \lElseIf{$a$ appears in the head of $r$}{ delete a from the head of $r$}
            \lElse{delete each negative appearance of $a$ in the body of $r$}
        }
    }
    \caption{$Reduce(P,X,Y)$}\label{alg:reduce}
\end{algorithm}

\begin{theorem}[Adopted from~\cite{liftur94a,ben-eliyahu21a}]\label{thm:splitting}
    Let $P$ be a logic program, and let $U$ be a splitting set for $P$.
    A set of atoms $S$ is a stable model for $P$ if and only if
    $X = X \cup Y$ where $X$ is a stable model of $b_{U}(P)$ and $Y$ is a stable model of
    $Reduce(P\setminus b_{U},X,U\setminus X)$.
\end{theorem}
 \section{Proofs}
\label{sec:proofs}

\subsection{Plain Metric Logic Programs}

\begin{definition}[Correspondance HT to MHT Trace]
    \label{def:ht:correspondance:httomht:trace}
    Given an \HT\ interpretation $\tuple{H,T}$ over $\alphabets$
    we define $\htToMhttrace{\tuple{H,T}}$ as
    the \HT-trace
\begin{align*}
        \tuple{ \{a\in\alphabet\mid a_\kvar\in H\}, \{a\in\alphabet\mid a_\kvar\in T\}}_{\rangeco{\kvar}{0}{\lambda}}
    \end{align*}

\end{definition}
\comment{REV: We could use this to define $\htToMht{\tuple{H,T}}$}

\begin{proposition}[Next relaxed]
    \label{prop:mel:next:relaxed}
    Given a metric formula $\varphi$,
    and a timed trace $\M$,
    if $\M,k \models \metricI{\Next}\varphi $,
    then $\M,k \models \metric{\Next}{0}{\omega} \varphi$.
    \footnote{Notice the absence of an interval in \MHT\ is understood as abbreviation for the fixed interval $\intervco{0}{\omega}$ according to \cite{becadiscsc24a}}
\end{proposition}
\begin{proof}

    If $\M,k \models \metricI{\Next}\varphi $,
    then
    $\M, k{+}1 \models \varphi \text{ and } \text{ and } k+1<\lambda,$
    using MHT semantics.
    Given that with the definition of $\in \intervco{0}{\omega}$
    we have $\tmf(k{+}1)-\tmf(k) \in \intervco{0}{\omega}$,
    then
    $\M \models  \metric{\Next}{0}{\omega} \varphi$.
\end{proof}

\begin{definition}[Relaxed intervals for metric formulas]
    \label{def:mlp:ht:open-intervals-formula}
    Given a metric formulas $\varphi$,
    we define $\openi{\varphi}$ as the metric logic program obtained
    by replacing all intervals in $\varphi$ with $\intervco{0}{\omega}$.
\end{definition}

\begin{definition}[Relaxed intervals for Metric Logic Programs]
    \label{def:mlp:ht:open-intervals-program}
    Given a metric logic program $\program$,
    we define $\openi{\program}$ as the metric logic program obtained
    applying $\openi{r}$ to every rule $r$ in $\program$.
\end{definition}

\begin{lemma}[Metric formulas with relaxed intervals MHT to HT]
    \label{lem:boundedk:mht:ht}
    Let $\varphi$ be a metric formula
    whose only temporal operator is metric next of form $\Next_{\intervco{0}{\omega}}$,
    and let
    \M\ be a timed \HT-trace of length
    $\lambda\in\mathbb{N}$.
    Then,
    for any $\kinlambda$,
    if $\M,k\models\varphi$ then  $\mhtToHt{\M}\models\tk{\varphi}$.
\end{lemma}

\begin{proof}
    By structural induction on $\varphi$.
    Let $\M = (\tuple{H_\kvar,T_\kvar}_{\rangeco{\kvar}{0}{\lambda}},
        \tmf)$ and $\mhtToHt{\M}=\handt$

    \def \casephi {\top}
    \casex{\varphi=\casephi}{
        Since $\handt\models\top$ by definition,
            the result follows.
    }

    \def \casephi {\bot}
    \casex{\varphi=\casephi}{
        This leads to a contradiction,
            as $\M,\kvar\not\models\bot$.
    }

    \def \casephi {\initially}
    \casex{\varphi=\casephi}{
        If $\kvar = 0$,
            then $\handt\models\top$ by the HT semantics,
            which implies that $\handt \models \tk{\casephi}$.
        If $\kvar > 0$,
            this leads to a contradiction,
            as $\M,
            \kvar\not\models\casephi$.
    }

    \def \casephi {\finally}
    \casex{\varphi=\casephi}{
        The reasoning is analogous to the case $\varphi=\initially$.
    }

    \def \casephi {a \in \alphabet}
    \casex{\varphi=\casephi}{
        By the $\MHT$ semantics,
            $a$ must be in $H_k$.
        By definition of $\mhtToHt$,
            this means that $a_\kvar \in H$.
        By the HT semantics,
            we conclude that $\handt \models a_\kvar$,
            which implies that $\handt \models \tk{a}$.
    }

    \def \casephi { \Next_{\intervco{0}{\omega}} a }
    \casex{\varphi=\casephi}{
        Since $\M,
            k\models \casephi$,
            we analyze two cases:
        If $\kvar = \lambda-1$,
            we get a contradiction due to the $\MHT$ semantics.
        If $\kvar < \lambda-1$,
            then $\M,
            k+1\models a$ and $\tmf(k+1)-\tmf(k) \in \intervco{0}{\omega}$.
        By induction,
            this implies that $\handt \models \trans{a}{\kvar+1}$,
            which means that $\handt \models \tk{\casephi}$.
    }

    \def \casephi {\neg \psi}
    \casex{\varphi=\casephi}{
        Since $\M,
            \kvar\not\models\psi$,
            we get that $\handt\not\models\tk{\psi}$ by induction.
        By the HT semantics,
            this implies that $\handt\models\neg \tk{\psi}$,
            which means that $\handt \models \tk{\neg\psi}$.
    }

    \def \casephi {\psi_1 \wedge \psi_2}
    \casex{\varphi=\casephi}{
        Since $\M,
            \kvar$ satisfies both $\psi_1$ and $\psi_2$,
            we get by induction that $\handt$ satisfies both $\tk{\psi_1}$ and $\tk{\psi_2}$.
        By the HT semantics,
            this means that $\handt\models\tk{\psi_1} \wedge \tk{\psi_2}$,
            which implies that $\handt \models \tk{\casephi}$.
    }

    \def \casephi {\psi_1 \vee \psi_2}
    \casex{\varphi=\casephi}{
        The reasoning is analogous to the case $\varphi=\psi_1 \wedge \psi_2$.
    }

    \def \casephi {\psi \leftarrow \body}
    \casex{\varphi=\casephi}{
        If $\M,\kvar\not\models\body$,
            then by induction,
            $\handt\not\models\tk{\body}$.
        By the HT semantics,
            this implies that $\handt\models\tk{\psi} \leftarrow \tk{\body}$,
            which means that $\handt \models \tk{\casephi}$.

        If $\M,
            \kvar\models\body$,
            then by induction,
            $\handt\models\tk{\body}$.
        Since $\M,\kvar\models\psi$,
            we conclude that $(\tuple{T_\kvar,T_\kvar}_{\rangeco{\kvar}{0}{\lambda}},
            \tmf),\kvar\models\psi$.
        By induction,
            this means that both $\handt$ and $\tandt$ satisfy $\tk{\psi}$.
        By the HT semantics,
            we conclude that $\handt\models\tk{\psi} \leftarrow \tk{\body}$,
            which means that $\handt \models \tk{\casephi}$.
    }

\end{proof}

\begin{lemma}[Metric formulas with relaxed intervals HT to MHT]
    \label{lem:boundedk:ht:mht}
    Let $\varphi$ be a metric formula
    whose only temporal operator is metric next of form $\Next_{\intervco{0}{\omega}}$,
    and let
    $\handt$ a HT interpretation
    such that
    for any $\kinlambda$
    $\handt\models\tk{\varphi}$,
    then,
    $(\htToMhttrace{\handt},\tmf)\models\varphi$
    for any strict timing function $\tmf$.
\end{lemma}

\begin{proof}
    By structural induction on $\varphi$.
    Let $\M = (\htToMhttrace{\handt},\tmf)= (\tuple{H_\kvar,T_\kvar}_{\rangeco{\kvar}{0}{\lambda}}, \tmf)$.

    \def \casephi {\top}
    \casex{\varphi=\casephi}{
        Since $\M, k \models \top$ by definition, the result follows.
    }

    \def \casephi {\bot}
    \casex{\varphi=\casephi}{
        This leads to a contradiction, as $\handt \not\models \bot$.
    }

    \def \casephi {\initially}
    \casex{\varphi=\casephi}{
        If $\kvar = 0$, then $\M, \kvar \models \casephi$ by the $\MHT$ semantics.
        If $\kvar > 0$, then $\handt \models \tk{\casephi}$, but this contradicts the fact that $\handt \not\models \bot$.
    }

    \def \casephi {\finally}
    \casex{\varphi=\casephi}{
        The reasoning is analogous to the case $\varphi=\initially$.
    }

    \def \casephi {a \in \alphabet}
    \casex{\varphi=\casephi}{
        Since $\handt \models a_\kvar$ by definition of $\tk{r}$, we get that $a_\kvar \in H$ by the HT semantics.
        By the definition of $\htToMhttrace$, this implies that $a \in H_k$.
        Therefore, by the $\MHT$ semantics, we conclude that $\M, \kvar \models a$.
    }

    \def \casephi { \Next_{\intervco{0}{\omega}} a }
    \casex{\varphi=\casephi}{
        Since $\handt \models \tk{\casephi}$, we analyze two cases:
        If $\kvar = \lambda-1$, then $\handt \models \bot$, which is a contradiction.
        If $\kvar < \lambda-1$, then by the definition of $\tk{r}$, we get that $\handt \models \trans{a}{\kvar+1}$.
        By induction, this implies that $\M, k+1 \models a$.
        Additionally, since $\tmf(k+1)-\tmf(k) \in \intervco{0}{\omega}$ (which holds for any strict timing function), we conclude by the $\MHT$ semantics that $\M, k \models \casephi$.
    }

    \def \casephi {\neg \psi}
    \casex{\varphi=\casephi}{
        Since $\handt \models \tk{\neg\psi}$, we know that $\handt \models \neg \tk{\psi}$ by the definition of $\tk{r}$.
        This means that $\handt \not \models \tk{\psi}$ by the HT semantics.
        By induction, this implies that $\M, \kvar \not \models \psi$.
        Therefore, by the $\MHT$ semantics, we conclude that $\M, \kvar \models \neg\psi$.
    }

    \def \casephi {\psi_1 \wedge \psi_2}
    \casex{\varphi=\casephi}{
        Since $\handt \models \tk{\casephi}$, we get that $\handt \models \tk{\psi_1} \wedge \tk{\psi_2}$ by the definition of $\tk{r}$.
        This implies that $\handt \models \tk{\psi_1}$ and $\handt \models \tk{\psi_2}$ by the HT semantics.
        By induction, we conclude that $\M, \kvar \models \psi_1$ and $\M, \kvar \models \psi_2$.
        Therefore, by the $\MHT$ semantics, we get that $\M, \kvar \models \casephi$.
    }

    \def \casephi {\psi_1 \vee \psi_2}
    \casex{\varphi=\casephi}{
        The reasoning is analogous to the case $\varphi=\psi_1 \wedge \psi_2$.
    }

    \def \casephi {\psi \leftarrow \body}
    \casex{\varphi=\casephi}{
        Since $\handt \models \tk{\casephi}$, we get that $\handt \models \tk{\psi} \leftarrow \tk{\body}$ by the definition of $\tk{r}$.

        If $\handt \not \models \tk{\body}$, then by induction, $\M, \kvar \not \models \body$.
        By the $\MHT$ semantics, this means that $\M, \kvar \models \casephi$.

        If $\handt \models \tk{\body}$, then by induction, $\M, \kvar \models \body$.
        Since $\handt \models \tk{\psi}$ and $\tandt \models \tk{\psi}$ by the HT semantics, we conclude by induction that
        $\M, \kvar \models \psi \quad$and $(\tuple{T_\kvar,T_\kvar}_{\rangeco{\kvar}{0}{\lambda}}, \tmf),\kvar \models \psi$.
        Therefore, by the $\MHT$ semantics, we conclude that $\M, \kvar \models \casephi$.
    }

\end{proof}

\begin{theorem}[Temporal completeness with relaxed intervals]
    \label{thm:mlp:ht:completeness:open}
    Let \program\ be a metric logic program
    and
    $\M=(\tuple{\Htrace,\Ttrace}, \tmf)$ a timed \HT-trace of length $\lambda$.

    If
    $\M\models\openi{\program}$,
    then
    $\mhtToHt{\M}\models\Pi_\lambda(\program)$
\end{theorem}

\begin{proof}
    We must prove that
    $\mhtToHt{\M}\models
        \textstyle\bigcup_{\alwaysF{r} \in P,\, 0 \leq \kvar < \lambda} \tk{r}$
    Since $\M\models\openi{\program}$,
    we have  $\M\models\openi{\alwaysF r} \;\; \text{for all } \alwaysF r \in \program$.
    \MHT\ Semantics  then gives us $\M,\kvar \models \openi{r} \;\; \text{for all } \alwaysF r \in \program, 0 \leq \kvar < \lambda$,
    and by applying Lemma~\ref{lem:boundedk:mht:ht}, we get
    $\mhtToHt{\M} \models \tk{\openi{r}} \;\; \text{for all } \alwaysF r \in \program, 0 \leq \kvar < \lambda$.
    Since
    $\tk{\openi{r}}=\tk{r}$,
    because
    intervals are irrelevant in $\tk{r}$,
    we conclude that
    $\mhtToHt{\M} \models \tk{r} \;\; \text{for all } \alwaysF r \in \program, 0 \leq \kvar < \lambda$,
    which is the definition of $\Pi_\lambda(\program)$.

\end{proof}

\begin{corollary}[Completness wihout relaxed intervals]
    \label{cor:mlp:ht:completeness-temp-old-normal}
    Let \program\ be a metric logic program
    and
    $\M=(\tuple{\Htrace,\Ttrace}, \tmf)$ a timed \HT-trace of length $\lambda$.

    If
    $\M\models\program$,
    then
    $\mhtToHt{\M}\models\Pi_\lambda(\program)$
\end{corollary}

\begin{proof}
    We notice that the only difference form $\openi{\program}$ to $\program$
    is the
    $\Next_\cI$ metric operator.
    As stated in Proposition~\ref{prop:mel:next:relaxed},
    if $\M,k \models \metricI{\Next}\varphi $,
    then $\M,k \models \metric{\Next}{0}{\omega} \varphi$.
    Therefore, know that if $\M\models\program$ then $\M\models\openi{\program}$,
    so, we can apply Theorem~\ref{thm:mlp:ht:completeness:open}.
\end{proof}

\begin{theorem}[Temporal correctness with relaxed intervals]
    \label{thm:mlp:ht:correctness:open}
    Let \program\ be a metric logic program over $\alphabet$
    and
    let $\handt$ be an HT interpretation over $\alphabets$.

    If
    $\handt\models\Pi_\lambda(\program)$,
    then
    $(\htToMhttrace{\handt},\tmf)\models\openi{\program}$
    for any strict timing function $\tmf$.
\end{theorem}

\begin{proof}
    We take a strict timing function $\tmf$ and
    let $\M=(\htToMhttrace{\handt},\tmf)$.
    We must prove that
    $\M\models\openi{\program}$.
We know $\handt \models \Pi_\lambda(\program)$,
    so
    $\handt \models \tk{r} \;\; \text{for all } \alwaysF r \in \program, 0 \leq \kvar < \lambda$.
    Since intervals are irrelevant in $\tk{r}$, then
    $\handt \models \tk{\openi{r}} \;\; \text{for all } \alwaysF r \in \program, 0 \leq \kvar < \lambda$.
    By applying Lemma~\ref{lem:boundedk:ht:mht}, we get
    $\M,k \models \openi{r} \;\; \text{for all } \alwaysF r \in \program, 0 \leq \kvar < \lambda$,
    which using \MHT\ Semantics, gives us
    $\M \models \openi{\program}$
\end{proof}
  \subsubsection{Translation of Plain Metric Logic Programs to \HT}
\begin{lemma}[Splitting Plain HT]
    \label{lem:splitting:plain:ht}
    Given a metric logic program $\program$
    and a $\lambda,\tmflimit\in\mathbb{N}$,
    the stable models of $\Pi_\lambda(\program)\cup\Delta_{\lambda,\tmflimit}\cup\Psi_{\lambda,\tmflimit}(\program)$ can be computed by means of the splitting technique.
\end{lemma}
\begin{proof}
    We will show how the stable models of $\mathbb{P} = \Pi_\lambda(\program)\cup\Delta_{\lambda,\tmflimit}\cup\Psi_{\lambda,\tmflimit}(\program)$ can be computed
    by sequentially applying the splitting theorem. Let us start by fixing $U_0 = \lbrace t_{k,d} \mid k, d \in \mathbb{N}\rbrace$.
    Clearly it is an splitting set because the elements of $U_0$ are only generated by $\Delta_{\lambda,\tmflimit}$. Moreover, $b_{U_0} = \Delta_{\lambda,\tmflimit}$.
    By Theorem~\ref{thm:splitting}, an stable model $T$ of $\mathbb{P}$ can be computed as,
    $T=T_0\cup T_1$ where $T_0$ is a stable model of $b_{U_0}$.
$T_1$ is a stable model of $Reduce(\mathbb{P}\setminus b_{U_0},T_0,U_0\setminus T_0)$.
    If we analyze this first simplification,
    we will see that
$Reduce(\mathbb{P}\setminus b_{U_0},T_0,U_0\setminus T_0) =  \Pi_\lambda(\program)\cup Reduce(\Psi_{\lambda,\tmflimit}(\program), T_0,U_0\setminus T_0)$,
    because $\Pi_\lambda(\program)$ does not contain atoms in $U_0$.
    More precisely,
    $Reduce(\Psi_{\lambda,\tmflimit}(\program), T_0,U_0\setminus T_0)$
    are just constraints that use only atoms from $\Pi_\lambda(\program)$.
\end{proof}

\begin{proposition}[Model if timed equilibrium]
    \label{prop:ht:timed:delta:equilibrium}
    Let $\handt$ be an \HT\ interpretation
    and
    $\lambda\in\mathbb{N}$.

    If $\handt$ is timed wrt $\lambda$ and induces $\tmf$
    then $\handt|_{\alphabetT}$ is a total \HT\ interpretation
    and an equilibrium model of
    $\Delta_{\lambda,\tmflimit}$
    for $\tmflimit=\tmf(\lambda-1)$.
\end{proposition}

\begin{proof}
    The satisfaction follows from Proposition~\ref{prop:ht:timed:delta}.
    From the definition of being timed it follows that $\handt|_{\alphabetT}$ must be total.
    Now we prove that it is in equilibrium.
    Assume towards a contradiction that
    there is $\handt$ such that $H\subset T$ and $\handt\models\Delta_{\lambda,\tmflimit}$.
    Then, there is $\timet_{k,d}\in\alphabetT$ such that $\timet_{k,d}\in T$ and $\timet_{k,d}\not\in H$
    for some $0\leq k <\lambda$ and $0\leq d \leq \tmflimit$.
Without loss of generality let this $k$ be such that for all $0<k'<k$,
    we have $\timet_{k',d}\in T$ iff $\timet_{k',d}\in H$.
    Now lets analyze the two cases for $k$.
    \casex{k=0}{
        Then  $d=0$ since $\tmf(0)=0$ becase $\tandt$ is timmed inducing $\tmf$.
        Then $\timet_{0,0}\not \in H$ which is a
        contradiction since $\handt\models\timet_{0,0}$.
    }
    \casex{k>0}{
        Then $\tmf(\kvar-1)=d'$ with $d'<d$,
        where
        $\timet_{\kvar-1,d'}\in T$ and
        $\timet_{\kvar-1,d'}\in H$ since we selected $k$ as the first to not be in the set,
        therefore,
        $\handt\models\timet_{\kvar-1,d'}$.
        Since $\handt\models\Delta_{\lambda,\tmflimit}$
        and the head of the rule holds, then
        $\handt\models\bigvee_{d'<d''\leq\tmflimit}\timet_{\kvar,\tmvar''}$.
        This means
        $\handt\models\timet_{\kvar,\tmvar''}$ for some $d'<d''\leq\tmflimit$,
        so $\timet_{\kvar,d''}\in H$,
        and given that it is timed we have
        $\tmf(\kvar)=\tmvar''$.
        Since $\tmf(k)=d$  we have that $d''=d$,
        which is a contradiction because $\timet_{k,d}\not \in H$.
    }
\end{proof}

\begin{lemma}
    \label{lem:completeness-ht}
    Let \program\ be a metric logic program
    and
    $\M=(\tuple{\Htrace,\Ttrace}, \tmf)$ a timed \HT-trace of length $\lambda$.

    If
    $\M\models\program$,
    then
    $\mhtToHt{\M}\models\Pi_\lambda(\program)\cup\Delta_{\lambda,\tmflimit}\cup\Psi_{\lambda,\tmflimit}(\program)$
    with $\tmflimit=\tmf(\lambda-1)$.
\end{lemma}

  \begin{proof}
    We prove that $\mhtToHt{\M}$ models each part of the translation.

    We know
    $\M\models\program$
    so due to Corollary~\ref{cor:mlp:ht:completeness-temp-old-normal},
    we have that $\M\models\openi{\program}$.
    With the Definition for $\openi{\program}$, get
    $\M\models\openi{\alwaysF r} \;\; \text{for all } \alwaysF r \in \program$,
    which means that
    $\M,k\models\alwaysF \openi{r} \;\; \text{for all } \alwaysF r \in \program,0 \leq \kvar < \lambda$.
    With Lemma~\ref{lem:boundedk:mht:ht}, we have that
    $\mhtToHt{\M} \models \tk{\openi{r}} \;\; \text{for all } \alwaysF r \in \program, 0 \leq \kvar < \lambda$.
    Since intervals are irrelevant in $\tk{r}$, we then know
    $\tk{\openi{r}}=\tk{r}$ and therefore
    $\mhtToHt{\M} \models \tk{r} \;\; \text{for all } \alwaysF r \in \program, 0 \leq \kvar < \lambda$,
    which by definition is $\Pi_\lambda(\program)$.

    $\mhtToHt{\M}$ is timed inducing $\tmf$ by construction,
    so it is enough to apply Proposition~\ref{prop:ht:timed:delta} to show that
    $\mhtToHt{\M}\models\Delta_{\lambda,\tmflimit}$.

    Assume towards a contradiction that
    $\mhtToHt{\M}\not\models \Psi_{\lambda,\tmflimit}(\program)$.
    Then, there is a rule $r = \bot \leftarrow \tk{\body} \wedge \timet_{\kvar,\tmvar} \wedge \timet_{\kvar+1,\tmvar'}$
    for some $0 \leq \kvar < \lambda-1$ and $0\leq d<d'\leq\tmflimit$,
    such that
    $\mhtToHt{\M}\models \tk{\body} \wedge \timet_{\kvar,\tmvar} \wedge \timet_{\kvar+1,\tmvar'}$,
    where $\melrule{\metric{\Next}{m}{n} a \leftarrow \body}\in \program$.
    Given that $\mhtToHt{\M}\models \tk{\body}$, with Lemma~\ref{lem:boundedk:ht:mht}
    we have $\M,k\models \body$.
    Which means that $\M,k\models \metric{\Next}{m}{n} a$ since $\M\models\program$.
    With MHT semantics, we have $\M,k+1\models a$ and $\tmf(k+1)-\tmf(k)\in \intervco{m}{n}$.
    Given that $\M$ is timed inducing $\tmf$, we have $\tmf(k)=d$ and $\tmf(k+1)=d'$.
    Then, we reach a contradiction since the conditions for $r$ are either
    $\tmvar - \tmvar' > -m$ or $n\in\mathbb{N} \;\text{and}\; \tmvar' - \tmvar > n$,
    for \eqref{def:ht:psi:one} and \eqref{def:ht:psi:two} respectively,
    which are both contradicted by the fact that $\tmf(k+1)-\tmf(k)\in \intervco{m}{n}$.

\end{proof}

\begin{lemma}
    \label{lem:correctness-ht}
    Let
    \program\ be a metric logic program,
    $\lambda,\tmflimit \in \mathbb{N}$,
    and
    \handt\ an \HT\ interpretation
    timed wrt. $\lambda$.

    If
    $\handt\models\Pi_\lambda(\program)\cup\Psi_{\lambda,\tmflimit}(\program)$,
    then
    $\htToMht{\handt}\models\program$.
\end{lemma}

  \begin{proof}
    Let $\tmf$ be the timing function induced by $\handt$.

    Since $\handt\models\Pi_\lambda(\program)$, then
    $\handt\models \tk{r}$ for all  $\alwaysF r \in \program, 0 \leq \kvar < \lambda$.
    We know that $\tk{r}=\tk{\openi{r}}$
    since intervals are not relevant for this definition.
    So $\handt\models \tk{\openi{r}}$ for all  $\alwaysF r \in \program, 0 \leq \kvar < \lambda$.
    With this we can apply Lemma~\ref{lem:boundedk:ht:mht} with $\M=\htToMht{\handt}$,
    to get
    $\htToMht{\handt}, k\models \openi{r}$ for all  $\alwaysF r \in \program, 0 \leq \kvar < \lambda$.
    We now analyze the two possible forms for $r$ and show that $\htToMht{\handt}, k\models {r}$,
    which would mean that $\htToMht{\handt}\models\program$.
    \casex{r = a_1\vee...\vee a_m \leftarrow \body}{
Since there are no metric operators in $r$, we have
        $\openi{r}=r$ so
        we have
        $\htToMht{\handt}, k\models {r}$
    }
    \casex{r = \metricI{\Next} a \leftarrow \body}{
        If $\htToMht{\handt}, k \not \models \body$ then
        $\htToMht{\handt}, k\models {r}$ is trivially satisfied.
        Otherwise,
        we must show that $\htToMht{\handt}, k\models \metricI{\Next} a$.
        We know that $\htToMht{\handt}, k\models \Next_{\intervco{0}{\omega}} a$
        because $\handt\models \tk{\openi{r}}$ and $\htToMht{\handt}, k \models \body$,
        so it is enough to show that $\tmf(k+1) -\tmf(k) \in \cI$
        to have the MHT semantics for $\metricI{\Next} a$.
Assume towards a contradiction that $\tmf(k+1) -\tmf(k) \not \in \cI$,
        with $\cI=\intervco{m}{n}$.
Since there are no metric operators in \body, we have
        $\handt \models \tk{\body}$ because of Lemma~\ref{lem:boundedk:mht:ht}.
        We also know that $\handt \models \timet_{k,\tmf(k)}$ and $\handt \models \timet_{k+1,\tmf(k+1)}$,
        since \handt\ induces $\tmf$,
        so $\handt \models \tk{\body} \wedge \timet_{k,\tmf(k)} \wedge \timet_{k+1,\tmf(k+1)}$.
If the reason for $\tmf(k+1) -\tmf(k) \not \in \cI$ is that $\tmf(k+1) - \tmf(k) < m$
        then $\handt \models \bot$ because it would meet the conditions of rule \eqref{def:ht:psi:one}.
        If the reason is that $n\in\mathbb{N}$ and $\tmf(k+1) - \tmf(k) \geq n$
        then $\handt \models \bot$ because it would meet the conditions of rule \eqref{def:ht:psi:two}.
        In both cases reaching a contradiction,
        therefore $\tmf(k+1) -\tmf(k) \in \cI$.
    }

\end{proof}
  \setpropcounter{prop:ht:delta:timed}
\begin{proposition}
    Let
    $\lambda,\tmflimit\in\mathbb{N}$.

    If $\tandt$ is an equilibrium model of $\Delta_{\lambda,\tmflimit}$
    then $\tandt$ is timed wrt $\lambda$.
\end{proposition}
\resetpropcounter
\begin{proof}
    \newcommand{\Dom}{\ensuremath{N}}
    \newcommand{\Cdom}{\ensuremath{D}}

    Let us construct a relation $\relR=\{(k,d)\mid\tandt \models\timet_{\kvar,\tmvar}\}$.

    We proof that \relR\ represents a function $\tmf: \Dom\to \Cdom$ where $0,\dots,\lambda-1\in \Dom$ and  $0,\dots,\tmflimit\in \Cdom$.
    Let us proof by induction that for all $k\in\Dom$ there is exactly one $d$ such that $(k,d)\in\relR$.

    \casex{k=0}{
        Since $\tandt \models \timet_{0,0}$ due to Definition~\ref{def:ht:delta},
        we know $(0,0) \in \relR$.
        We show that there is no other $d$
        where $(0,d)\in\relR$ with $d\neq 0$.
        Assume towards a contradiction that there is one,
        then $\tandt \models \timet_{0,d}$,
        and $\timet_{0,d}\in T$ due to \HT\ semantics.
        We argue that $\tuple{T\setminus{\{\timet_{0,d}\}},T}\models\Delta_{\lambda,\tmflimit}$.
        This is the case since
        $\timet_{0,d}$ can't appear in the head of rule $\bigvee_{d<d'\leq\tmflimit}\timet_{\kvar+1,\tmvar'}\leftarrow\timet_{k,d}$.
        So we can remove $\timet_{0,d}$ from $T$ and still satisfy $\Delta_{\lambda,\tmflimit}$.
        Therefore, $\tandt$ is not in equilibrium
        which is a contradiction.

    }
    \casex{k+1 \in \intervco{1}{\lambda}}{
By induction hypothesis, there is exactly one $d$ such that $(k,d) \in \relR$,
        so by construction
        $\tandt \models \timet_{k,d}$.
        With rule \eqref{def:ht:delta} and \HT\ semantics, we get
        $\tandt \models \bigvee_{d<d'\leq\tmflimit}\timet_{\kvar+1,\tmvar'}$.
        This means that
        $\tandt \models \timet_{\kvar+1,\tmvar'}$ for some $d'$,
        so by construction
        $(\kvar+1,\tmvar')\in\relR$.
Now we show that
        there is no $d''\neq d'$ such that $(\kvar+1,\tmvar'')\in\relR$.
        Assume towards a contradiction
        there is $(\kvar+1,\tmvar'')\in\relR$ with  $d''\neq d'$.
        Then  $\tandt \models \timet_{\kvar+1,\tmvar''}$ and
        $\timet_{\kvar+1,\tmvar''}\in T$.
        We argue that $\tuple{T',T}\models\Delta_{\lambda,\tmflimit}$ with $T' = T\setminus\{\timet_{\kvar+1,\tmvar''}\}$.
        Notice
        $\timet_{\kvar+1,\tmvar''}$ only appears as a head in rule $\bigvee_{d<d'\leq\tmflimit}\timet_{\kvar+1,\tmvar'}\leftarrow\timet_{k,d}$.
        Since $\timet_{k,d}\in T$, after removing $\timet_{\kvar+1,\tmvar''}$ ,
        $\bigvee_{d<d'\leq\tmflimit}\timet_{\kvar+1,\tmvar'}$ is still satisfied,
        so $\tandt$ is not in equilibrium, which is a contradiction.
    }

    Now we prove that $\tmf$ is a (strict) timing function wrt $\lambda$.
    First, $\tmf(0)=0$ as shown before.
    We assume towards a contradiction that
    $\tmf(k)\geq\tmf(k+1)$  for some $k$.
    By construction $(k,d),(k+1,d')\in\relR$ with $d\geq d'$,
    and $\tandt\models\timet_{k+1,d'}$ and $\tandt\models\timet_{k,d}$.
    Then, there must be a supporting rule for $\timet_{k+1,d'}$.
    This must be rule $\bigvee_{d<d'\leq\tmflimit}\timet_{\kvar+1,\tmvar'}\leftarrow\timet_{k,d}$,
    since it is the only rule that can introduce $\timet_{k+1,d'}$.
    Then $d<d'$ to appear in the head, which contradicts the assumption.
\end{proof}

\setpropcounter{prop:ht:timed:delta}
\begin{proposition}
    Let $\handt$ be an \HT\ interpretation
    and
    $\lambda\in\mathbb{N}$.

    If $\handt$ is timed wrt $\lambda$ and induces $\tmf$,
    then $\handt$ is an \HT-model of $\Delta_{\lambda,\tmflimit}$
    for $\tmflimit=\tmf(\lambda-1)$.
\end{proposition}
\resetpropcounter
\begin{proof}
    We proof that $\handt$ satisfies every rule in $\Delta_{\lambda,\tmflimit}$
    \casex{\timet_{0,0}}{
      $\tmf$ is a timing function so $\tmf(0)=0$.
      Since $\handt$ induces $\tmf$ we have that
      $\handt\models \timet_{0,0}$
    }
    \casex{\bigvee_{d<d'\leq\tmflimit}\timet_{\kvar+1,\tmvar'}\leftarrow\timet_{\kvar,\tmvar} }{
      Let $k \in \intervco{0}{\lambda-1}$ and $d \in \intervco{0}{\tmflimit}$.
      We analyze the two cases for $k$.
      \casex{\tmf(k)\neq d}{
        The definiton of timed gives us
        $\handt\not\models \timet_{k,d}$ and $\tandt\not\models \timet_{k,d}$,
        and with HT semantics we get
        $\handt\models\bigvee_{d<d'\leq\tmflimit}\timet_{\kvar+1,\tmvar'}\leftarrow\timet_{\kvar,\tmvar}$,
        since the body is false.
      }
      \casex{\tmf(k)= d}{
        $\tmf(k+1)=d'$ for some $d'>d$ because $\tmf$ is a strict timing function and $k+1$ is in the domain.
        $\handt$ induces $\tmf$ so $\handt\models \timet_{k+1,d'}$ and $\tandt\models \timet_{k+1,d'}$.
        Due to \HT\ semantics and the fact that $d'\leq \tmflimit$ since $\tmflimit=\tmf(\lambda-1)$,
        we get
        $\handt\models \bigvee_{d<d'\leq\tmflimit}\timet_{\kvar+1,\tmvar'}$ and
        $\tandt\models \bigvee_{d<d'\leq\tmflimit}\timet_{\kvar+1,\tmvar'}$
        which gives us the consequent of the implication.
      }
    }

\end{proof}

\setcounter{theorem}{\getrefnumber{thm:mlp:ht:completeness}}
\addtocounter{theorem}{-1}
\begin{theorem}[Completeness]
    Let \program\ be a plain metric logic program
    and
    $\M=(\tuple{\Ttrace,\Ttrace}, \tmf)$ a total timed \HT-trace of length $\lambda$.

    If
    $\M$ is a metric equilibrium model of $\program$,
    then
    $\mhtToHt{\M}$ is an equilibrium model of $\Pi_\lambda(\program)\cup\Delta_{\lambda,\tmflimit}\cup\Psi_{\lambda,\tmflimit}(\program)$
    with $\tmflimit=\tmf(\lambda-1)$.
\end{theorem}
\begin{proof}
    We first prove that it models the translation
    $\Pi_\lambda(\program)\cup\Delta_{\lambda,\tmflimit}\cup\Psi_{\lambda,\tmflimit}(\program)$
    using Lemma~\ref{lem:completeness-ht}.
Now we prove that it is in equilibrium.
    By construction $\mhtToHt{\M}=\tandt$ is total since $\M$ is total.
    Assume towards a contradiction that
    there is $\handt$ such that $H\subset T$ and $\handt\models\Pi_\lambda(\program)\cup\Delta_{\lambda,\tmflimit}\cup\Psi_{\lambda,\tmflimit}(\program)$.
    Then there is $x\in\alphabets\cup\alphabetT$ such that $x\in T$ and $x\not\in H$.
    Let's analyze each case.
    \casex{x\in\alphabetT}{
        We know that $\mhtToHt{\M}$ is timed by construction,
        so with Proposition~\ref{prop:ht:timed:delta:equilibrium}
        we get
        $\mhtToHt{\M}|_{\alphabetT}$ is an equilibrium model of $\Delta_{\lambda,\tmflimit}$,
        then
        $\handt|_{\alphabetT}\not\models\Delta_{\lambda,\tmflimit}$,
        which contradicts the hypothesis.
    }
    \casex{x\in\alphabets}{
        Then $x=a_k$ for some $0\leq k <\lambda$
        so, there is $a_k\in\alphabets$ such that $a_k\in T$ and $a_k\not\in H$.
        Notice that then $\handt$ is timed wrt $\lambda$ inducing $\tmf$
        and that $\handt\models\Pi_\lambda(\program)\cup\Psi_{\lambda,\tmflimit}(\program)$,
        which means we can apply Lemma~\ref{lem:completeness-ht} to get
        $\htToMht{\handt}\models\program$.
        This contradicts that $\M$ is an equilibrium model,
        because $\htToMht{\handt}$ has the same timing function and it is smaller due to $x$.
    }

\end{proof}

\setcounter{theorem}{\getrefnumber{thm:mlp:ht:correctness}}
\addtocounter{theorem}{-1}
\begin{theorem}[Correctness]
    Let
    \program\ be a plain metric logic program,
    and
    $\lambda,\tmflimit \in \mathbb{N}$.

    If
    $\tandt$ is an equilibrium model of $\Pi_\lambda(\program)\cup\Delta_{\lambda,\tmflimit}\cup\Psi_{\lambda,\tmflimit}(\program)$,
    then
    $\htToMht{\tandt}$ is a metric equilibrium model of $\program$.
\end{theorem}
\begin{proof}
    We argue that $\tandt$ is an equilibrium model of $\Delta_{\lambda,\tmflimit}$
    since this program is over alphabet $\alphabetT$ and $\Pi_\lambda(\program)$ is over alphabet $\alphabets$,
    where $\alphabetT\cap\alphabets=\emptyset$.
    Additionally $\Psi_{\lambda,\tmflimit}(\program)$ consist of only integrity constraints.
    We can apply the splitting property (Lemma~\ref{lem:splitting:plain:ht}) to say that
    $\tandt|_\alphabetT$ is an equilibrium model of $\Delta_{\lambda,\tmflimit}$,
    $\tandt|_\alphabetT$ is timed wrt. $\lambda$ inducing timing function $\tmf$
    due to Proposition~\ref{prop:ht:delta:timed}.
    Notice that $\tandt$ is also timed.
    Knowing this, we have the necessary conditions to create $\htToMht{\tandt}$.
    Then, we can apply Lemma~\ref{lem:correctness-ht} to get that
    $\htToMht{\tandt}\models\program$.

    We now prove that $\htToMht{\tandt}$ is an equilibrium model.
    Let $\htToMht{\tandt} = (\tuple{T_\kvar,T_\kvar}_{\rangeco{\kvar}{0}{\lambda}}, \tmf)$.
    First notice that $\htToMht{\tandt}$ is total by construction.
Let's assume towards a contradiction that $\htToMht{\tandt}$ is not minimal.
    Then, there is some metric trace $\M'=(\tuple{H_\kvar,T_\kvar}_{\rangeco{\kvar}{0}{\lambda}}, \tmf)$
    such that $\M'\models\program$ and $H_\kvar\subset T_\kvar$ for some $\rangeco{\kvar}{0}{\lambda}$.
    So, there is $a\in\alphabet$ such that $a\in T_\kvar$ and $a\not\in H_\kvar$.
    By construction $a_k\in T$ and $a_k\not\in H$ where $\mhtToHt{\M'}=\tuple{H,T}$.
    We apply Lemma~\ref{lem:completeness-ht} and get
    $\tuple{H,T}\models\Pi_\lambda(\program)\cup\Delta_{\lambda,\tmflimit}\cup\Psi_{\lambda,\tmflimit}(\program)$
    so $\tandt$ is not an equilibrium model, which is a contradiction.

\end{proof}
 \subsubsection{Translation of Plain Metric Logic Programs to \HTC}

\begin{lemma}[Splitting Plain \HTC]
  \label{lem:splitting:plain:htc}
  Given a metric logic program $\program$
  and a $\lambda\in\mathbb{N}$,
  the stable models of $\Pi_\lambda(\program)\cup\Delta_{\lambda}^c\cup\Psi^c_{\lambda}(\program)$
  can be computed by means of the splitting technique.
\end{lemma}
\begin{proof}
  Proven as Lemma~\ref{lem:splitting:plain:ht} but using the notion of splitting set as defined in
  \citep[Definition 10]{cafascwa20b}.
\end{proof}

\begin{proposition}[Model if timed equilibrium]
  \label{prop:htc:timed:delta:equilibrium}
  Let $\htcinterp$ be an \HTC\ interpretation
  and
  $\lambda\in\mathbb{N}$.

  If $\htcinterp$ is timed wrt. $\lambda$ and induces $\tmf$
  then $\htcinterp|_{\alphabetT^c}$ is a total \HTC\ interpretation
  and an constraint equilibrium model of
  $\Delta^c_{\lambda}$.
\end{proposition}
\begin{proof}
  The satisfaction follows from Proposition~\ref{prop:htc:timed:delta}.
  From the definition of being timed
  it follows that $\htcinterp|_{\alphabetT^c}=\tuple{\Vt',\Vt'}$ must be total.
  Now we prove that it is in equilibrium.
  Assume towards a contradiction that
  there is $\tuple{\Vh',\Vt'}$ such that $\Vh'\subset \Vt'$ and $\tuple{\Vh',\Vt'}\models\Delta^c_{\lambda}$
  Then, there is $(\timet_{k},d)\in\Vt'$ such that $(\timet_{k},d)\not\in\Vh'$
  for some $0\leq k <\lambda$ and $d \in \mathbb{N}$,
  so $\Vh'(\timet_{k})=\undefined$.
  We analyze the two cases for $k$.
  \casex{k=0}{
  Then
  $\Vh'\not\in\den{\timet_{0}=0}$
  since $\Vh'(\timet_{0})=\undefined$.
  Contradiction since $\tuple{\Vh',\Vt'}\models\timet_{0}=0$
  because $\tuple{\Vh',\Vt'}\models\Delta^c_{\lambda}$.
  }
  \casex{k>0}{
    Since $\Vh'(\timet_{k})=\undefined$ and $x\in\mathbb{N}$ is a condition for $\den{\x-\y\leq \tmvar}$,
    we have
    $\Vh'\not\in\den{ \timet_{\kvar}-\timet_{\kvar{+}1} \leq -1}$.
    Contradiction since $\tuple{\Vh',\Vt'}\models\timet_{\kvar}-\timet_{\kvar{+}1} \leq -1$
    because $\tuple{\Vh',\Vt'}\models\Delta^c_{\lambda}$.
  }
\end{proof}

\begin{lemma}[Correctness \HTC]
  \label{lem:correctness-htc}
  Let
  \program\ be a metric logic program,
  \htcinterp\ an \HTC\ interpretation
  timed wrt. $\lambda$.

  If
  $\htcinterp\models\Pi_\lambda(\program)\cup\Psi^c_{\lambda}(\program)$,
  then
  $\htcToMht{\htcinterp}\models\program$.
\end{lemma}

\begin{proof}

    Let $\tmf$ be the timing function induced by $\htcinterp$.
    We will prove that $\htcToMht{\htcinterp}\models\program$
    by showing that it models each rule $r\in\program$.
    For this, first notice that
    $\htcinterp\models\Pi_\lambda(\program)$,
    so
    $\htcinterp\models \tk{r}$ for all  $\alwaysF r \in \program, 0 \leq \kvar < \lambda$.
    Since $\tk{r}=\tk{\openi{r}}$, then
    $\htcinterp\models \tk{\openi{r}}$,
    therefore we can apply Lemma~\ref{lem:boundedk:ht:mht}
    to get $\htcToMht{\htcinterp}, k\models \openi{r}$.
    Notice that Lemma~\ref{lem:boundedk:mht:ht} is defined for HT,
    but since it uses only boolean variables it holds the for \HTC\ (Observation 1 in \cite{cakaossc16a}).
Now we analyze the two cases for $r$
    and show that $\htcToMht{\htcinterp}, k\models {r}$,
    which would mean that $\htcToMht{\htcinterp}\models\program$.
    \casex{r = a_1\vee...\vee a_m \leftarrow \body}{
        Since there are no metric operators in $r$,
        then $\openi{r}=r$.
        Therefore, $\htcToMht{\htcinterp}, k\models {r}$.
    }
    \casex{r = \metricI{\Next} a \leftarrow \body}{
        If $\htcToMht{\htcinterp}, k \not \models \body$ then
        $\htcToMht{\htcinterp}, k\models {r}$ is trivially satisfied.
        Otherwise,
        we must show that $\htcToMht{\htcinterp}, k\models \metricI{\Next} a$.
        We know that $\htcToMht{\htcinterp}, k\models \Next_{\intervco{0}{\omega}} a$
        because $\htcinterp\models \tk{\openi{r}}$ and $\htcToMht{\htcinterp}, k \models \body$,
        so it is enough to show that $\tmf(k+1) -\tmf(k) \in \cI$
        to have the MHT semantics for $\metricI{\Next} a$.

        Assume towards a contradiction that $\tmf(k+1) -\tmf(k) \not \in \cI$,
        with $\cI=\intervco{m}{n}$.
Since there are no metric operators in \body, we have
        $\htcinterp \models \tk{\body}$ because of Lemma~\ref{lem:boundedk:mht:ht}.
If the reason for $\tmf(k+1) -\tmf(k) \not \in \cI$ is that $\tmf(k+1) - \tmf(k) < m$
        then
        $\timet_{k+1} - \timet_{k} < m$,
        since it is timed inducing $\tmf$.
        This is equivalent to
        $\timet_{k} - \timet_{k+1} > - m$
        so $\Vh \not \in \den{\timet_{k} - \timet_{k+1} \leq - m}$
        and $\Vh \models \neg (\timet_{k} - \timet_{k+1} \leq - m)$.
        Therefore,
        $\handt \models \bot$ because of rule \eqref{def:htc:psi:one},
        which is a contradiction.
If the reason is that $n\in\mathbb{N}$ and $\tmf(k+1) - \tmf(k) \geq n$,
        then we reach a contradiction
        by reasoning analogously and using rule \eqref{def:htc:psi:two}.
Therefore, $\tmf(k+1) -\tmf(k) \in \cI$.
    }

\end{proof}

\begin{lemma}[Completeness \HTC]
  \label{lem:completeness-htc}
  Let \program\ be a metric logic program
  and
  $\M=(\tuple{\Htrace,\Ttrace}, \tmf)$ a timed \HT-trace of length $\lambda$.

  If
  $\M\models\program$,
  then
  $\mhtToHtc{\M}\models\Pi_\lambda(\program)\cup\Delta^c_{\lambda}\cup\Psi^c_{\lambda}(\program)$.
\end{lemma}

\begin{proof}
    We prove that $\mhtToHtc{\M}$ models each part of the translation.

    We know
    $\M\models\program$
    so due to Corollary~\ref{cor:mlp:ht:completeness-temp-old-normal}
    \footnote{
        Notice that Corollary~\ref{cor:mlp:ht:completeness-temp-old-normal} is defined for HT,
        but since it uses only boolean variables it holds the for \HTC\ (Observation 1 in \cite{cakaossc16a})
    },
    we have that $\M\models\openi{\program}$.
    With the Definition for $\openi{\program}$, get
    $\M\models\openi{\alwaysF r} \;\; \text{for all } \alwaysF r \in \program$,
    which means that
    $\M,k\models\alwaysF \openi{r} \;\; \text{for all } \alwaysF r \in \program,0 \leq \kvar < \lambda$.
    With Lemma~\ref{lem:boundedk:mht:ht}
    \footnote{Notice that Lemma~\ref{lem:boundedk:mht:ht} is defined for HT,
        but since it uses only boolean variables it holds the for \HTC\ (Observation 1 in \cite{cakaossc16a}).},
    we have that
    $\mhtToHtc{\M} \models \tk{\openi{r}} \;\; \text{for all } \alwaysF r \in \program, 0 \leq \kvar < \lambda$.
    Since intervals are irrelevant in $\tk{r}$, we then know
    $\tk{\openi{r}}=\tk{r}$ and therefore
    $\mhtToHtc{\M} \models \tk{r} \;\; \text{for all } \alwaysF r \in \program, 0 \leq \kvar < \lambda$,
    which by definition is $\Pi_\lambda(\program)$.

    $\mhtToHtc{\M}$ is timed inducing $\tmf$ by construction,
    so it is enough to apply Proposition~\ref{prop:htc:timed:delta} to show that
    $\mhtToHtc{\M}\models\Delta^c_{\lambda}$.

    Assume towards a contradiction that
    $\mhtToHtc{\M}\not\models \Psi^c_{\lambda}(\program)$.
    Then, there is a rule
    $r\in\Psi^c_{\lambda}(\program)$ such that $\mhtToHtc{\M}\not\models r$.
    Without loss of generality, let $r$ come from equation \eqref{def:htc:psi:one},
    where the case for \eqref{def:htc:psi:two} is analogous.
    Then, $r = \bot \leftarrow \tk{\body} \wedge \neg (\timet_{\kvar}-\timet_{\kvar+1} \leq -m)$
    for some $0 \leq \kvar < \lambda-1$
    and
    $\mhtToHtc{\M}\models \tk{\body} \wedge \neg (\timet_{\kvar}-\timet_{\kvar+1} \leq -m)$,
    where $\melrule{\metric{\Next}{m}{n} a \leftarrow \body}\in \program$.
    Given that $\mhtToHtc{\M}\models \tk{\body}$, with Lemma~\ref{lem:boundedk:ht:mht}
    we have $\M,k\models \body$.
    Which means that $\M,k\models \metric{\Next}{m}{n} a$ since $\M\models\program$.
    With MHT semantics, we have $\M,k+1\models a$ and $\tmf(k+1)-\tmf(k)\in \intervco{m}{n}$.
Given that $\M$ is timed inducing $\tmf$, we have $\tmf(k)=\timet_k$ and $\tmf(k+1)=\timet_k$,
    so $\timet_{k+1}-\timet_{k}\geq m$,
    which is equivalent to $\timet_{k}-\timet_{k+1} \leq -m$.
    Then $\Vh\in\den{\timet_{\kvar}-\timet_{\kvar+1} \leq -m}$
    and
    $\mhtToHtc{\M}\models (\timet_{\kvar}-\timet_{\kvar+1} \leq -m)$,
    which is a contradiction,
    since $\mhtToHtc{\M}\not\models (\timet_{\kvar}-\timet_{\kvar+1} \leq -m)$.

\end{proof}
  \setpropcounter{pro:htc:defined}
\begin{proposition}
    Let $\htcinterp$ be an \HTC\ interpretation and $\lambda \in \mathbb{N}$

    If $\htcinterp$ is an \HTC-model of $\Delta^{c}_{\lambda}$, then $\Vh(\timet_{\kvar})\in\mathbb{N}$ for all $0\leq\kvar<\lambda$.
\end{proposition}
\resetpropcounter

\begin{proof} Case analysis on $k$

    \casex{k=0}{
        Since $\timet_0=0 \in \Delta^{c}_{\lambda}$
        then $\Vh\in\den{\timet_0=0}$.
        Therefore, $\Vh(\timet_0)\in\mathbb{N}$,
        since it is a condition of the denotation.
    }

    \casex{0<k<\lambda}{
        Since $\timet_{\kvar-1}-\timet_{\kvar} \leq -1 \in \Delta^{c}_{\lambda}$,
        then $\Vh\in\den{\timet_{\kvar-1}-\timet_{\kvar} \leq -1}$.
        Therefore, $\Vh(\timet_{\kvar})\in\mathbb{N}$
        since it is a condition of the denotation.
    }

\end{proof}

\setpropcounter{prop:htc:delta:timed}
\begin{proposition}
    Let $\htcinterp$ be an \HTC\ interpretation
    and
    $\lambda\in\mathbb{N}$.

    If $\htcinterp$ is an \HTC-model of $\Delta^{c}_{\lambda}$
    then $\htcinterp$ is timed wrt $\lambda$.
\end{proposition}
\resetpropcounter

\begin{proof}
    Let us construct the induced timing function $\tmf$ such that $\tmf(k)=\Vh(\timet_k)$  for $0\leq k\leq\lambda$.
    We prove that $\tmf$ is a timing function by showing that it is strictly increasing.
    Notice that the functional nature of $\timet_k$ is already given by the \HTC\ semantics,
    unlike for the \HT\ approach.
First, $\tmf(0) = 0$ is implied by the denotation since $\htcinterp\models\timet_0=0$.
To show that it is increasing,
    let $k\in\intervco{0}{\lambda-1}$.
    We have that $\htcinterp\models\timet_{\kvar}-\timet_{\kvar+1} \leq -1$.
    This means that $\Vh(\timet_k)<\Vh(\timet_{k+1})$ by the denotation semantics.
    Therefore $\tmf(\kvar)<\tmf(\kvar+1)$ by construction.

  \end{proof}
\setpropcounter{prop:htc:timed:delta}
\begin{proposition}
    Let $\htcinterp$ be an \HTC\ interpretation
    and
    $\lambda\in\mathbb{N}$.

    If $\htcinterp$ is timed wrt $\lambda$
    then $\htcinterp$ is an \HTC-model of $\Delta^{c}_{\lambda}$.
\end{proposition}
\resetpropcounter

\begin{proof}
    Let $\tmf$ be the timing function induced by $\htcinterp$.
    We prove that $\htcinterp$ models each rule in $\Delta^{c}_{\lambda}$.
Since $\tmf(0)=0$ then $\Vh(\timet_0)=0$, because $\tmf$ was induced by $\htcinterp$.
    Therefore, $\htcinterp\models\timet_0=0$.
Let $0\leq\kvar<\lambda-1$.
    Since $\tmf$ is a timing function, $\tmf(k)<\tmf(k+1)$.
    This means that $\Vh(\timet_{k})<\Vh(\timet_{k+1})$.
    Therefore, $\htcinterp\models\timet_{\kvar}-\timet_{\kvar+1} \leq -1$.
\end{proof}

\setcounter{theorem}{\getrefnumber{thm:mlp:htc:completeness}}
\addtocounter{theorem}{-1}
\begin{theorem}[Completeness]
    Let \program\ be a plain metric logic program
    and
    $\M=(\tuple{\Ttrace,\Ttrace}, \tmf)$ a total timed \HT-trace of length $\lambda$.

    If
    $\M$ is a metric equilibrium model of $\program$,
    then
    $\mhtToHtc{\M}$ is a constraint equilibrium model of $\Pi_\lambda(\program)\cup\Delta^c_{\lambda}\cup\Psi^c_{\lambda}(\program)$.
\end{theorem}
\begin{proof}
    We first prove that it models the program by applying Lemma~\ref{lem:completeness-ht}.
    Now we prove that it is in equilibrium.
    By construction $\mhtToHtc{\M}=\htcttinterp$ is total
    since $\M$ is total.
Assume towards a contradiction that
    there is $\htcinterp$ such that $\Vh\subset \Vt$ and $\htcinterp\models\Pi_\lambda(\program)\cup\Delta^c_{\lambda}\cup\Psi^c_{\lambda}(\program)$.
    Then, there is some $(x,v)$ such that $(x,v)\in\Vt$ and $(x,v)\not\in\Vh$.
    We analyze the two cases for $x$.
    \casex{x=\timet_k}{
        If $(\timet_k,v)\not\in\Vh$ then there is no assignment for $\timet_k$
        because $\Vh\subset\Vt$.
        This contradict with Proposition~\ref{pro:htc:defined} since $\htcinterp\models\Delta^c_{\lambda}$
    }
    \casex{x=a_k \in \alphabets}{
We apply Lemma~\ref{lem:correctness-htc} and get $\htToMht{\htcinterp}\models\program$,
        where
        $\htToMht{\htcinterp}=(\tuple{\Htrace,\Ttrace}, \tmf)$.
        Notice that $\Ttrace$ and $\tmf$ are the same as in $\M$ since $\htcinterp$ induces $\tmf$
        and that $\Htrace<\Ttrace$
        since $a\in T_k$ but $a\not\in H_k$.
        This is a contradiction since $\M$ is in equilibrium.

    }

\end{proof}

\setcounter{theorem}{\getrefnumber{thm:mlp:htc:correctness}}
\addtocounter{theorem}{-1}
\begin{theorem}[Correctness]
    Let
    \program\ be a plain metric logic program,
    and
    $\lambda \in \mathbb{N}$.

    If
    $\htcttinterp$ is a constraint equilibrium model of $\Pi_\lambda(\program)\cup\Delta^c_{\lambda}\cup\Psi^c_{\lambda}(\program)$,
    then
    $\htcToMht{\htcttinterp}$ is a metric equilibrium model of $\program$.
\end{theorem}
\begin{proof}
    We can prove that $\htcToMht{\htcttinterp}\models\program$ by applying Lemma~\ref{lem:correctness-htc}.
    For this we only need to show  that $\htcttinterp$ is timed, which follows from Proposition~\ref{prop:htc:delta:timed},
    given that $\htcttinterp\models\Delta^c_{\lambda}$.
We it remains to prove that $\htcToMht{\htcttinterp}$ is an equilibrium model.
    Let $\htcToMht{\htcttinterp} = (\tuple{T_\kvar,T_\kvar}_{\rangeco{\kvar}{0}{\lambda}}, \tmf)$,
    and lets assume towards a contradiction that $\htcToMht{\htcttinterp}$ is not minimal.
    There is some metric trace $\M'=(\tuple{H_\kvar,T_\kvar}_{\rangeco{\kvar}{0}{\lambda}}, \tmf)$
    such that $\M'\models\program$ and $H_\kvar\subset T_\kvar$ for some $\rangeco{\kvar}{0}{\lambda}$.
    This means, that there is $a\in\alphabet$ such that $a\in T_\kvar$ and $a\not\in H_\kvar$,
    so $(a_k,\true)\in \Vt$ and $(a_k,\true)\not\in \Vh$ where $\mhtToHt{\M'}=\htcinterp$.
    By applying Lemma~\ref{lem:completeness-ht}
    we get that $\htcinterp\models\Pi_\lambda(\program)\cup\Delta^c_{\lambda}\cup\Psi^c_{\lambda}(\program)$.
    This is a contradiction since $\htcttinterp$ is a constraint equilibrium model.

\end{proof}

\subsection{General Metric Logic Programs}
\begin{definition}[Interval alphabet]
    \label{def:interval-alphabe}
    We define the interval alphabet
    \begin{align}
        \alphabetI \eqdef \{\faili{\cI}{k}{j} \mid \cI \in \allI, k,j \in \mathbb{N}, k\leq j\}
    \end{align}
\end{definition}

\begin{definition}[Auxiliary alphabet]
    \label{def:auxiliary-alphabet-mlp-to-ht}
    Given a metric logic program $\program$,
    we define the alphabet
    \begin{align}
        \alphabetaux \eqdef \{\laux{\matom}{\kvar} \mid \matom \in \program, \kvar\in\mathbb{N}\}\cup
        \{\ljaux{\metricI\odot b}{\kvar}{j} \mid \odot \in \{\eventuallyF, \alwaysF\}, \metricI\odot b \in \program, \kvar,j\in\mathbb{N}\}
    \end{align}
\end{definition}

\begin{proposition}[Equivalence of auxiliary atoms]
    \label{prop:eq:eta:aux}
    Given $\lambda\in\mathbb{N}$,
    a (simple) metric formula $\mu$ over $\alphabet$,
    an \HT\ interpretation $\handt$
    with its corresponding three-valued interpretation $\trivalI$
    and
    a timed trace $\M=(\tuple{\Htrace,\Ttrace}, \tmf)$ of length $\lambda$
    with its corresponding metric three-valued interpretation $\trivalIm$,
    such that the following propositions hold:
    \begin{enumerate}
        \item[e1.] $\trivalI(a_k)=\trivalIm(k,a)$ for all $\kinlambda$, $a\in\alphabet$
        \item[e2.] $\trivalI(\faili{\cI}{\kvar}{j})=2$ iff $\tmf(j)-\tmf(\kvar) \not\in \cI$ for all $\kinlambda$, $\cI\in\allI$
        \item[e3.] $\trivalI(\faili{\cI}{\kvar}{j})=0$ iff $\tmf(j)-\tmf(\kvar) \in \cI$ for all $\kinlambda$, $\cI\in\allI$
        \item[e4.] if $\matom=\metricI{\alwaysF}b$ then $\trivalI(\ljaux{\matom}{k}{j})=\max\{\trivalI(\laux{b}{j}),\trivalI(\faili{\cI}{\kvar}{j})\}$
              for all $\kinlambda$, $\rangeco{j}{k}{\lambda}$
        \item[e5.] if $\matom=\metricI{\eventuallyF}b$ then $\trivalI(\ljaux{\matom}{k}{j})=\min\{\trivalI(\laux{b}{j}),\trivalI(\neg\faili{\cI}{\kvar}{j})\}$
              for all $\kinlambda$, $\rangeco{j}{k}{\lambda}$
    \end{enumerate}

    Then, the following holds for any \kinlambda:

    \begin{align}
        \trivalI(\eta_k(\matom))=2 \text{ iff }
        \trivalI(\laux{\matom}{k})=\trivalIm(k,\matom)
    \end{align}
\end{proposition}
\begin{proof}
    We proceed by case analysis on the structure of $\matom$.

      \casex{\matom = a\in\alphabet}{
          We start with $\trivalI(\laux{a}{k})=\trival{k}{a}$.
          Since condition e1 holds,
          we have that $\trivalI(\laux{a}{k})=\trivalI(a_k)$.
          By applying Proposition~\ref{prop:three-valued:eq:ht},
          we get that $\trivalI(\laux{a}{k}\leftrightarrow a_k)=2$.
      }

    \def \caseg {\bot}
      \casex{\matom=\caseg}{
          We start with
          $\trivalI({\laux{\caseg}{k}})=\trival{k}{\caseg}$.
          Using the semantics of the metric three-valued interpretation,
          we have that $\trivalI({\laux{\caseg}{k}})=0$.
          With the semantics of HT three valued we get that
          $\trivalI({\laux{\caseg}{k}})=\trivalI(\caseg)$.
          By applying Proposition~\ref{prop:three-valued:eq:ht},
          we get that $\trivalI(\laux{\caseg}{k}\leftrightarrow \caseg)=2$.
      }

    \def \caseg {\top}
    \casex{\matom=\caseg}{
      Analogous to the previous case.
    }

      \def \caseg {\finally}
      \casex{\matom=\caseg}{
        We start with
        $\trivalI({\laux{\caseg}{k}})=\trival{k}{\caseg}$.
        Now we analyze the two cases for $k$.
        \casex{k < \lambda-1}{
            With the semantics of the metric three-valued interpretation,
            we have that $\trivalI({\laux{\caseg}{k}})=0$.
            With the semantics of HT three valued we get that
            $\trivalI({\laux{\caseg}{k}})=\trivalI(\caseg)$.
            By applying Proposition~\ref{prop:three-valued:eq:ht},
            we get that $\trivalI(\laux{\caseg}{k}\leftrightarrow \caseg)=2$.
        }
        \casex{k = \lambda-1}{
            With the semantics of the metric three-valued interpretation,
            we have that $\trivalI({\laux{\caseg}{k}})=1$.
            With the semantics of HT three valued we get that
            $\trivalI({\laux{\caseg}{k}})=\trivalI(\caseg)$.
            By applying Proposition~\ref{prop:three-valued:eq:ht},
            we get that $\trivalI(\laux{\caseg}{k}\leftrightarrow \caseg)=2$.
        }
      }

      \def \caseg {\initially}
      \casex{\matom=\caseg}{
        Analogous to the previous case.
      }

      \def \caseg {\metricI{\Next}b}
      \casex{\matom=\caseg}{
        We start with
       $\trivalI(\laux{\caseg}{k})=\trival{k}{\caseg}$,
       and analyze the two cases for $k$.
        \casex{k < \lambda-1}{
          Now further analysis on weather a one setp falls in the interval.
          \casex{\tau(k + 1) - \tau(k) \not\in I}{
              Three-valued semantics for metric gives us
              $\trivalI(\laux{\caseg}{k})=0$,
              which is equivalent to
              $\min\{0, \trivalI(\laux{b}{k+1})\}$.
              Using condition e2. we get
              $\min\{\trivalI(\neg\faili{\cI}{k}{j}),\trivalI(\laux{b}{k+1})\}$,
              since $\trivalI(\neg\faili{\cI}{k}{j})=0$ because $\trivalI(\faili{\cI}{k}{j})=2$.
              Finally $\trivalI(\laux{\caseg}{k})=\trivalI(\neg\faili{\cI}{k}{j}\wedge\laux{b}{k+1})$.
              Therefore, $\trivalI(\laux{\caseg}{k}\leftrightarrow \neg\faili{\cI}{k}{j}\wedge\laux{b}{k+1})=2$,
              by applying Proposition~\ref{prop:three-valued:eq:ht}.

          }
          \casex{\tau(k + 1) - \tau(k) \in I}{
            Three-valued semantics for metric gives us
              $\trivalI(\laux{\caseg}{k})=\trival{k+1}{b}$.
              Due  to condition e1. we have that
              $\trival{k+1}{b}=\trivalI(\laux{b}{k+1})$,
              so
              $\trivalI(\laux{\caseg}{k})=\trivalI(\laux{b}{k+1})$,
              which is equivalent to
              $\min\{2,\trivalI(\laux{b}{k+1})\}$.
              Using condition e3. we get
              $\min\{\trivalI(\neg\faili{\cI}{k}{j}),\trivalI(\laux{b}{k+1})\}$
              since $\trivalI(\neg\faili{\cI}{k}{j})=2$, because $\trivalI(\faili{\cI}{k}{j})=0$.
              Finally $=\trivalI(\neg\faili{\cI}{k}{j}\wedge\laux{b}{k+1})$.
              Therefore, $\trivalI(\laux{\caseg}{k}\leftrightarrow \neg\faili{\cI}{k}{j}\wedge\laux{b}{k+1})=2$,
              by applying Proposition~\ref{prop:three-valued:eq:ht}.

          }
        }
        \casex{k=\lambda-1}{
            Three-valued semantics for metric gives us
            $\trivalI(\laux{\caseg}{k})=0$,
            which is equivalent to
            $\trivalI(\bot)$.
            We apply Proposition~\ref{prop:three-valued:eq:ht}
            and get that $\trivalI(\laux{\caseg}{k}\leftrightarrow \bot)=2$.
        }
      }

      \def \caseg {\metricI{\alwaysF}\, b}
      \casex{\matom=\caseg}{
        We check that the two equivalences in $\eta_k(\caseg)$ hold.
First, notice that
        $\ljaux{\caseg}{k}{j} \leftrightarrow \faili{\cI}{k}{j} \vee \laux{b}{j}$
        holds trivially by construction.
Now for $\laux{\caseg}{k} \leftrightarrow \bigwedge_{j=k}^{\lambda-1} \ljaux{\caseg}{k}{j}$,
        we start with $\trivalI(\laux{\caseg}{k})=\trival{k}{\caseg}$.
        This is equivalent to
        $\min\{2\}\cup \{\trival{j}{b} \mid k \leq j < \lambda \text{ and } \tau(j) - \tau(k) \in I\} $,
        which due to condition e1. is equivalent to
        $\min\{2\}\cup \{\trivalI(\laux{b}{j}) \mid k \leq j < \lambda \text{ and } \tau(j) - \tau(k) \in I\} $.
        The invariant of $\min$ lets us add multiple instances of $2$, giving us
        $\min \{\trivalI(\laux{b}{j}) \mid k \leq j < \lambda \text{ and } \tau(j) - \tau(k) \in I\} \cup
                \{2 \mid k \leq j < \lambda \text{ and } \tau(j) - \tau(k) \not\in I\} $.
        With conditions e2. and e3.
        we know that the value of $\trivalI(\faili{\cI}{k}{j})$
        for the conditions of each set is 0 and 2 respectively,
        so we can substitute them and get
        $\min \{\max\{\trivalI(\laux{b}{j}),\trivalI(\faili{\cI}{k}{j})\} \mid k \leq j < \lambda \text{ and } \tau(j) - \tau(k) \in I\} \cup
                \{\trivalI(\faili{\cI}{k}{j}) \mid k \leq j < \lambda \text{ and } \tau(j) - \tau(k) \not\in I\} $.
        Now the invariant of $\max$ gives us
        $\min \{\max\{\trivalI(\laux{b}{j}),\trivalI(\faili{\cI}{k}{j})\} \mid k \leq j < \lambda \text{ and } \tau(j) - \tau(k) \in I\} \cup
                \{\max\{\trivalI(\laux{b}{j}),\trivalI(\faili{\cI}{k}{j})\} \mid k \leq j < \lambda \text{ and } \tau(j) - \tau(k) \not\in I\} $.
        With simple set operations since the values are the same we can make the union and remove the condition
        $\min \{\max\{\trivalI(\laux{b}{j}),\trivalI(\faili{\cI}{k}{j})\} \mid k \leq j < \lambda \}$,
        which thanks to condition e4. is equivalent to
        $\min \{\trivalI(\ljaux{\caseg}{k}{j}) \mid k \leq j < \lambda \}$
        and finally to
        $\trivalI(\bigwedge_{j=k}^{\lambda-1}\ljaux{\caseg}{k}{j})$.
        Therefore $\laux{\caseg}{k} \leftrightarrow \bigwedge_{j=k}^{\lambda-1} \ljaux{\caseg}{k}{j}$
      }

      \def \caseg {\metricI{\eventuallyF}\, b}
      \casex{\matom=\caseg}{
      Analogous to the previous case.
      }

  \end{proof}

 \subsubsection{Translation of MHT to General Metric Logic Programs}

\begin{lemma}[Completeness]
    \label{lem:mel2mlp1}
    Let $\varphi$ be a metric formula over alphabet $\alphabet$
    and let
    $\M=(\tuple{\Htrace,\Ttrace}, \tmf)$ be a timed trace of length $\lambda$
    with corresponding metric three-valued interpretation $\trivalIm$,
    such that $\M \models \varphi$.
    Then, there exists some timed trace $\M'=(\tuple{\Htrace',\Ttrace'}, \tmf)$ over alphabet $\alphabet \cup \mathcal{L}_\varphi$
    such that $\tuple{\Htrace,\Ttrace} = \tuple{\Htrace',\Ttrace'}|_{\alphabet}$
    and $\M' \models \Theta(\varphi)$.
\end{lemma}

\begin{proof}
  Given $\trivalIm$,
  we construct $\M'$
  by defining its corresponding
  three-valued metric interpretation ${\trivalIm}'$
  such that
  $\trivaluation{k}{\Label\phi}{{\trivalIm}'}=\trivaluation{k}{\phi}{\trivalIm}$
  or any $\phi \in \subformulas{\varphi}$ and $\kinlambda$,
and for any atom $a \in \alphabet$  we have
  $\trivaluation{k}{a}{{\trivalIm}'}=\trivaluation{k}{a}{\trivalIm}$ for $\kinlambda$.
By construction we then know that $\tuple{\Htrace,\Ttrace} = \tuple{\Htrace',\Ttrace'}|_{\alphabet}$.

  Now we proceed to proof that $\M' \models \Theta(\varphi)$,
  so $\M' \models \{\metric{\alwaysF}{0}{\omega}\initially \rightarrow \Label{\varphi}\} \cup
      \{\eta^*(\phi) \mid \phi \in \subformulas{\varphi}\}$.
  We can prove  $\M' \models \eta(\phi)$ for any $\phi \in \subformulas{\varphi}$
  by analyzing the structure as done in the proof for Lemma~\ref{lem:mel2mlp2},
  (note that we use the construction of $\M'$ instead of the induction hypothesis).
Now we must show that $\M' \models \{\metric{\alwaysF}{0}{\omega}\initially \rightarrow \Label{\varphi}\}$.
We know $\M\models\varphi$ by hypothesis,
  so $\M,0\models\varphi$.
  Proposition~\ref{prop:three-valued-semantics-eq} gives us
  $\trivaluation{0}{\varphi}{\trivalIm}=2$,
  which by construction of $\M'$ gives us
  $\trivaluation{0}{\Label\varphi}{{\trivalIm}'}=2$.
  Since
  $\trivaluation{0}{\initially \rightarrow \Label{\varphi}}{{\trivalIm}'}=\trivaluation{0}{\Label\varphi}{{\trivalIm}'}$
  we have $\trivaluation{0}{\initially \rightarrow \Label{\varphi}}{{\trivalIm}'}=2$.
  Notice also that $\trivaluation{j}{\initially \rightarrow \Label{\varphi}}{{\trivalIm}'}=2$
  for all $j>0$.
  since $\trivaluation{j}{\initially}{{\trivalIm}'}=0$.
  Then $\min\{2\}\cup \{\trivaluation{j}{\varphi}{\trivalIm} \mid 0 \leq j < \lambda \text{ and } \tau(j) - \tau(k) \in I\}=2$
  so $\trivaluation{0}{\metric{\alwaysF}{0}{\omega}\initially \rightarrow \Label{\varphi}}{{\trivalIm}'}=2$,
  which by proposition~\ref{prop:three-valued-semantics-eq} gives us
  $\M',0 \models \metric{\alwaysF}{0}{\omega}\initially \rightarrow \Label{\varphi}$
  and $\M' \models \Theta(\varphi)$.

\end{proof}

\begin{lemma}[Correctness]
    \label{lem:mel2mlp2}
    Let $\varphi$ be a metric formula over $\alphabet$ and
    let $\M=(\tuple{\Htrace,\Ttrace}, \tmf)$ be a timed trace of $\Theta(\varphi)$,
    with corresponding metric three-valued interpretation $\trivalIm$.
    Then, for any $\gamma \in \mathit{\subformulas\varphi}$ and $\kinlambda$,
    we have $\trivaluation{k}{\Label{\gamma}}{\trivalIm} = \trivaluation{k}{\gamma}{\trivalIm}$.
\end{lemma}

\begin{proof}
  We proceed by structural induction on $\phi$.

  \def \caseg {a}
  \casex{\phi\in\{a,\bot,\initially\}}{
  From \eqref{eq:tseitin:general},
  since $\M \models \Theta(\varphi)$,
  we obtain $\M \models \alwaysF(\Label{\phi} \leftrightarrow \phi)$. The satisfaction of MHT then ensures that for any $k \in \lambda$,
  we have $\M,
  k \models \Label{\phi} \leftrightarrow \phi$. Applying Proposition~\ref{prop:three-valued-semantics-eq},
  we conclude that for any $k \in \lambda$,
  the valuation $\trivaluation{k}{\Label{\phi}}{\trivalIm}$ is equal to $\trivaluation{k}{\phi}{\trivalIm}$.
  }
  \def \caseg {\zeta  \otimes \psi}
  \def \casegl {\Label\zeta  \otimes \Label\psi}
  \casex{\phi=\caseg}{
  From \eqref{eq:tseitin:general},
  we have $\M \models \alwaysF(\Label{\phi} \leftrightarrow \Label{\zeta} \otimes \Label{\psi})$. From MEL satisfaction,
  it follows that for any $k \in \lambda$,
  $\M,
  k \models \Label{\phi} \leftrightarrow \Label{\zeta} \otimes \Label{\psi}$. Using Proposition~\ref{prop:three-valued-semantics-eq},
  we get $\trivaluation{k}{\Label{\phi}}{\trivalIm} = \trivaluation{k}{\Label{\zeta} \otimes \Label{\psi}}{\trivalIm}$.
  By Definition~\ref{def:three-valued-semantics-mht},
  this is equal to $f^\otimes\{\trivaluation{k}{\Label{\zeta}}{\trivalIm},
  \trivaluation{k}{\Label{\psi}}{\trivalIm}\}$. Applying the induction hypothesis,
  we replace $\trivaluation{k}{\Label{\zeta}}{\trivalIm}$ and $\trivaluation{k}{\Label{\psi}}{\trivalIm}$ with $\trivaluation{k}{\zeta}{\trivalIm}$ and $\trivaluation{k}{\psi}{\trivalIm}$,
  respectively. This simplifies to $\trivaluation{k}{\phi}{\trivalIm}$ by Definition~\ref{def:three-valued-semantics-mht},
  completing this case.
  }
  \def \caseg {\zeta  \to \psi}
  \def \casegl {\Label\zeta  \to \Label\psi}
  \casex{\phi=\caseg}{
  From \eqref{eq:tseitin:general},
  we obtain $\M \models \alwaysF(\Label{\phi} \leftrightarrow \Label{\zeta} \to \Label{\psi})$. By MEL satisfaction,
  this implies that for any $k \in \lambda$,
  $\M,
  k \models \Label{\phi} \leftrightarrow \Label{\zeta} \to \Label{\psi}$. From Proposition~\ref{prop:three-valued-semantics-eq},
  we conclude that $\trivaluation{k}{\Label{\phi}}{\trivalIm} = \trivaluation{k}{\Label{\zeta} \to \Label{\psi}}{\trivalIm}$.

  We now distinguish two cases.
  - If $\trivaluation{k}{\Label{\zeta}}{\trivalIm} \leq \trivaluation{k}{\Label{\psi}}{\trivalIm}$,
  then by Definition~\ref{def:three-valued-semantics-mht},
  $\trivaluation{k}{\Label{\zeta} \to \Label{\psi}}{\trivalIm} = 2$. The induction hypothesis tells us that $\trivaluation{k}{\zeta}{\trivalIm} \leq \trivaluation{k}{\psi}{\trivalIm}$,
  so applying Definition~\ref{def:three-valued-semantics-mht} again,
  we conclude $\trivaluation{k}{\phi}{\trivalIm} = 2$.
  - Otherwise,
  if $\trivaluation{k}{\Label{\zeta}}{\trivalIm} > \trivaluation{k}{\Label{\psi}}{\trivalIm}$,
  then by Definition~\ref{def:three-valued-semantics-mht},
  we have $\trivaluation{k}{\Label{\zeta} \to \Label{\psi}}{\trivalIm} = \trivaluation{k}{\Label{\psi}}{\trivalIm}$. The induction hypothesis then gives us $\trivaluation{k}{\zeta}{\trivalIm} > \trivaluation{k}{\psi}{\trivalIm}$,
  so by applying Definition~\ref{def:three-valued-semantics-mht},
  we again obtain $\trivaluation{k}{\phi}{\trivalIm} = \trivaluation{k}{\psi}{\trivalIm}$.
  }
  \def \caseg {\metricI{\Next}\psi}
  \def \casegl {\metricI{\Next}\Label\psi}
  \casex{\phi=\caseg}{
  By \eqref{eq:tseitin:general},
  we have $\M \models \alwaysF(\Label{\phi} \leftrightarrow \metricI{\Next} \Label{\psi})$. From MEL satisfaction,
  it follows that for any $k \in \lambda$,
  we have $\M,
  k \models \Label{\phi} \leftrightarrow \metricI{\Next} \Label{\psi}$. By Proposition~\ref{prop:three-valued-semantics-eq},
  this gives us $\trivaluation{k}{\Label{\phi}}{\trivalIm} = \trivaluation{k}{\metricI{\Next} \Label{\psi}}{\trivalIm}$.

  If $k + 1 = \lambda$ or $\tau(k + 1) - \tau(k) \notin I$,
  then by Definition~\ref{def:three-valued-semantics-mht},
  we obtain $\trivaluation{k}{\casegl}{\trivalI} = 0$,
  where also $0= \trivaluation{k}{\caseg}{\trivalIm}$.
  Otherwise,
  we use Definition~\ref{def:three-valued-semantics-mht} to get
  $\trivaluation{k}{\casegl}{\trivalI} = \trivaluation{k+1}{\Label{\psi}}{\trivalIm}$. Applying the induction hypothesis,
  we replace $\trivaluation{k+1}{\Label{\psi}}{\trivalIm}$ with $\trivaluation{k+1}{\psi}{\trivalIm}$,
  which simplifies to $\trivaluation{k}{\phi}{\trivalIm}$ as required.
  }
  \def \caseg {\metricI{\alwaysF}\, \psi}
  \def \casegl {\metricI{\alwaysF}\, \Label\psi}
  \casex{\phi=\caseg}{
  By \eqref{eq:tseitin:general},
  we obtain $\M \models \alwaysF(\Label{\phi} \leftrightarrow \metricI{\alwaysF} \Label{\psi})$,
  which implies through MEL satisfaction that $\M,
  k \models \Label{\phi} \leftrightarrow \metricI{\alwaysF} \Label{\psi}$ for all $k \in \lambda$. By Proposition~\ref{prop:three-valued-semantics-eq},
  we get $\trivaluation{k}{\Label{\phi}}{\trivalIm} = \trivaluation{k}{\metricI{\alwaysF} \Label{\psi}}{\trivalIm}$.
  By Definition~\ref{def:three-valued-semantics-mht} we know
  $\trivaluation{k}{\casegl}{\trivalIm} = \min\{2, \trivaluation{j}{\Label\psi}{\trivalIm} \mid k \leq j < \lambda \text{ and } \tau(j) - \tau(k) \in I\}$,
  which by induction can be replaced to get
  $= \min\{2, \trivaluation{j}{\psi}{\trivalIm} \mid k \leq j < \lambda \text{ and } \tau(j) - \tau(k) \in I\}$.
  Finally with Definition~\ref{def:three-valued-semantics-mht},
  this is $\trivaluation{k}{\caseg}{\trivalIm}$.
  }
  \def \caseg {\metricI{\eventuallyF}\, \psi}
  \def \casegl {\metricI{\eventuallyF}\, \Label\psi}
  \casex{\phi=\caseg}{
    Analogous to previous case.
  }
  \end{proof}

\setcounter{theorem}{\getrefnumber{thm:correct-mel-to-mlp}}
\addtocounter{theorem}{-1}
\begin{theorem}[Correctness and Completeness]
    Let $\varphi$ be a metric formula over $\alphabet$.
    Then, we have
    \begin{align*}
      \{(\tuple{\Htrace,\Ttrace}, \tmf) \mid (\tuple{\Htrace,\Ttrace}, \tmf) \models \varphi\} =
      \{({\tuple{\Htrace',\Ttrace'}}|_\alphabet, \tmf) \mid (\tuple{\Htrace',\Ttrace'}, \tmf) \models \Theta(\varphi)\}
    \end{align*}
  \end{theorem}
\begin{proof}
    \casex{\subseteq }{
      Follows from Lemma~\ref{lem:mel2mlp1}.
    }
    \casex{\supseteq}{
      Take some timed trace $\M=(\tuple{\Htrace',\Ttrace'},\tmf)$ of length $\lambda$
      that models $\Theta(\varphi)$,
      with corresponding metric three-valued interpretation $\trivalIm$.
      Then $\M\models\metric{\alwaysF}{0}{\omega} (\initially \rightarrow \Label{\varphi})$ (Definition of satisfaction of MLP)
      so $\M,0\models\Label{\varphi}$ (Satisfaction of always and initial).
      This implies that
      $\trivaluation{0}{\Label\varphi}{\trivalIm}=2$ (Proposition~\ref{prop:three-valued:eq:mht:pointwise}).
      We can apply Lemma~\ref{lem:mel2mlp2} with $\gamma=\varphi$ and $k=0$ to conclude
      $\trivaluation{0}{\Label\varphi}{\trivalIm}=\trivaluation{0}{\varphi}{\trivalIm}=2$ and so,
      $\M$ is also a model of $\varphi$.
      Finally, $(\tuple{\Htrace',\Ttrace'}|_\alphabet,\tmf)$ is still a model of $\varphi$ because the latter is a theory for vocabulary $\alphabet$.
    }
\end{proof}
 \subsubsection{Translation of General Metric Logic Programs to \HT}

\begin{lemma}[Splitting General HT]
  \label{lem:splitting:full:ht}
  Given a metric logic program $\program$
  and a $\lambda,\tmflimit\in\mathbb{N}$,
  the stable models of $\newPi_\lambda(\program)\cup\Delta_{\lambda,\tmflimit}\cup\newPsi_{\lambda,\tmflimit}(\program)$ can be computed by means of the splitting technique.
\end{lemma}
\begin{proof}
  We will show how the stable models of $\mathbb{P} = \newPi_\lambda(\program)\cup\Delta_{\lambda,\tmflimit}\cup\newPsi_{\lambda,\tmflimit}(\program)$ can be computed
  by sequentially applying the splitting theorem. Let us start by fixing $U_0 = \lbrace t_{k,d} \mid k, d \in \mathbb{N}\rbrace$.
  Clearly it is an splitting set because the elements of $U_0$ are only generated by $\Delta_{\lambda,\tmflimit}$. Moreover, $b_{U_0} = \Delta_{\lambda,\tmflimit}$.
  By Theorem~\ref{thm:splitting}, an stable model $T$ of $\mathbb{P}$ can be computed as,
  $T=T_0\cup T'$ where $T_0$ is a stable model of $b_{U_0}$.
$T'$ is a stable model of $Reduce(\mathbb{P}\setminus b_{U_0},T_0,U_0\setminus T_0)$.
  If we analyze this first simplification, we can assume, without loss of generality,
  that $Reduce(\mathbb{P}\setminus b_{U_0},T_0,U_0\setminus T_0) =  \newPi_\lambda(\program)\cup Reduce(\newPsi_{\lambda,\tmflimit}(\program), T_0,U_0\setminus T_0)$,
  because $\newPi_\lambda(\program)$ does not contain atoms in $U_0$.
Let us denote by $\mathbb{P}_1 = \newPi_\lambda(\program)\cup Reduce(\newPsi_{\lambda,\tmflimit}(\program), T_0,U_0\setminus T_0)$
  and let us consider the set $U_1 = \lbrace \delta_{[m..n)}^{k,j} \mid k,j \in \mathbb{N} \hbox{ and }[m..n)\in \mathbb{I\mid_{P}}\rbrace$.
  $U_1$ is also an splitting set because elements of $U_1$ only occur (as facts) in $ Reduce(\newPsi_{\lambda,\tmflimit}(\program), T_0,U_0\setminus T_0)$.
  More precisely, $b_{U_1} = Reduce(\newPsi_{\lambda,\tmflimit}(\program), T_0,U_0\setminus T_0)$.
  By applying (once again) Theorem~\ref{thm:splitting}, we can say that $T' = T_1 \cup T_2$ where $T_1$ is an stable model of $ Reduce(\newPsi_{\lambda,\tmflimit}(\program), T_0,U_0\setminus T_0)$ and
  $T_2$ is a stable model of
  \begin{equation*}
      Reduce(\mathbb{P}_1\setminus b_{U_1},T_1, U_1\setminus T_1) = Reduce(\newPi_\lambda(\program), T_1,U_1\setminus T_1).
  \end{equation*}
\end{proof}

\begin{proposition}[Equivalences for correspondence]
  \label{prop:ht:corr:eq}
  Given an \HT\ interpretation $\handt$
  with a corresponding three-valued interpretation $\trivalI$,
  and
  a timed trace $\M=(\tuple{\Htrace,\Ttrace}, \tmf)$ of length $\lambda$
  over alphabet $\alphabet$
  with a corresponding three-valued metric interpretation $\trivalIm$,
  we have that
  $\handt|_{\alphabets\cap\alphabetT}=\mhtToHt{\M}$,
  iff the following properties hold:

  \begin{enumerate}
    \item $\trivalI(a_k)=\trivalIm(a,k)$ for all $\kinlambda$, $a\in\alphabet$
    \item $\trivalI(\timet_{k,\tmvar})=2$ iff $\tmf(k)=\tmvar$ for all $\kinlambda$, $\tmvar\in\intervcc{0}{\tmflimit}$
    \item $\trivalI(\timet_{k,\tmvar})=0$ iff $\tmf(k)\neq\tmvar$ for all $\kinlambda$, $\tmvar\in\intervcc{0}{\tmflimit}$
  \end{enumerate}
  where $\tmflimit=\tmf(\lambda-1)$,
\end{proposition}

\begin{proof}
  Trivial by construction of $\mhtToHt{\M}$
\end{proof}

\begin{definition}[Construction of an HT three-valued interpretation given a metric interpretation]
  \label{def:ht:trival:construction}
  Given a
  timing function $\tmf$,
  a metric three-valued interpretation $\trivalIm$,
  and a program $\program$,
  we construct a three-valued interpretation $\exmtri$
  over alphabet $\alphabets\cup\alphabetT\cup\alphabetaux\cup\alphabetI$
  as follows:

  \begin{enumerate}
    \item $\exmtri(a_k)=\trivalIm(a,k)$ for all $a\in\alphabet$
    \item $\exmtri(\timet_{k,\tmvar})=2$ iff $\tmf(k)=\tmvar$ for all $\tmvar\in\intervcc{0}{\tmflimit}$
    \item $\exmtri(\timet_{k,\tmvar})=0$ iff $\tmf(k)\neq\tmvar$ for all $\tmvar\in\intervcc{0}{\tmflimit}$
    \item $\exmtri(\laux{\matom}{k})=\trivalIm(\matom,k)$ for all $\matom\in\program$
    \item $\exmtri(\faili{\cI}{\kvar}{j})=2$ if $\tmf(j)-\tmf(\kvar) \not\in \cI$ for all $\cI\in\allI$
    \item $\exmtri(\faili{\cI}{\kvar}{j})=0$ if $\tmf(j)-\tmf(\kvar) \in \cI$ for all $\cI\in\allI$
    \item $\exmtri(\ljaux{\metricI{\alwaysF}b}{k}{j})=\max\{\exmtri(\laux{b}{j}),\exmtri(\faili{\cI}{\kvar}{j})\}$
          for all $\metricI{\alwaysF}b\in\program$
    \item $\exmtri(\ljaux{\metricI{\eventuallyF}b}{k}{j})=\min\{\exmtri(\laux{b}{j}),\exmtri(\neg\faili{\cI}{\kvar}{j})\}$
          for all $\metricI{\eventuallyF}b\in\program$
  \end{enumerate}
  for all $\kinlambda$, $\rangeco{j}{k}{\lambda}$
\end{definition}

\begin{proposition}[Construction is total]
  \label{prop:ht:construction:total}
  Given a metric three-valued interpretation $\trivalIm$,
  and a program $\program$,
  if $\trivalIm$ is total
  then $\exmtri$ is total.
\end{proposition}

\begin{proof}
  For $x\in\alphabetT\cup\alphabetI$ it is trivial by construction.
  For $x\in\alphabets$ it holds since $\trivalIm$ is total.
  For $x\in\alphabetaux$ it holds recursively since min and max only give values in the set.
\end{proof}

\begin{definition}[Additional interval constraints for monotonic]
  \label{def:ht:interval:eq}
  \begin{align}
    \newPsi^{\bot}_{\lambda,\tmflimit}(\program) \eqdef
     & \ \{\bot \leftarrow
      \faili{\cI}{\kvar}{j} \wedge
      \timet_{\kvar,\tmvar} \wedge
      \timet_{j,\tmvar'}
      \mid
      0 \leq \kvar \leq j < \lambda-1,
      0\leq d<d'\leq\tmflimit,
      \tmvar'-\tmvar \in \cI,
      \cI \in \mathcal{I}\}\;\cup
  \end{align}

\end{definition}

\begin{lemma}[Modeling of extra constraints]
  \label{lem:mlp:full:ht:monotonic}
  Let
  \program\ be a metric logic program,
  and
  $\lambda,\tmflimit \in \mathbb{N}$.

  If
  $\tandt$ is an equilibrium model of $\Delta_{\lambda,\tmflimit}\cup\newPsi_{\lambda,\tmflimit}(\program)$,
  then
  $\tandt\models\newPsi^\bot_{\lambda,\tmflimit}(\program)$
\end{lemma}
\begin{proof}
  We show that it models each integrity constraint in $\newPsi^\bot_{\lambda,\tmflimit}(\program)$.
Lets assume towards a contradiction that there is constraint $c = \bot \leftarrow
    \faili{\intervco{m}{n}}{\kvar}{j} \wedge
    \timet_{\kvar,\tmvar} \wedge
    \timet_{j,\tmvar'}$ that is not modeled by $\tandt$.
  This implies that  $\tandt\models\faili{\intervco{m}{n}}{\kvar}{j}$
  and $\tandt\models\timet_{\kvar,\tmvar}$
  and $\tandt\models\timet_{j,\tmvar'}$,
Since $\tandt$ is an equilibrium model of $\Delta_{\lambda,\tmflimit}$,
  $\tandt$ is timed inducing $\tmf$,
  where $\tmf(k)=\tmvar$ and $\tmf(j)=\tmvar'$.
Given that $\tandt\models\faili{\intervco{m}{n}}{\kvar}{j}$,
  then it must be founded by a rule in $\newPsi_{\lambda}(\program)$,
  because it is in equilibrium.
  If it comes from \eqref{def:ht:interval:one} then, the condition is that
  $\tmvar'-\tmvar < m$,
  which leads to a contradiction since
  $\tmvar'-\tmvar \in \cI$ is a condition for $c$.
  If it comes from \eqref{def:ht:interval:two} then, the condition is that
  $\tmvar'-\tmvar \geq n$,
  which leads to a contradiction since
  $\tmvar'-\tmvar \in \cI$ is a condition for $c$.

\end{proof}

\begin{lemma}[Completeness lemma HT]
  \label{lem:mlp:full:ht:completeness}
  Let \program\ be a metric logic program over $\alphabet$,
  and
  $\M=(\tuple{\Htrace,\Ttrace}, \tmf)$ a timed \HT-trace of length $\lambda$.

  If
  $\M\models\program$,
  then
  there exists an HT interpretation  $\handt$
  such that
  $\handt\models\newPi_\lambda(\program)\cup
    \Delta_{\lambda,\tmflimit}\cup
    \newPsi_{\lambda,\tmflimit}(\program)$
  with $\tmflimit=\tmf(\lambda-1)$,
  and $\tuple{H,T}|_{\alphabets\cup\alphabetT}=\mhtToHt{\M}$.

\end{lemma}

\begin{proof}

    We construct $\trivalI=\exmtri$ as done in Definition~\ref{def:ht:trival:construction}
    and let $\handt$ be the HT interpretation associated to $\trivalI$.
    We notice the properties from Proposition~\ref{prop:ht:corr:eq} are satisfied
    by the construction of $\exmtri$, so we can apply Proposition~\ref{prop:ht:corr:eq}
    to assure that $\tuple{H,T}|_{\alphabets\cup\alphabetT}=\mhtToHt{\M}$.
We now proceed to show that $\handt\models\newPi_\lambda(\program)\cup\Delta_{\lambda,\tmflimit}\cup\newPsi_{\lambda,\tmflimit}(\program)$
    by showing that it is a model of each part.

    \medskip
    First, we notice that by construction $\handt$ is timed wrt. $\lambda$ inducing $\tmf$.
    We can apply Proposition~\ref{prop:ht:timed:delta}
    to show that $\handt\models\Delta_{\lambda,\tmflimit}$ for $\tmflimit=\tmf(\lambda-1)$.

    \medskip
    Next, we show that $\handt\models\newPsi_{\lambda,\tmflimit}(\program)$
    by showing that $\trivalI(\newPsi_{\lambda,\tmflimit}(\program))=2$.
    Then, we must show that for every rule $r=\faili{\cI}{\kvar}{j} \leftarrow \timet_{\kvar,\tmvar} \wedge \timet_{j,\tmvar'} $
    in $\newPsi_{\lambda,\tmflimit}(\program)$,
    we have that $\trivalI(r)=2$.
    Notice that $r$ must come either from
    \eqref{def:ht:interval:one} or \eqref{def:ht:interval:two}.
    The analysis for both cases is the same so we will proceed, with \eqref{def:ht:interval:one},
    and leave \eqref{def:ht:interval:two} to the reader.
The conditions for \eqref{def:ht:interval:one} state that $\tmvar'-\tmvar< m$,
    since $\handt$ it is timed, we have
    $\timet_{\kvar,\tmvar}$
    and $\timet_{j,\tmvar'}$
    with $\tmf(j)=\tmvar'$ and $\tmf(\kvar)=\tmvar$.
    So we have that $\tmf(j)-\tmf(\kvar) \not\in \cI$.
    If this is the case, then $\trivalI(\faili{\cI}{\kvar}{j})=2$
    by construction of $\exmtri$.
    Then the three-valued semantics for \HT\ gives us that
    $\trivalI(\faili{\cI}{\kvar}{j} \leftarrow \timet_{\kvar,\tmvar} \wedge \timet_{j,\tmvar'})=2$,
    since $2\geq\{0,1,2\}$.

    \medskip
    Finally, we show that $\handt\models\newPi_\lambda(\program)$ by showing that $\trivalI(\newPi_\lambda(\program))=2$.
    Then, we must show that
    for every rule $\alwaysF{r} \in \program$, and \kinlambda,
    we have $\trivalI(\tk{r})=2$,
    and for every $\matom \in \program$,
    we have $\trivalI(\eta_k^*(\matom))=2$.
    For the first point,
    we know by construction of $\exmtri$,
    that
    $\trival{k}{\mu}=\trivalI(\laux{k}{\mu})$.
    With this equivalence we use the substitution to know that
    if $\trival{k}{r}=2$ then $\trivalI(\tk{r})=2$.
    And $\trival{k}{r}=2$ since $\M$ is a model of $\program$.
For the second point,
    we can directly apply Proposition~\ref{prop:eq:eta:aux}
    given that the requirements are satisfied by Proposition~\ref{prop:ht:corr:eq}.

  \end{proof}

\begin{lemma}[Correctness lemma HT]
  \label{lem:mlp:full:ht:correctness}
  Let
  \program\ be a metric logic program,
  $\lambda,\tmflimit \in \mathbb{N}$,
  and $\handt$ an timed HT interpretation wrt. $\lambda$ inducing $\tmf$.

  If
  $\handt\models\newPi_\lambda(\program)\cup
    \newPsi_{\lambda,\tmflimit}(\program)\cup
    \newPsi^{\bot}_{\lambda,\tmflimit}(\program)$,
  then
  $\htToMht{\tuple{H,T}|_{\alphabets\cup\alphabetT}}\models\program$.
\end{lemma}
\begin{proof}
Let $\htToMht{\handt|_{\alphabetT\cup\alphabets}}=\M$
and let $\trivalI$ be the three-valued interpretation corresponding to $\handt$.

We will show below that each of the conditions of Proposition~\ref{prop:eq:eta:aux}.
to prove the point-wise equivalence for the auxiliary atoms,
namely $\trivalI(\laux{\matom}{k})=\trival{k}{\matom}$
for all \kinlambda.
With this in hand we argue that $\M\models\program$ follows directly from $\handt\models\newPi_\lambda(\program)$,
given that with these equivalence for any \kinlambda,
$\trivalI(\tk{r})=2$ iff $\trival{k}{r}=2$.

\medskip
Notice that $\handt|_{\alphabets\cap\alphabetT}=\mhtToHt{\M}$,
so we apply Proposition~\ref{prop:ht:corr:eq} and get
$\trivalI(a_k)=\trival{k}{a}$ for all $\kinlambda$, $a\in\alphabet$.

\medskip
We will now show that
$\trivalI(\faili{\cI}{\kvar}{j})=2$ iff $\tmf(j)-\tmf(\kvar) \not\in \cI$ for all $\kinlambda$, $\cI\in\allI$.
For this
let  $\kinlambda$ and $\cI=\intervco{m}{n}\in\allI$.
We know that $\handt\models\timet_{k,\tmvar}$ and $\handt\models\timet_{j,\tmvar'}$
with $\tmvar'>d$, where $\tmf(k)=\tmvar$ and $\tmf(j)=\tmvar'$,
since $\handt$ is timed wrt. $\lambda$.
This shows that
$\trivalI(\timet_{k,\tmvar}\wedge\timet_{j,\tmvar'})=2$,
since
$\trivalI(\timet_{k,\tmvar})=2$ and $\trivalI(\timet_{j,\tmvar'})=2$.
Now we prove both directions of the equivalence:
\casex{\leftarrow}{
    We start with $\tmf(j)-\tmf(\kvar) \not\in \cI$,
    which gives us the conditions in \eqref{def:ht:interval:one}
    or \eqref{def:ht:interval:two}
    depending on the two cases for not being in $\cI$.
    Then,
    $\trivalI(\faili{\cI}{\kvar}{j} \leftarrow \timet_{\kvar,\tmvar} \wedge \timet_{j,\tmvar'})=2$
    since $\handt\models\newPsi_{\lambda,\tmflimit}(\program)$.
    By using the three-valued semantics for \HT\ we know that
    $\trivalI(\faili{\cI}{\kvar}{j})=2$,
    since $\trivalI(\timet_{k,\tmvar}\wedge\timet_{j,\tmvar'})=2$.
}
\casex{\rightarrow}{
    We start with $\trivalI(\faili{\cI}{\kvar}{j})=2$,
    and let's assume towards a contradiction that $\tmf(j)-\tmf(\kvar) \in \cI$.
    Then,
    $\trivalI(\bot \leftarrow
        \faili{\intervco{m}{n}}{\kvar}{j} \wedge
        \timet_{\kvar,\tmvar} \wedge
        \timet_{j,\tmvar'})=2$
        since
    $\handt\models\newPsi^{\bot}_{\lambda,\tmflimit}(\program)$.
    With three-valued semantics for \HT\ we know that
    $\trivalI(
        \faili{\intervco{m}{n}}{\kvar}{j} \wedge
        \timet_{\kvar,\tmvar} \wedge
        \timet_{j,\tmvar'})=0$, which leads to a contradiction,
        since $\trivalI(\faili{\cI}{\kvar}{j})=2$ and $\trivalI(\timet_{k,\tmvar}\wedge\timet_{j,\tmvar'})=2$.

}
Therefore, $\trivalI(\faili{\cI}{\kvar}{j})=2$ iff $\tmf(j)-\tmf(\kvar) \not\in \cI$ for all $\kinlambda$, $\cI\in\allI$

\medskip
We will now show that
$\trivalI(\faili{\cI}{\kvar}{j})=0$ iff $\tmf(j)-\tmf(\kvar) \in \cI$ for all $\kinlambda$, $\cI\in\allI$.
For this
let  $\kinlambda$ and $\cI=\intervco{m}{n}\in\allI$.
As before, we know that $\handt\models\timet_{k,\tmvar}$ and $\handt\models\timet_{j,\tmvar'}$,
with $\tmvar'>d$, where $\tmf(k)=\tmvar$ and $\tmf(j)=\tmvar'$,
since it is timed.
Now we prove both directions of the equivalence:
\casex{\leftarrow}{
    We start with $\tmf(j)-\tmf(\kvar) \in \cI$.
    Since $\handt\models\newPsi^{\bot}_{\lambda,\tmflimit}(\program)$,
    then
    $\trivalI(\bot \leftarrow
        \faili{\intervco{m}{n}}{\kvar}{j} \wedge
        \timet_{\kvar,\tmvar} \wedge
        \timet_{j,\tmvar'})=2$,
    which means that
    $\trivalI(
                \faili{\intervco{m}{n}}{\kvar}{j} \wedge
                \timet_{\kvar,\tmvar} \wedge
                \timet_{j,\tmvar'})=0$.
    The three-valued semantics for \HT\ gives us that
    $\min\{
        \trivalI(\faili{\intervco{m}{n}}{\kvar}{j}), 2\}=0$,
        after substituting the known values.
    Then $\trivalI(\faili{\cI}{\kvar}{j})=0$.

}
\casex{\rightarrow}{
    We start with $\trivalI(\faili{\cI}{\kvar}{j})=0$,
    and assume towards a contradiction that $\tmf(j)-\tmf(\kvar) \not\in \cI$.
    Then
    $\trivalI(\faili{\cI}{\kvar}{j} \leftarrow \timet_{\kvar,\tmvar} \wedge \timet_{j,\tmvar'})=2$
    since $\handt\models\newPsi_{\lambda,\tmflimit}(\program)$
    and we can apply \eqref{def:ht:interval:one}
    or \eqref{def:ht:interval:two}
    depending on the two cases for not being in $\cI$.
    Contradiction since $\trivalI(\timet_{k,\tmvar}\wedge\timet_{j,\tmvar'})=2$.
}
Therefore, $\trivalI(\faili{\cI}{\kvar}{j})=0$ iff $\tmf(j)-\tmf(\kvar) \in \cI$ for all $\kinlambda$, $\cI\in\allI$.

\medskip
We will now show that that
$\trivalI(\ljaux{\matom}{k}{j})$=$\max\{\trivalI(\laux{b}{j}),\trivalI(\faili{\cI}{\kvar}{j})\}$.
Consider $\matom=\metricI{\alwaysF}b$ and let \kinlambda, $\rangeco{j}{k}{\lambda}$.
Since $\htcinterp\models\newPi_\lambda(\program)$, we have that
$\trivalI(\eta_k(\matom))=2$.
We use the definition of Table~\ref{tab:tseitin:metric:temporal}
and get $\trivalI(\ljaux{\matom}{k}{j}\leftrightarrow\laux{b}{j}\vee\faili{\cI}{\kvar}{j})=2$.
By applying Proposition~\ref{prop:three-valued:eq:ht},
$\trivalI(\ljaux{\matom}{k}{j})$=$\trivalI(\laux{b}{j}\vee\faili{\cI}{\kvar}{j})$,
which by the semantics of three-valued logic gives us that
$\trivalI(\ljaux{\matom}{k}{j})$=$\max\{\trivalI(\laux{b}{j}),\trivalI(\faili{\cI}{\kvar}{j})\}$.
The prove that $\trivalI(\ljaux{\matom}{k}{j})=\min\{\trivalI(\laux{b}{j}),\trivalI(\neg\faili{\cI}{\kvar}{j})\}$,
follows analogously.

\end{proof}

 \setcounter{theorem}{\getrefnumber{thm:mlp:full:ht:completeness}}
\addtocounter{theorem}{-1}
\begin{theorem}[Completeness]
  Let \program\ be a metric logic program over $\alphabet$,
  and
  $\M=(\tuple{\Ttrace,\Ttrace}, \tmf)$ a total timed \HT-trace of length $\lambda$.

  If
  $\M$ is a metric equilibrium model of $\program$,
  then
  there exists an equilibrium model $\tandt$ of $\newPi_\lambda(\program)\cup\Delta_{\lambda,\tmflimit}\cup\newPsi_{\lambda,\tmflimit}(\program)$
  with $\tmflimit=\tmf(\lambda-1)$
  such that $\tandt|_{\alphabets\cap\alphabetT}=\mhtToHt{\M}$.
\end{theorem}
\begin{proof}

  We construct $\tandt$ such that its corresponding three-valued interpretation is $\exmtri$.
  Notice the properties from Proposition~\ref{prop:ht:corr:eq} are satisfied
  by the construction of $\exmtri$, so we can apply Proposition~\ref{prop:ht:corr:eq}
  to assure that $\tandt|_{\alphabets\cup\alphabetT}=\mhtToHt{\M}$.
  In what follows,
  we will prove the theorem by proving incrementally that $\tandt$
  is an equilibrium model
  of each part of the program using Lemma~\ref{lem:splitting:full:ht},
  limiting the alphabet to the corresponding atoms.

  We begin proving
  that $\tandt|_\alphabetT$ is an equilibrium model of $\Delta_{\lambda,\tmflimit}$
  Notice that by construction $\tandt$ is timed wrt. $\lambda$ inducing $\tmf$.
  Therefor, $\tandt|_\alphabetT$ is an equilibrium model of $\Delta_{\lambda,\tmflimit}$ by Proposition~\ref{prop:ht:timed:delta}.

  We will now prove that $\tandt|_{\alphabetT\cup\alphabetI}$
  is an equilibrium model of $\Delta_{\lambda,\tmflimit}\cup\newPsi_{\lambda,\tmflimit}(\program)$.
  We can prove that $\tandt|_{\alphabetT\cup\alphabetI}\models\newPsi_{\lambda,\tmflimit}(\program)$
  as in the proof of Lemma~\ref{lem:mlp:full:ht:monotonic}.
  For the equilibrium criteria,
  we assume towards a contradiction that
  there is $\handt$ such that $H\subset T$ and $\handt\models\newPsi_{\lambda,\tmflimit}(\program)$,
  and let $\trivalI'$ be its corresponding three-valued interpretation.
  This means that there is $x\in\alphabetI $ such that $x\in T$ and $x\not\in H$,
  where $\trivalI'(x)=1$.
  Then, $x=\faili{\cI}{\kvar}{j}$ for some \kinlambda, $\rangeco{j}{k}{\lambda}$ and $\cI\in\allI$.
  Since $x\in T$,
  by construction of $\exmtri$, $\tmf(j)-\tmf(k)\not\in\cI$,
  we get that
  $\exmtri(\faili{\cI}{\kvar}{j})=2$.
  Since it is timed wrt. $\lambda$, then
  $d'-d\not\in\cI$ and $\trivalI'(\timet_{\kvar,\tmvar} \wedge \timet_{j,\tmvar'})=2$.
  We also know that
  $\trivalI'(\faili{\intervco{m}{n}}{\kvar}{j} \leftarrow\timet_{\kvar,\tmvar} \wedge \timet_{j,\tmvar'})=2$
  since $\handt\models\newPsi_{\lambda,\tmflimit}(\program)$.
  Then,
  by the three-valued semantics for \HT\ we know that
  $\trivalI'(\faili{\intervco{m}{n}}{\kvar}{j})=2$.
  This leads to a contradiction since $\trivalI'(x)=1$.
Therefore,
  $\tandt|_{\alphabetT\cup\alphabetI}$ is an equilibrium model of $\Delta_{\lambda,\tmflimit}\cup\newPsi_{\lambda,\tmflimit}(\program)$.

  We will now prove that $\tandt$ is an equilibrium model of $\newPi_\lambda(\program)\cup\Delta_{\lambda,\tmflimit}\cup\newPsi_{\lambda,\tmflimit}(\program)$.
  We can proof that $\tandt\models\newPi_\lambda(\program)$
  as in the proof of Lemma~\ref{lem:mlp:full:ht:completeness}.
  We assume towards a contradiction that
  there is $\handt$ such that $H\subset T$ and $\handt\models\newPi_\lambda(\program)$,
  and let $\trivalI'$ be its corresponding three-valued interpretation.
  Then,
  there is $x\in\alphabets\cup\alphabetaux $ such that $x\in T$ and $x\not\in H$ so $\trivalI'(x)=1$.
  We analyze the cases for $x$:
  \casex{x\in\alphabets}{
      $x=a_k$ for some $0\leq k <\lambda$.
      since $\tandt|_{\alphabetT\cup\alphabetI}$ is an equilibrium model of $\Delta_{\lambda,\tmflimit}\cup\newPsi_{\lambda,\tmflimit}(\program)$,
      we can apply Lemma ~\ref{lem:mlp:full:ht:monotonic}
      to get that $tandt|_{\alphabetT\cup\alphabetI}\models\newPsi^{\bot}_{\lambda,\tmflimit}(\program)$,
      and since $\handt|_{\alphabetT\cup\alphabetI}=\tandt|_{\alphabetT\cup\alphabetI}$,
      we get that $\handt\models\newPsi^{\bot}_{\lambda,\tmflimit}(\program)$.
      We also know that $\handt\models\newPi_\lambda(\program)\cup\newPsi_{\lambda,\tmflimit}(\program)$,
      so by Lemma~\ref{lem:mlp:full:ht:correctness},
      we get that
      $\htToMht{\handt|_{\alphabets\cup\alphabetT}}\models\program$.
      This leads to a contradiction since $\M$ is an equilibrium model of $\program$ but $\htToMht{\handt|_{\alphabets\cup\alphabetT}}<\M$.
      Therefore, $\tandt|_{\alphabets}=\handt|_{\alphabets}$
  }
    \casex{x\in\alphabetaux}{
      Notice that if $\handt\models\newPi_\lambda(\program)$
      then the value of $x\in\alphabetaux$ is fully defined by the atoms in $\alphabets\cup\alphabetI$.
      Since $\tandt|_{\alphabets\cup\alphabetI}=\handt|_{\alphabets\cup\alphabetI}$, then,
      $\trivalI'(x)=\exmtri(x)$.
      But since $\tandt$ is total then $\exmtri(x)$ is either $0$ or $1$, but $\trivalI'(x)=1$.
      This leads to a contradiction
    }
  Therefore,
  $\tandt$ is an equilibrium model of $\newPi_\lambda(\program)\cup\Delta_{\lambda,\tmflimit}\cup\newPsi_{\lambda,\tmflimit}(\program)$.

\end{proof}

\setcounter{theorem}{\getrefnumber{thm:mlp:full:ht:correctness}}
\addtocounter{theorem}{-1}

\begin{theorem}[Correctness]
  Let
  \program\ be a metric logic program,
  and
  $\lambda,\tmflimit \in \mathbb{N}$.

  If
  $\tandt$ is an equilibrium model of $\newPi_\lambda(\program)\cup\Delta_{\lambda,\tmflimit}\cup\newPsi_{\lambda,\tmflimit}(\program)$,
  then
  $\htToMht{\tandt|_{\alphabets\cap\alphabetT}}$ is a metric equilibrium model of $\program$.
\end{theorem}
\begin{proof}
  Let $\htToMht{\tandt|_{\alphabetT\cup\alphabets}}=\M$.

  We begin proving that $\M\models\program$.
  We do so by using Lemma~\ref{lem:mlp:full:ht:correctness},
  so it remains proving that the conditions for this Lemma hold.
Firstly,
  $\tandt|_{\alphabetT}$ is an equilibrium model of $\Delta_{\lambda}$
  given the splitting property in Lemma~\ref{lem:splitting:full:ht}.
  Due to Proposition~\ref{prop:ht:delta:timed}, $\tandt|_{\alphabetT}$ is timed wrt. $\lambda$ inducing $\tmf$.
  Since adding other atoms not in ${\alphabetT}$ keeps timed property,
  $\tandt$ is also timed.
For the second condition,
  $\tandt|_{\alphabetT\cup\alphabetI}$ is an equilibrium model of $\Delta_{\lambda}\cup\newPsi_{\lambda}(\program)$ thanks to the Splitting property.
  Applying Lemma~\ref{lem:mlp:full:ht:monotonic} we get that $\tandt|_{\alphabetT\cup\alphabetI}\models\newPsi^\bot_{\lambda}(\program)$.
  As before, we can add the rest of the atoms since there is no interference so $\tandt\models\newPsi^\bot_{\lambda}(\program)$.
We also know that $\tandt\models\newPi_\lambda(\program)\cup\newPsi_{\lambda,\tmflimit}(\program)$ from the hypothesis.
With this we can apply Lemma~\ref{lem:mlp:full:ht:correctness} to get that $\M\models\program$.

  Now we prove that it is in equilibrium.
  By construction $\M$ is total.
  We assume towards a contradiction that
  there is $\M'=(\tuple{\Htrace,\Ttrace},\tmf)$ such that $\Htrace<\Ttrace$ and $\M'\models\program$.
  Then, there is $a\in\alphabet$ such that $a\in T_k$ and $a\not\in H_k$.
We apply Lemma~\ref{lem:mlp:full:ht:completeness},
  and
  let $\handt$ be an HT interpretation
  such that
  $\handt\models\newPi_\lambda(\program)\cup
  \Delta_{\lambda,\tmflimit}\cup
  \newPsi_{\lambda,\tmflimit}(\program)$
  with $\tmflimit=\tmf(\lambda-1)$,
  and $\tuple{H,T}|_{\alphabets\cup\alphabetT}=\mhtToHt{\M'}$
  By construction
  $a_k\in T$ and $a_k\not\in H$,
  which means that $\handt<\tandt$.
  Leading to a contradiction since $\tandt$ is an equilibrium model.

\end{proof}

 \subsubsection{Translation of General Metric Logic Programs to \HTC}
\begin{lemma}[Splitting General \HTC]
    \label{lem:splitting:full:htc}
    Given a metric logic program $\program$
    and a $\lambda\in\mathbb{N}$,
    the stable models of $\newPi_\lambda(\program)\cup\Delta_{\lambda}^c\cup\newPsi^c_{\lambda}(\program)$
    can be computed by means of the splitting technique.
\end{lemma}
\begin{proof}
    Proven as Lemma~\ref{lem:splitting:full:ht} but using the notion of splitting set as defined in
    \citep[Definition 10]{cafascwa20b}.

\end{proof}

\begin{proposition}[Equivalences for correspondence]
    \label{prop:htc:corr:eq}
Given
    an \HTC\ interpretation $\htcinterp$,
    a timed trace $\M=(\tuple{\Htrace,\Ttrace}, \tmf)$ of length $\lambda$
    over alphabet $\alphabet$,
    and its corresponding three-valued metric interpretation $\trivalIm$,
    we have that
    $\htcinterp|_{\alphabets\cap\alphabetT^c}=\mhtToHtc{\M}$,
    iff the following properties hold:

    \begin{enumerate}
      \item $\trivalIhtc(a_k=\true)=\trivalIm(a,k)$ for all $\kinlambda$, $a\in\alphabet$
      \item $\trivalIhtc(\timet_{k}=\tmvar)=2$ iff $\tmf(k)=\tmvar$ for all $\kinlambda$, $\tmvar\in\mathbb{N}$
      \item $\trivalIhtc(\timet_{k}=\tmvar)=0$ iff $\tmf(k)\neq\tmvar$ for all $\kinlambda$, $\tmvar\in\mathbb{N}$
    \end{enumerate}
  \end{proposition}

  \begin{proof}
    Trivial by construction of $\mhtToHtc{\M}$
  \end{proof}

\begin{definition}[Construction of an HTc three-valued interpretation given a metric interpretation]
    \label{def:htc:trival:construction}
    Given a
    timing function $\tmf$,
    a metric three-valued interpretation $\trivalIm$,
    and a program $\program$,
    construct an \HTC\ interpretation, $\htcinterp=\htcinterpex$
    over alphabet $\alphabets\cup\alphabetT^c\cup\alphabetaux\cup\alphabetI$
    as follows:

    \begin{eqnarray*}
      \Vh(\timet_{k})= \Vt(\timet_{k}) & \eqdef & \tmf(k)\\
      \Vh(\faili{\cI}{\kvar}{j})= \Vt(\faili{\cI}{\kvar}{j})  & \eqdef &
      \begin{cases}
        \true & \text{if\;} \tmf(j)-\tmf(\kvar) \not\in \cI \\
        \undefined & \text{otherwise}
      \end{cases}\hspace{20pt} \text{for all interval } \cI \in  \allI\\
          \Vh(a_k) & \eqdef &
          \begin{cases}
              \true & \text{if\;} \trivalIm(k,a)=2    \\
              \undefined & \text{otherwise}
          \end{cases}\hspace{20pt} \text{for all atom } a \in  \alphabet\\
      \Vt(a_k) & \eqdef &
      \begin{cases}
        \true & \text{if\;} \trivalIm(k,a)>0 \\
        \undefined & \text{otherwise}
      \end{cases}\hspace{20pt} \text{for all atom } a \in  \alphabet\\
      \Vh(\laux{\matom}{k}) & \eqdef &
      \begin{cases}
        \true & \text{if\;} \trivalIm(k,\matom)=2 \\
        \undefined & \text{otherwise}
      \end{cases}\hspace{20pt} \text{for all } \matom \in  \program\\
      \Vt(\laux{\matom}{k}) & \eqdef &
      \begin{cases}
        \true & \text{if\;} \trivalIm(k,\matom)>0 \\
        \undefined & \text{otherwise}
      \end{cases}\hspace{20pt} \text{for all } \matom \in  \program\\
      \Vh'(\ljaux{\metricI{\alwaysF}b}{k}{j}) & \eqdef &
      \begin{cases}
        \true & \text{if\;} \Vh'(\laux{b}{j})=\true \text{ or } \Vh'(\faili{\cI}{\kvar}{j})=\true \\
        \undefined & \text{otherwise}
      \end{cases}\hspace{20pt} \text{for all } \metricI{\alwaysF}b \in \program\\
      \Vh'(\ljaux{\metricI{\eventuallyF}b}{k}{j}) & \eqdef &
      \begin{cases}
        \true & \text{if\;} \Vh'(\laux{b}{j})=\true \text{ and } \Vh'(\faili{\cI}{\kvar}{j})=\undefined \\
        \undefined & \text{otherwise}
      \end{cases}\hspace{20pt} \text{for all } \metricI{\eventuallyF}b \in \program
      \end{eqnarray*}

    for all $\kinlambda$, $\rangeco{j}{k}{\lambda}$, $\Vh'\in\{\Vh,\Vt\}$

  \end{definition}

\begin{proposition}[Properties of the three-valued interpretation]
    \label{prop:htc:trival:construction}
    Given a
    timing function $\tmf$,
    a metric three-valued interpretation $\trivalIm$,
    a program $\program$,
    and its corresponding \HTC\ interpretation $\htcinterpex$,
    then $\htcexmtri$
    the following properties hold:
\begin{enumerate}
    \item $\htcexmtri(a_k=\true)=\trivalIm(k,a)$ for all $a\in\alphabet$
    \item $\htcexmtri(\timet_{k}=\tmvar)=2$ iff $\tmf(k)=\tmvar$ for all $\tmvar\in\mathbb{N}$
    \item $\htcexmtri(\timet_{k}=\tmvar)=0$ iff $\tmf(k)\neq\tmvar$ for all $\tmvar\in\mathbb{N}$
    \item $\htcexmtri(\laux{\matom}{k}=\true)=\trivalIm(k,\matom)$ for all $\matom\in\program$
    \item $\htcexmtri(\faili{\cI}{\kvar}{j}=\true)=2$ iff $\tmf(j)-\tmf(\kvar) \not\in \cI$ for all $\cI\in\allI$
    \item $\htcexmtri(\faili{\cI}{\kvar}{j}=\true)=0$ iff $\tmf(j)-\tmf(\kvar) \in \cI$ for all $\cI\in\allI$
    \item $\htcexmtri(\ljaux{\metricI{\alwaysF}b}{k}{j}=\true)=\max\{\htcexmtri(\laux{b}{j}=\true),\htcexmtri(\faili{\cI}{\kvar}{j}=\true)\}$
        for all $\metricI{\alwaysF}b\in\program$
    \item $\htcexmtri(\ljaux{\metricI{\eventuallyF}b}{k}{j}=\true)=\min\{\htcexmtri(\laux{b}{j}=\true),\htcexmtri(\neg(\faili{\cI}{\kvar}{j}=\true))\}$
        for all $\metricI{\eventuallyF}b\in\program$
\end{enumerate}
for all $\kinlambda$, $\rangeco{j}{k}{\lambda}$

\end{proposition}

\begin{proof}
Trivial by construction of $\htcinterpex$ in Definition~\ref{def:htc:trival:construction}
\end{proof}

\begin{proposition}[Construction is total]
    \label{prop:htc:construction:total}
    Given a metric three-valued interpretation $\trivalIm$,
    and a program $\program$,
    if $\trivalIm$ is total
    then $\htcexmtri$ is total.
  \end{proposition}

  \begin{proof}
    For $x\in\alphabetT^c\cup\alphabetI$ it is trivial by construction.
    For $x\in\alphabets$ it holds since $\trivalIm$ is total.
    For $x\in\alphabetaux$ it holds recursively since min and max only give values in the set.
  \end{proof}

\begin{definition}[Additional interval constraints for monotonic]
    \label{def:htc:interval:eq}
    \begin{align}
        \newPsi^{c\bot}_{\lambda}(\program) \eqdef
         & \ \{ \bot \leftarrow \faili{\intervco{m}{n}}{\kvar}{j}
        \wedge (\timet_{\kvar}-\timet_{j} \leq -m)
        \wedge (\timet_{j} - \timet_{\kvar} \leq n - 1)
        \mid
        0 \leq \kvar \leq j < \lambda-1, \intervco{m}{n} \in \mathcal{I} \}\;\cup \\
         & \ \{ \bot \leftarrow \faili{\intervco{m}{\omega}}{\kvar}{j}
        \wedge (\timet_{\kvar}-\timet_{j} \leq -m)
        \mid
        0 \leq \kvar \leq j < \lambda-1, \intervco{m}{\omega} \in \mathcal{I} \}
    \end{align}

\end{definition}

\begin{lemma}[Modeling of extra constraints]
    \label{lem:mlp:full:htc:monotonic}
    Let
    \program\ be a metric logic program,
    and
    $\lambda \in \mathbb{N}$.

    If
    $\htcttinterp$ is an constraint equilibrium model of
    $\Delta^c_{\lambda}\cup\newPsi^c_{\lambda}(\program)$,
    then
    $\htcttinterp\models\newPsi^{\bot,c}_{\lambda}(\program)$
\end{lemma}
\begin{proof}
  We show that it models each rule in $\newPsi^{c\bot}_{\lambda}(\program)$.

  Lets assume towards a contradiction that there is a rule
  $c\in\newPsi^{c\bot}_{\lambda}(\program)$
  such that $\htcttinterp\not\models c$.
  Let
  $c = \bot \leftarrow \faili{\intervco{m}{n}}{\kvar}{j}
  \wedge (\timet_{\kvar}-\timet_{j} \leq -m)
  \wedge (\timet_{j} - \timet_{\kvar} \leq n - 1)$.
  \footnote{ Notice that the case for $\intervco{m}{\omega}$
  is a simpler version of this case as it has less conditions, so will not be considered here.
  }
  Then
  $\htcttinterp\models\faili{\intervco{m}{n}}{\kvar}{j}$
  and $\htcttinterp\models\timet_{\kvar}-\timet_{j} \leq -m$
  and $\htcttinterp\models\timet_{j} - \timet_{\kvar} \leq n - 1$.

  Since $\htcttinterp$ is an equilibrium model of $\Delta^c_{\lambda}$,
  $\htcttinterp\models\Delta^c_\lambda$.
  By applying Proposition~\ref{prop:htc:delta:timed},
  we get that $\htcttinterp$ is timed inducing $\tmf$,
  and
  $\tmf(k)=\Vt(\timet_{k})$ and $\tmf(j)=\Vt(\timet_{j})$.
  Given that $\htcttinterp\models\faili{\intervco{m}{n}}{\kvar}{j}$,
  then it must be founded by a rule in $\newPsi^c_{\lambda}(\program)$,
  because it is in equilibrium.
  If it comes from \eqref{def:htc:psi:one:general} then
  $\htcttinterp\models \neg (\timet_{\kvar}-\timet_{j} \leq -m)$
  and $\htcttinterp\not\models  \timet_{\kvar}-\timet_{j} \leq -m$,
  which leads to a contradiction.
  If it comes from \eqref{def:htc:psi:two:general} then
  $\htcttinterp\models \neg (\timet_{j} - \timet_{\kvar} \leq n - 1)$
  and $\htcttinterp\not\models  \timet_{j} - \timet_{\kvar} \leq n - 1$,
  which leads to a contradiction as well.

\end{proof}

\begin{lemma}[Completeness lemma HTC]
    \label{lem:mlp:full:htc:completeness}
    Let \program\ be a metric logic program over $\alphabet$,
    and
    $\M=(\tuple{\Htrace,\Ttrace}, \tmf)$ a timed \HT-trace of length $\lambda$.

    If
    $\M\models\program$,
    then
    there exists an \HTC\ interpretation  $\htcinterp$
    such that
    $\htcinterp\models\newPi_\lambda(\program)\cup
        \Delta^{c}_{\lambda}\cup
        \newPsi^{c}_{\lambda}(\program)$
    and $\htcinterp|_{\alphabets\cup\alphabetT^c}=\mhtToHtc{\M}$.

\end{lemma}
\begin{proof}

    Let $\trivalIm$ be the three-valued metric interpretation
    corresponding to $\M$ as in Definition~\ref{def:three-valued:mht:correspondence}.
    We construct $\htcinterp=\htcinterpex$ following Definition~\ref{def:htc:trival:construction}.
    We notice the properties from Proposition~\ref{prop:htc:corr:eq} are satisfied
    by the construction of $\trivalIhtc$ (Proposition~\ref{prop:htc:trival:construction}), so we can apply Proposition~\ref{prop:htc:corr:eq}
    to assure that $\htcinterp|_{\alphabets\cup\alphabetT}=\mhtToHtc{\M}$.
We now proceed to show that $\htcinterp\models\newPi_\lambda(\program)\cup\Delta^c_{\lambda}\cup\newPsi^c_{\lambda}(\program)$
    by showing that it is a model of each part.

    Notice that by construction $\htcinterp$ is timed wrt. $\lambda$ inducing $\tmf$.
    We can apply Proposition~\ref{prop:htc:timed:delta}
    to show that $\htcinterp\models\Delta^c_{\lambda}$.

    \medskip
    Next, we show that $\htcinterp\models\newPsi^c_{\lambda}(\program)$ by showing that $\trivalIhtc(\newPsi^c_{\lambda}(\program))=2$.
    Then, we must show that for every rule $r \in\newPsi^c_{\lambda}(\program)$,
    we have that $\trivalIhtc(r)=2$.
    Notice that $\Vh(\timet_{\kvar})=\tmf(k)$ and $\Vh(\timet_{j})=\tmf(j)$, since it is timed.
We proof this by case analysis on $r$
    \casex{\faili{\intervco{m}{n}}{\kvar}{j} \leftarrow \neg (\timet_{\kvar}-\timet_{j} \leq -m)}{
      Now we analyze the two cases below.
      \casex{\tmf(k)-\tmf(j) > -m }{
        We then know that $\tmf(k)-\tmf(j) \not\in \intervco{m}{n}$.
        By Construction and Proposition~\ref{prop:htc:trival:construction} then
        $\trivalIhtc(\faili{\cI}{\kvar}{j})=2$.
        Since $2\geq\max\{0,1,2\}$, then
        $\trivalIhtc(\faili{\cI}{\kvar}{j}) \geq \trivalIhtc(\neg (\timet_{\kvar}-\timet_{j} \leq -m))$.
        Finally, the three-valued semantics for \HTC\ gives us that
        $\trivalIhtc(\faili{\cI}{\kvar}{j} \leftarrow \neg (\timet_{\kvar}-\timet_{j} \leq -m))=2$.
      }
      \casex{\tmf(k)-\tmf(j) \leq -m }{
        We then know that $\trivalIhtc(\timet_{\kvar}-\timet_{j} \leq -m)=2$.
        The three-valued semantics for \HTC\ gives us that
        $\trivalIhtc(\neg\timet_{\kvar}-\timet_{j} \leq -m)=0$,
        and consequently
        $\trivalIhtc(\faili{\cI}{\kvar}{j} \leftarrow \neg (\timet_{\kvar}-\timet_{j} \leq -m))=2$.
      }
    }
    \casex{\faili{\intervco{m}{n}}{\kvar}{j} \leftarrow \neg (\timet_{j} - \timet_{\kvar} \leq n - 1)}{
      Analogous
    }

    \medskip
    Finally, we show that $\htcinterp\models\newPi_\lambda(\program)$ by showing that $\trivalIhtc(\newPi_\lambda(\program))=2$.
    Then, we must show that
    for every rule $\alwaysF{r} \in \program$, and \kinlambda,
    we have $\trivalIhtc(\tk{r})=2$,
    and for every $\matom \in \program$,
    we have $\trivalIhtc(\eta_k^*(\matom))=2$.
    For the first point,
    we know by construction of $\htcinterpex$ and Proposition~\ref{prop:htc:trival:construction},
    that
    $\trival{k}{\mu}=\trivalIhtc(\laux{k}{\mu})$.
    With this equivalence know that
    if $\trival{k}{r}=2$ then $\trivalIhtc(\tk{r})=2$.
    And $\trival{k}{r}=2$ since $\M$ is a model of $\program$.
For the second point,
    again by construction,
    $\trivalIhtc(\laux{\matom}{k}=\true)=\trival{k}{\matom}$.
    We can then apply Proposition~\ref{prop:eq:eta:aux},
    given that the requirements are satisfied by Proposition~\ref{prop:htc:trival:construction}.
    \footnote{Notice that Proposition~\ref{prop:eq:eta:aux} is defined for HT
    but it holds the same for \HTC\ since it uses only boolean variables. (Observation 1 in \cite{cakaossc16a})}
    With this we get $\trivalIhtc(\eta_k^*(\matom)=\true)=2$.

  \end{proof}

\begin{lemma}[Correctness lemma HT]
    \label{lem:mlp:full:htc:correctness}
    Let
    \program\ be a metric logic program,
    $\lambda \in \mathbb{N}$,
    and $\htcinterp$ an timed \HTC\ interpretation wrt. $\lambda$

    If
    $\htcinterp\models\newPi_\lambda(\program)\cup
        \newPsi^{c}_{\lambda}(\program)\cup
        \newPsi^{c\bot}_{\lambda}(\program)$,
    then
    $\htcToMht{\htcinterp|_{\alphabets\cup\alphabetT^c}}\models\program$.
\end{lemma}
\begin{proof}
Let $\htcToMht{\htcinterp|_{\alphabetT^c\cup\alphabets}}=\M$,
and let $\tmf$ be the timing function induced by $\htcinterp$.

We will show below that each of the conditions of Proposition~\ref{prop:eq:eta:aux} hold,
to prove the point-wise equivalence for the auxiliary atoms,
namely $\trivalIhtc(\laux{\matom}{k})=\trival{k}{\matom}$
for all \kinlambda.
With this in hand we argue that $\M\models\program$ follows directly from $\htcinterp\models\newPi_\lambda(\program)$,
given that with these equivalence for any \kinlambda,
$\trivalIhtc(\tk{r})=2$ iff $\trival{k}{r}=2$.

\medskip
Notice $\htcinterp|_{\alphabets\cap\alphabetT^c}=\mhtToHtc{\M}$,
so we apply Proposition~\ref{prop:htc:corr:eq} and get
$\trivalIhtc(a_k=\true)=\trival{k}{a}$ for all $\kinlambda$, $a\in\alphabet$.

\medskip
We will now show that
$\trivalIhtc(\faili{\cI}{\kvar}{j}=\true)=2$ iff $\tmf(j)-\tmf(\kvar) \not\in \cI$ for all $\kinlambda$, $\cI\in\allI$.
For this
let  $\kinlambda$ and $\cI=\intervco{m}{n}\in\allI$.
We know that $\Vh(\timet_{k})=\tmf(k)$ and $\Vh(\timet_{j})=\tmf(j)$,
since $\htcinterp$ is timed inducing $\tmf$.
Now we prove both directions of the equivalence:
\casex{\leftarrow}{
    We start with $\tmf(j)-\tmf(\kvar) \not\in \cI$
    and analyze the two cases for not being in $\cI$:
    \casex{\tmf(j)-\tmf(k) < m }{
        $\trivalIhtc{(\timet_{\kvar}-\timet_{j} \leq -m)}=0$  since $\tmf(k)=\Vh(\timet_k)$ and $\tmf(j)=\Vh(\timet_j)$.
        The three-valued semantics fo \HTC\ gives us that $\trivalIhtc{(\neg(\timet_{\kvar}-\timet_{j} \leq -m))}=2$.
        Since $\htcinterp\models\newPsi^{c}_{\lambda}(\program)$
        then $\trivalIhtc(\faili{\intervco{m}{n}}{\kvar}{j}\leftarrow\neg(\timet_{\kvar}-\timet_{j} \leq -m))=2$,
        which leads to
        $\trivalIhtc{(\faili{\cI}{\kvar}{j}=\true)}=2$
        using the three-valued semantics for \HTC.
    }
    \casex{\tmf(j)-\tmf(k) \geq n }{
        Analogous
    }
}
\casex{\rightarrow}{
    We start with $\trivalIhtc(\faili{\cI}{\kvar}{j}=\true)=2$,
    and let's assume towards a contradiction that $\tmf(j)-\tmf(\kvar) \in \cI$.
    Then, $\trivalIhtc(\timet_{\kvar}-\timet_{j} \leq -m)=2$
    and $\trivalIhtc(\timet_{j} - \timet_{\kvar} \leq n - 1)=2$.
    Using the three-valued semantics for \HTC
    we can put together the following conjunction:
    $\trivalIhtc(\faili{\cI}{\kvar}{j} \wedge
    (\timet_{\kvar}-\timet_{j} \leq -m) \wedge
    (\timet_{j} - \timet_{\kvar} \leq n - 1))=2$,
    and its corresponding constraint,
    $\trivalIhtc(\bot \leftarrow \faili{\cI}{\kvar}{j} \wedge
    (\timet_{\kvar}-\timet_{j} \leq -m) \wedge
    (\timet_{j} - \timet_{\kvar} \leq n - 1))=0$.
    This leads to a contradiction since $\htcinterp\models\newPsi^{c\bot}_{\lambda}(\program)$.
}

\medskip
We will now show that
$\trivalIhtc(\faili{\cI}{\kvar}{j}=\true)=0$ iff $\tmf(j)-\tmf(\kvar) \in \cI$ for all $\kinlambda$, $\cI\in\allI$.
For this
let  $\kinlambda$ and $\cI=\intervco{m}{n}\in\allI$.
As before, we know that $\Vh(\timet_{k})=\tmf(k)$ and $\Vh(\timet_{j})=\tmf(j)$.
Now we prove both directions of the equivalence:
\casex{\leftarrow}{

    We start with $\tmf(j)-\tmf(\kvar) \in \cI$.
    Then, $\trivalIhtc(\timet_{\kvar}-\timet_{j} \leq -m)=2$
    and $\trivalIhtc(\timet_{j} - \timet_{\kvar} \leq n - 1)=2$.
    Since $\htcinterp\models\newPsi^{c\bot}_{\lambda}(\program)$,
    then
    $\trivalIhtc(\bot \leftarrow \faili{\cI}{\kvar}{j} \wedge
    (\timet_{\kvar}-\timet_{j} \leq -m) \wedge
    (\timet_{j} - \timet_{\kvar} \leq n - 1))=2$,
    which means that
    $\trivalIhtc(\faili{\cI}{\kvar}{j} \wedge
    (\timet_{\kvar}-\timet_{j} \leq -m) \wedge
    (\timet_{j} - \timet_{\kvar} \leq n - 1))=0$
    By substituting the know values the three-valued semantics for \HTC\ gives us that
    $\min\{
        \trivalIhtc(\faili{\intervco{m}{n}}{\kvar}{j}), 2\}=0$.
    Therefore, $\trivalIhtc(\faili{\cI}{\kvar}{j}=\true)=0$.
    }
\casex{\rightarrow}{

    We start with $\trivalIhtc(\faili{\cI}{\kvar}{j}=\true)=0$,
    and assume towards a contradiction that $\tmf(j)-\tmf(\kvar) \not\in \cI$.
    We analyze the two cases for not being in $\cI$:
    \casex{\tmf(j)-\tmf(k) < m }{
        We have that $\trivalIhtc{(\timet_{\kvar}-\timet_{j} \leq -m)}=0$
        so
        $\trivalIhtc{(\neg(\timet_{\kvar}-\timet_{j} \leq -m))}=2$.
        Since
        $\htcinterp\models\newPsi^{c}_{\lambda}(\program)$,
        then
        $\trivalIhtc(\faili{\intervco{m}{n}}{\kvar}{j}\leftarrow\neg(\timet_{\kvar}-\timet_{j} \leq -m))=2$,
        which leads to
        $\trivalIhtc{(\faili{\cI}{\kvar}{j}=\true)}=2$.
        This is a contradiction.
    }
    \casex{\tmf(j)-\tmf(k) \geq n }{
        Analogous
    }
    Then $\tmf(j)-\tmf(\kvar) \in \cI$
}

We will now show that that
$\trivalIhtc(\ljaux{\matom}{k}{j})$=$\max\{\trivalIhtc(\laux{b}{j}),\trivalIhtc(\faili{\cI}{\kvar}{j})\}$.
Consider $\matom=\metricI{\alwaysF}b$ and let \kinlambda, $\rangeco{j}{k}{\lambda}$.
Since $\htcinterp\models\newPi_\lambda(\program)$, we have that
$\trivalIhtc(\eta_k(\matom))=2$.
We use the definition of Table~\ref{tab:tseitin:metric:temporal}
and get $\trivalIhtc(\ljaux{\matom}{k}{j}\leftrightarrow\laux{b}{j}\vee\faili{\cI}{\kvar}{j})=2$.
By applying Proposition~\ref{prop:three-valued:eq:htc},
$\trivalIhtc(\ljaux{\matom}{k}{j})$=$\trivalIhtc(\laux{b}{j}\vee\faili{\cI}{\kvar}{j})$,
which by the semantics of three-valued logic gives us that
$\trivalIhtc(\ljaux{\matom}{k}{j})$=$\max\{\trivalIhtc(\laux{b}{j}),\trivalIhtc(\faili{\cI}{\kvar}{j})\}$.
The prove that $\trivalIhtc(\ljaux{\matom}{k}{j})=\min\{\trivalIhtc(\laux{b}{j}),\trivalIhtc(\neg\faili{\cI}{\kvar}{j})\}$,
follows analogously.

\end{proof}
  \setcounter{theorem}{\getrefnumber{thm:mlp:full:htc:completeness}}
\addtocounter{theorem}{-1}
\begin{theorem}[Completeness]
    Let \program\ be a metric logic program over $\alphabet$,
    and
    $\M=(\tuple{\Ttrace,\Ttrace}, \tmf)$ a total timed \HT-trace of length $\lambda$.

    If
    $\M$ is a metric equilibrium model of $\program$,
    then
    there exists a constraint equilibrium model $\htcttinterp$ of $\newPi_\lambda(\program)\cup\Delta^{c}_{\lambda}\cup\newPsi^{c}_{\lambda}(\program)$
    such that $\htcttinterp|_{\alphabets\cap\alphabetT^c}=\mhtToHtc{\M}$.
  \end{theorem}
\begin{proof}
    Let $\trivalIm$ be the
    three-valued metric interpretation $\trivalIm$
    in correspondence
    with $\M$ (Definition~\ref{def:three-valued:mht:correspondence}).
    We construct $\htcttinterp=\htcinterpex$.
    We notice the properties from Proposition~\ref{prop:htc:corr:eq} are satisfied
    by the construction of $\htcinterpex$, so we can apply Proposition~\ref{prop:htc:corr:eq}
    to assure that $\htcttinterp|_{\alphabets\cup\alphabetT^c}=\mhtToHtc{\M}$.
    We also know that $\htcinterpex$ is total by Proposition~\ref{prop:htc:construction:total} since $\M$ is total.
    In what follows, we will prove the theorem by proving
    incrementally that $\htcttinterp$
    is a constraint equilibrium model
    of each part of the program using Lemma~\ref{lem:splitting:full:htc},
    limiting the alphabet to the corresponding atoms.

    \medskip
    We begin proving that
    that $\htcttinterp|_{\alphabetT^c}$
    is an equilibrium model of $\Delta^c_{\lambda}$.
    Notice that by construction $\htcttinterp$ is timed wrt. $\lambda$ inducing $\tmf$.
    Therefor, $\htcttinterp|_{\alphabetT^c}$ is a constraint equilibrium model of $\Delta^c_{\lambda}$
    by applying Proposition~\ref{prop:htc:timed:delta:equilibrium}.

    \medskip
    We will now prove that $\htcttinterp|_{\alphabetT^c\cup\alphabetI}$ is an equilibrium model of $\Delta^c_{\lambda}\cup\newPsi^c_{\lambda}(\program)$
    We can prove that $\htcttinterp|_{\alphabetT^c\cup\alphabetI}\models\newPsi^c_{\lambda}(\program)$
    as in the prove of Lemma~\ref{lem:mlp:full:htc:completeness}.
For the equilibrium criteria,
    we assume towards a contradiction that
    there is $\htcinterp$ such that $\Vh\subset \Vt$ and $\htcinterp\models\newPsi^c_{\lambda}(\program)$.
    This means that there is $x\in\alphabetI $ such that $\Vh(x)=\undefined$ and $\Vt(x)=\true$
    (Notice that $\htcinterp$ is still timed).
    Then, $x=\faili{\cI}{\kvar}{j}$ for some \kinlambda, $\rangeco{j}{k}{\lambda}$ and $\cI\in\allI$.
    Since $x\in \Vt$, $\htcexmtri(\faili{\cI}{\kvar}{j}=\true)=2$.
    By construction construction of $\htcexmtri$ $\tmf(j)-\tmf(k)\not\in\cI$.
    Now we analyze each case for not being in $\cI$:
        \casex{\tmf(j)-\tmf(k) < m }{
          Then, $\trivalIhtc{(\timet_{\kvar}-\timet_{j} \leq -m)}=0$, since  $\tmf(k)=\Vh(\timet_k)$ and $\tmf(j)=\Vh(\timet_j)$.
          By  \HTC\ three-valued semantics we know that $\trivalIhtc{(\neg(\timet_{\kvar}-\timet_{j} \leq -m))}=2$,
          and since $\trivalIhtc(\faili{\intervco{m}{n}}{\kvar}{j}\leftarrow\neg(\timet_{\kvar}-\timet_{j} \leq -m))=2$,
          we have that
          $\trivalIhtc{(\faili{\cI}{\kvar}{j}=\true)}=2$,
          which leads to a contradiction since  $\Vh(x)=\undefined$
        }
        \casex{\tmf(j)-\tmf(k) \geq n }{
          Analogous
        }
    Therefor, $\htcttinterp|_{\alphabetT^c\cup\alphabetI}$ is a constraint equilibrium model of $\Delta^c_{\lambda}\cup\newPsi^c_{\lambda}(\program)$.

    \medskip
    We will prove that $\htcttinterp$ is a constraint equilibrium model of $\Delta^c_{\lambda}\cup\newPsi^c_{\lambda}(\program)\cup\newPi_\lambda(\program)$.
    We can prove that $\htcttinterp\models\newPi_\lambda(\program)$
    as in the proof of Lemma~\ref{lem:mlp:full:htc:completeness}.
    For the equilibrium criteria,
    we assume towards a contradiction that there is $\htcinterp$ such that $\Vh\subset \Vt$ and
    $\htcinterp\models\newPi_\lambda(\program)\cup\Delta^{c}_{\lambda}\cup\newPsi^{c}_{\lambda}(\program)$.
    Then, there is $x\in\alphabets\cup\alphabetaux $ such that $(x,\true)\in \Vt$ and $(x,\true)\not\in \Vh$ so $\trivalIhtc(x=\true)=1$.
    We know that it can't be $x\in\alphabetT^c\cup\alphabetI$ since $\htcttinterp|_{\alphabetT^c\cup\alphabetI}=\htcinterp|_{\alphabetT^c\cup\alphabetI}$.

      \casex{x\in\alphabets}{
        We have that $x=a_k$ for some $0\leq k <\lambda$.
        We get $\htcttinterp|_{\alphabetT^c\cup\alphabetI}\models\newPsi^{\bot,c}_{\lambda}(\program)$
        by applying Lemma ~\ref{lem:mlp:full:htc:monotonic} since $\htcttinterp|_{\alphabetT^c\cup\alphabetI}$
          is an equilibrium model of $\Delta^c_{\lambda}\cup\newPsi^c_{\lambda}(\program)$.
        Notice that, since $\htcinterp|_{\alphabetT^c\cup\alphabetI}=\htcttinterp|_{\alphabetT^c\cup\alphabetI}$,
        we have $\htcinterp\models\newPsi^{\bot,c}_{\lambda}(\program)$.
        From $\htcinterp\models\newPi_\lambda(\program)\cup\newPsi^c_{\lambda}(\program)$
        and Lemma~\ref{lem:mlp:full:htc:correctness},
        we know that
        $\htcToMht{\htcinterp|_{\alphabets\cup\alphabetT^c}}\models\program$.
        Which leads to a contradiction since $\M$ is an equilibrium model of $\program$ but $\htcToMht{\htcinterp|_{\alphabets\cup\alphabetT^c}}<\M$.
        This means that $\htcttinterp|_{\alphabets}=\htcinterp|_{\alphabets}$.

      }

      \casex{x\in\alphabetaux}{
        Notice that if $\htcinterp\models\newPi_\lambda(\program)$
        then the value of $x\in\alphabetaux$ is fully defined by the atoms in $\alphabets\cup\alphabetI$.
        Since $\htcttinterp|_{\alphabets\cup\alphabetI}=\htcinterp|_{\alphabets\cup\alphabetI}$, then,
        $\trivalIhtc(x)=\trivalIhtcletter^{\htcttinterp}(x)$.
        But since $\htcttinterp$ is total then $\trivalIhtcletter^{\htcttinterp}(x)$ is either $0$ or $1$, but $\trivalIhtc(x)=1$.
        Which leads to a contradiction.

      }

    Therefore, $\htcttinterp$ is a constraint equilibrium model of $\Delta^c_{\lambda}\cup\newPsi^c_{\lambda}(\program)\cup\newPi_\lambda(\program)$.

\end{proof}

\setcounter{theorem}{\getrefnumber{thm:mlp:full:htc:correctness}}
\addtocounter{theorem}{-1}
\begin{theorem}[Correctness]
    Let
    \program\ be a metric logic program,
    and
    $\lambda \in \mathbb{N}$.

    If
    $\htcttinterp$ is an constraint equilibrium model of $\newPi_\lambda(\program)\cup\Delta^{c}_{\lambda}\cup\newPsi^{c}_{\lambda}(\program)$,
    then
    $\htcToMht{\htcttinterp|_{\alphabets\cap\alphabetT^c}}$ is a metric equilibrium model of $\program$.
  \end{theorem}
\begin{proof}
    Let $\M=\htcToMht{\htcttinterp|_{\alphabets\cap\alphabetT^c}}$.

    We begin proving that $\M\models\program$.
    We do so by using Lemma~\ref{lem:mlp:full:htc:correctness},
    so it remains proving that the conditions for this Lemma hold.
Firstly,
    $\htcttinterp|_{\alphabetT^c}$ is a constraint equilibrium model of $\Delta^c_{\lambda}$
    due to Lemma~\ref{lem:splitting:full:htc}.
Due to Proposition~\ref{prop:htc:delta:timed}, $\htcttinterp|_{\alphabetT^c}$ is timed wrt. $\lambda$ inducing $\tmf$.
    Since adding other atoms not in ${\alphabetT^c}$ keeps timed property,
    $\htcttinterp$ is also timed wrt. $\lambda$ inducing $\tmf$.
For the second condition.
    $\htcttinterp|_{{\alphabetT^c}\cup\alphabetI}$ is a constraint equilibrium model of $\Delta^c_{\lambda}\cup\newPsi^c_{\lambda}(\program)$ thanks to the Splitting property.
    Applying Lemma~\ref{lem:mlp:full:htc:monotonic} we get that $\htcttinterp|_{{\alphabetT^c}\cup\alphabetI}\models\newPsi^\bot_{\lambda}(\program)$.
    As before, we can add the rest of the atoms since there is no interference so $\htcttinterp\models\newPsi^\bot_{\lambda}(\program)$.
With this we can apply Lemma~\ref{lem:mlp:full:htc:correctness} to get that $\M\models\program$.

    We then prove that it is in equilibrium.
    By construction $\M$ is total.
    We assume towards a contradiction that
    there is $\M'=(\tuple{\Htrace,\Ttrace},\tmf)$ such that $\Htrace<\Ttrace$ and $\M'\models\program$
    Then, there is $a\in\alphabet$ such that $a\in T_k$ and $a\not\in H_k$,
    for some $\trangeco{k}{0}{\lambda}$.
We apply Lemma~\ref{lem:mlp:full:htc:completeness},
    and
    let $\htcinterp$ be the \HTC\ interpretation
    such that
    $\htcinterp\models\newPi_\lambda(\program)\cup
    \Delta^c_{\lambda}\cup
    \newPsi^c_{\lambda}(\program)$
    with $\tmflimit=\tmf(\lambda-1)$,
    and $\tuple{H,T}|_{\alphabets\cup{\alphabetT^c}}=\mhtToHt{\M'}$.
    By construction $(a_k,\true)\in \Vt$ and $(a_k,\true)\not\in \Vh$,
    which means that $\htcinterp<\htcttinterp$
    leading to a
    contradiction since $\htcttinterp$ is a constraint equilibrium model.
    Since there is no smaller model, $\M$ is in equilibrium.

\end{proof}

\end{document}